\documentclass{article}
\usepackage{general_def}
\usepackage{specific_def}
\usepackage{multirow}
\usepackage{xcolor}
\usepackage{amsmath,amssymb,amsthm, bbm}

\usepackage[T1]{fontenc}    
\usepackage[utf8]{inputenc}
\usepackage{fullpage}
\usepackage{parskip}
\usepackage[round]{natbib}

\definecolor{mydarkblue}{rgb}{0,0.08,0.45}
\hypersetup{ %
colorlinks=true,
linkcolor=mydarkblue,
citecolor=mydarkblue,
filecolor=mydarkblue,
}
\usepackage[round]{natbib}

\renewcommand{\mg}[1]{}

\title{Max-Margin Works while Large Margin Fails: Generalization without Uniform Convergence}

\begin{document}
\date{}

\author{%
  Margalit Glasgow \\
  Stanford University \\
  \texttt{mglasgow@stanford.edu} 
  \and
  Colin Wei \\
  Stanford University \\
  \texttt{colinwei@stanford.edu}
  \and
  Mary Wootters \\
  Stanford University \\
  \texttt{marykw@stanford.edu}
  \and
  Tengyu Ma \\
  Stanford University \\
  \texttt{tengyuma@stanford.edu}
  }
\maketitle

\begin{abstract}
A major challenge in modern machine learning is theoretically understanding the generalization properties of overparameterized models. Many existing tools rely on \em uniform convergence \em (UC), a property that, when it holds, guarantees that the test loss will be close to the training loss, uniformly over a class of candidate models.
\citet{nagarajan2019uniform} show that in certain simple linear and neural-network settings, any uniform convergence bound will be vacuous, leaving open the question of how to prove generalization in settings where UC fails. Our main contribution is proving novel generalization bounds in two such settings, one linear, and one non-linear. We study the linear classification setting of \citet{nagarajan2019uniform}, and a quadratic ground truth function learned via a two-layer neural network in the non-linear regime. We prove a new type of margin bound showing that above a certain signal-to-noise threshold, any near-max-margin classifier will achieve almost no test loss in these two settings. Our results show that near-max-margin is important: while any model that achieves at least a $(1 - \epsilon)$-fraction of the max-margin generalizes well, a classifier achieving half of the max-margin may fail terribly. 
Building on the impossibility results of \citet{nagarajan2019uniform}, under slightly stronger assumptions, we show that \em one-sided \em UC bounds and classical margin bounds will fail on near-max-margin classifiers.
Our analysis provides insight on why memorization can coexist with generalization: we show that in this challenging regime where generalization occurs but UC fails, near-max-margin classifiers simultaneously contain some generalizable components and some overfitting components that memorize the data. The presence of the overfitting components is enough to preclude UC, but the near-extremal margin guarantees that sufficient generalizable components are present.

\end{abstract}

\section{Introduction}

A central challenge of machine learning theory is understanding the generalization of overparameterized models. While in many real-world settings deep networks achieve low test loss, their high capacity makes theoretical analysis with classical tools difficult, or sometimes impossible \citep{zhang2017understanding, nagarajan2019uniform}. Most classical theoretical tools are based on \em uniform convergence \em (UC), a property that, when it holds, guarantees that the test loss will be close to the training loss, uniformly over a class of candidate models. Many generalization bounds for neural networks are built on this property, e.g. \cite{neyshabur2015norm, neyshabur2017pac,neyshabur2018towards, harvey2017nearly,golowich2018size}.

The seminal work of \cite{nagarajan2019uniform} gives theoretical and empirical evidence that UC cannot hold in natural overparameterized linear and neural network settings. The impossibility results of Nagarajan and Kolter are very strong: they rule out UC on the smallest reasonable family of models, that is, the models output by gradient descent on clean data. In particular, they prove that in an overparameterized linear classification problem, a certain class of models found by gradient descent will achieve small test loss, but any UC bound over this class will be vacuous. 
In a two-layer neural network setting, \citet{nagarajan2019uniform} empirically demonstrate a similar phenomenon for the $0/1$ loss. Beyond these toy settings, they also empirically evaluate many generalization bounds on neural networks trained in practice and show their vacuity.

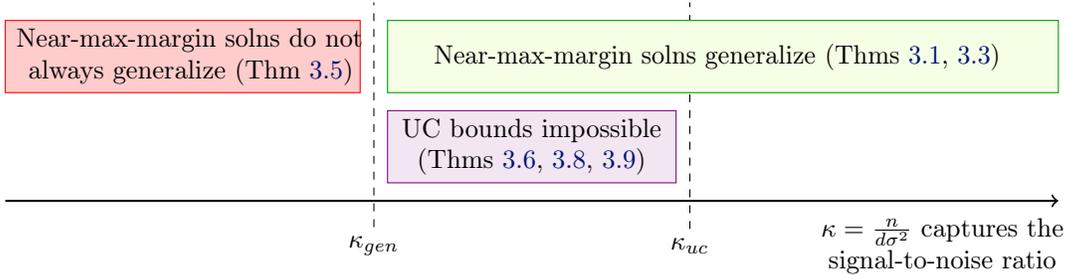
\begin{figure}[t]
    \centering
    \begin{tikzpicture}[xscale=1.4,yscale=1.2]
    \draw[thick, ->] (0,0) to (10,0);
    \node at (8.9, -.5) {\begin{minipage}{4cm}\begin{center}$\kappa = \frac{n}{d\sigma^2}$ captures the signal-to-noise ratio\end{center}\end{minipage}};
    \node(a) at (3.5,-.5) {$\kappa_{gen}$};
    \node(b) at (6.5, -.5) {$\kappa_{uc}$};
    \draw[dashed] (a) to (3.5, 2.2);
    \draw[dashed] (b) to (6.5, 2.2);
    \draw[violet,fill=violet!10] (3.63,.2) rectangle (6.37, 1);
    \draw[red,fill=red!20] (0, 1.2) rectangle (3.37, 2);
    \draw[green!60!black, fill=green!40!yellow!10] (3.63, 1.2) rectangle (10, 2);
    \node at (1.75, 1.6) {\begin{minipage}{4.6cm} \begin{center} Near-max-margin solns do not always
    generalize (Thm~\ref{thm:xor_no_gen}) \end{center} \end{minipage}};
    \node at (6.75, 1.6) {\begin{minipage}{7.5cm} \begin{center} Near-max-margin solns
    generalize (Thms~\ref{thm:linear_gen}, \ref{thm:xor_gen}) \end{center} \end{minipage}};
    \node at (5.0, .6) {\begin{minipage}{4.5cm} \begin{center} UC bounds impossible (Thms~\ref{thm:one_sided_impossible_linear}, \ref{thm:margin_impossible_linear}, \ref{thm:margin_impossible}) \end{center} \end{minipage}};
    \end{tikzpicture}
    \caption{Thresholds for Uniform Convergence and Generalization. All theorems require a sufficiently large dimension, ie. $d \geq \Omega(n)$.}
    \label{fig:thresholds}
\end{figure}

Many margin bound do not technically fit into the category of UC bounds defined by Nagarajan and Kolter, but still may be intrinsically limited for similar reasons. 
Classical margin bounds (see eg. \citet{shalev2014understanding, kakade2009complexity}) and related margin bounds for neural networks \citep{wei2019data,wei2020improved, bartlett2017spectrally,golowich2018size}, in which the generalization guarantee scales inversely polynomial with the margin size, are typically proved via uniform convergence on a surrogate loss (eg. the hinge loss or ramp loss) that upper bounds the $0/1$ misclassification loss. Nagarajan and Kolter's results show that any uniform bound on the ramp loss is vacuous in an overparameterized linear setting, suggesting that classical margin bounds may not be useful. \citet{muthukumar2021classification} shows empirically that such margin bounds are vacuous in a broader linear settings. 
In light of this, it is very important to develop theoretical tools to analyze generalization in settings where uniform convergence cannot yield meaningful bounds.

In this paper we establish novel margin-based generalization bounds in regimes where UC provably fails. These bounds guarantee generalization in the extremal case where the model has a near-maximal margin, and thus we call them \em extremal margin bounds. \em Our main motivation for studying near-max-margin solutions is that minimizing the logistic loss with weak $\ell_2$-regularization achieves max-margin solutions~\citep{wei2019regularization}, and that minimizing unregularized loss with gradient descent converges to the max-margin solution (or stationary points of the max-margin objective)~\citep{ lyu2019gradient,lyu2021gradient}. In linear settings, SGD converges to the max-margin~\citep{nacson2019stochastic}.

Our results consider two settings, the linear setting of \citet{nagarajan2019uniform}, and a commonly studied quadratic problem learned on a two-layer neural network~\citep{wei2019regularization, frei2022random}.
In Theorems~\ref{thm:linear_gen}, \ref{thm:xor_gen}, we prove that above a certain signal-to-noise threshold $\kgen$, near-max-margin solutions will generalize. Below this threshold, max-margin solutions may not generalize (Theorem~\ref{thm:xor_no_gen}). Below a second higher threshold, $\kphen$, uniform convergence fails (Theorems~\ref{thm:one_sided_impossible_linear}). Additionally in this regime where UC fails, we show that classical margin bounds can only yield loose guarantees, even for the max-margin solution (Theorem~\ref{thm:margin_impossible_linear} and \ref{thm:margin_impossible}). 
In Figure~\ref{fig:thresholds} we illustrate these three regions; the main significance of our results is in the challenging middle region between $\kgen$ and $\kphen$ where generalization occurs, but UC fails. Our extremal margin bounds are fundamentally different from classical margin bounds and are not based on uniform convergence. 

Prior works also studied the challenging regime where uniform convergence does not work, and notable progress has been made in linear settings. \citet{zhou2020uniform} and \citet{koehler2021uniform} show that for linear regression, the test loss can be uniformly bounded for all low-norm solutions that perfectly fit the data; nevertheless, \citet{yang2021exact} shows that such bounds are still loose on the min-norm solution. \citet{negrea2020defense} suggests an alternative framework based on uniform convergence over a less complex family of \em surrogate \em models; they use this technique to show generalization in a linear setting and in another high-dimensional problem amenable to analysis. To our knowledge, our results are the first instance of theoretically proving generalization in a neural network setting (that is not in the NTK regime) where UC provably fails.

We leverage near-max-margins in a unified way for both the linear and nonlinear settings, and we hope that this approach will be useful more broadly in overparameterized settings. 
In the challenging regime of generalization without UC, good learned models contain some generalizable signal components and some overfitting components that memorize the data. Our main technique is to show that any near max-margin solution has to contain \em both \em signal components and overfit components. The overfitting component causes UC to fail, but fortunately, has a reduced influence on a random test example, whereas the signal component has a similar influence on training and test examples.

Besides our generalization bounds, we prove impossibility results similar to those of \citet{nagarajan2019uniform}. Under a slightly stronger assumption than the work of \citet{nagarajan2019uniform}, namely, that the bound in question is useful for a \em set \em of ground truth distributions, we rule out one-sided UC bounds (which upper bound the test, instead of just two-sided bounds (which must upper and lower bound the test loss). This stronger assumption is formalized in Definition 2.3 and justified in Remark 2.4.
Further, our results show that there are models that achieve a large but non-near-max-margin (e.g., half the max-margin), but do not generalize at all. We prove that this phase transition cannot be captured by classical margin bounds where generalization decays inversely polynomially with the margin.
\subsection{Additional Related Work}
A large body of work highlights challenges in using classical statistical theory to explain generalization in deep learning. Experimental results~\citep{zhang2017understanding,neyshabur2017exploring} point out that despite being large in traditional capacity measures such as Rademacher complexities, deep networks still generalize well, and new explanations are needed to understand this behavior.~\citet{belkin2018understand} show that similar challenges hold in kernel methods. 
Beyond the work of \citet{nagarajan2019uniform},
~\citet{bartlett2021failures} prove that in a linear interpolation setting, model-dependent generalization bounds fail for the min-norm solution.~\citet{koren2022benign} show that SGD can exhibit a benign underfitting phenomenon where the test loss is small but empirical loss is large. 


One related body of work has focused more closely on characterizing ``benign overfitting'', where the model overfits to noise in labels of the training data but still attains good test performance. Our setting differs from benign overfitting because (i) UC provably fails in our setting (whereas such results were not presented in benign overfitting literature), and (ii) the overfitting in our setting cannot be avoided with regularization (See Remark~\ref{remark:benign} for more discussion). 
Most of the results in this area concern linear models: ~\citet{bartlett2020benign} analyze benign overfitting in regression problems by leveraging a closed form expression for the min-norm solution.~\citet{muthukumar2021classification, shamir2022implicit, cao2021risk, wang2020binary} and~\citet{wang2021benign} study classification settings. The works of \citet{muthukumar2020harmless} and \citet{shamir2022implicit} reveal that is often possible to have benign overfitting in classification, whereas in regression for the same covariate distribution, the overfitting would imply poor generalization. \citet{cao2021risk} achieves similar risk bounds to ours for more general distributions of the covariates, but under a tighter overparameterization assumption. (See Remark~\ref{rem:lin_class} for more discussion). Also closely related to our work on linear classification is the work of \citet{montanari2019generalization}, which asymptotically characterizes the generalization of the max-margin solution as $n, d \rightarrow \infty$. Benign overfitting in neural networks has been shown in several simple settings.~\citet{frei2022benign} analyzes two-layer neural networks trained by gradient descent on linearly-separable data.
~\citet{cao2022benign} studies benign overfitting for a two-layer simplified convolutional network.

More broadly, a variety of new generalization bounds have been derived in hopes of explaining generalization in deep learning. While none of these bounds have been explicitly proven to succeed in regimes where UC fails, they leverage additional properties of the training data or the optimization process and thus are not directly susceptible to the critiques of \citet{nagarajan2019uniform}. Among these are works that leverage properties such as Lipschitzness of the model on the training data~\citep{arora2018stronger,nagarajan2019deterministic,wei2019data,wei2019improved}, use algorithmic stability~\citep{mou2018generalization,li2019generalization, chatterjee2022generalization}, or information-theoretic perspectives~\citep{negrea2019information,haghifam2021towards}.


Finally, a body of work seeks to draw connections between optimization and generalization in deep learning by studying implicit regularization effects of the optimization algorithm (see e.g.~\citep{gunasekar2017implicit,li2017algorithmic,gunasekar2018characterizing,gunasekar2018implicit,woodworth2019kernel,damian2021label,haochen2020shape,li2019towards,wei2020implicit} and related references). Most relevent in this literature is the aforementioned work connecting gradient descent and max-margin solutions.
\section{Preliminaries}\label{sec:model}

Our work achieves results in two settings. The first is a linear setting previously studied by~\citet{nagarajan2019uniform} where both the ground truth and the trained model are linear. In the second nonlinear setting, studied before by~\citet{wei2019regularization,frei2022random}, the ground truth is quadratic, and the trained model is a two-layer neural network. In both settings, the data is drawn from a product distribution on features involved in the ground truth labeling function, and ``junk'' features orthogonal to the signal. We formalize the two settings below.

\textbf{Linear setting}\newline
$\blacktriangleright$ \textbf{Data Distribution}. Fix some ground truth unit vector direction $\mu \in \mathbb{R}^d$. Let $x = z + \xi$, where $z \sim \text{Uniform}(\{\mu, -\mu\})$ and $\xi$ is uniform on the sphere of radius $\sqrt{d - 1}\sigma$ in $d - 1$ dimensions, orthogonal to the direction $\mu$. Let $y = \mu^Tx$, such that $y = 1$ with probability $1/2$ and $-1$ with probability $1/2$. We denote this distribution of $(x,y)$ on $\mathbb{R}^d \times \{-1, 1\}$  by $\mathcal{D}_{\mu, \sigma, d}$. \newline
$\blacktriangleright$ \textbf{Model.} We learn a model $w \in \mathbb{R}^d$ that predicts $\hat{y} = \text{sign}(f_w(x))$ where $f_w(x) = w^Tx$.

\textbf{Setting for Two-Layer Neural Network Model with Quadratic ``XOR'' Ground Truth}\newline
$\blacktriangleright$ \textbf{Data Distribution.} Fix some orthogonal ground truth unit vector directions $\mu_1$ and $\mu_2$ in $\mathbb{R}^d$. Let $x = z + \xi$, where $z \sim \text{Uniform}(\{\mu_1, -\mu_1, \mu_2, -\mu_2\})$ and $\xi$ is uniform on the sphere of radius $\sqrt{d - 2}\sigma$ in $d - 2$ dimensions, orthogonal to the directions $\mu_1$ and $\mu_2$. Let $y = (\mu_1^Tx)^2 - (\mu_2^Tx)^2$ for some orthogonal ground truth directions $\mu_1$ and $\mu_2$ (see Figure~\ref{fig:xor}(left)). We denote this distribution of $(x, y)$ on $\mathbb{R}^d \times \{-1, 1\}$  by $\mathcal{D}_{\mu_1, \mu_2, \sigma, d}$.  We call this the XOR problem because $y = \on{XOR}\left((\mu_1 + \mu_2)^Tx, (-\mu_1 + \mu_2)^Tx\right)$. For instance, if $\mu_1 = e_1$ and $\mu_2 = e_2$, then $y = x_1^2 - x_2^2$. As can be seen in Figure~\ref{fig:xor}(left), this distribution is not linearly separable, and so one must use nonlinear model to learn in this setting.

$\blacktriangleright$ \textbf{Model.} 
Fix $a \in \{-1,1\}^m$ so that $\sum_i a_i = 0$. The model is a two-layer neural network with $m$ hidden units and activation function $\phi$, parameterized by $W \in \mathbb{R}^{m \times d}$. $W$ (which will be learned) represents the weights of the first layer and $a$ (which is fixed) is the second layer weights. The model predicts $f_W(x) = \sum_{i = 1}^m a_i \phi(w_i^Tx)$, where $w_i \in \mathbb{R}^d$ denotes the $i$'th column of $W$.
We work with activations $\phi$ of the form $\phi(z) = \max(0, z)^h$ for $h \in [1, 2)$, and require that $m$ is divisible by $4$\footnote{The assumption that $m$ is divisible by $4$ is for convenience, and can be removed if $m$ is large enough.}.

We define a \em problem class \em of distributions to be a set of data distributions. In this paper, we work with the linear problem class $\Omega_{\sigma, d}^{\on{linear}} := \{\mathcal{D}_{\mu, \sigma, d} : \mu \in \mathbb{R}^d, \|\mu\| = 1\}$, and the quadratic problem class  $\Omega_{\sigma, d}^{\on{XOR}} := \{\mathcal{D}_{\mu_1, \mu_2, \sigma, d} : \mu_1 \perp \mu_2 \in \mathbb{R}^d, \|\mu_1\| = \|\mu_2\| = 1\}$. Here $\|\cdot\|$ denotes the $\ell_2$ norm.

We will sometimes abuse notation and say that $x \sim \mathcal{D}$ instead of saying that $(x, y) \sim \mathcal{D}$.
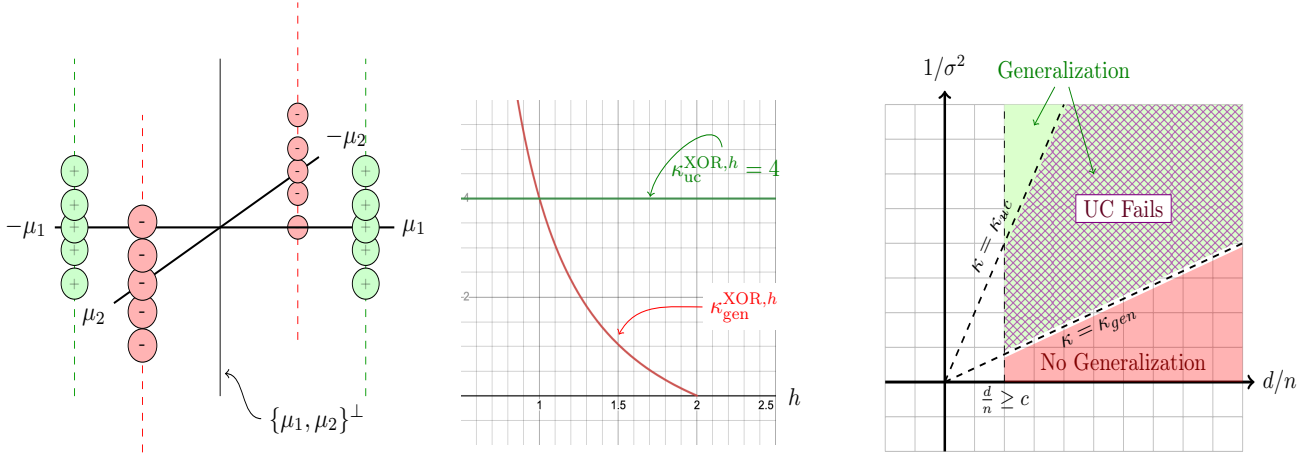
\begin{figure}
\resizebox{6cm}{6cm}{
\begin{tikzpicture}[scale=1.5]

    \node(p2) at (-1.3, -.8) {$\mu_2$};
    \node(m2) at (1.3, .8) {$-\mu_2$};
    \draw[thick](m2)--(p2);
    \draw[red, dashed] (.8, -1) to (.8, 2);
    \begin{scope}[yshift=.5cm]
     \foreach \y in {-0.2, -0.5, 0, 0.5, 0.2}{
        \node[draw, circle, fill=red!30, scale=.6] at (.8, \y) {-};
    }
    \end{scope}
        \node(m1) at (-2,0) {$-\mu_1$};
    \node(p1) at (2,0) {$\mu_1$};
        \draw[thick](m1)--(p1);
    \draw[green!60!black,dashed] (-1.5, -1.5) -- (-1.5, 1.5);
        \draw[green!60!black,dashed] (1.5, -1.5) -- (1.5, 1.5);
    \foreach \y in {-0.2,  -0.5, 0, 0.5, 0.2}{
        \foreach \x in {-1.5, 1.5}{
        \node[draw, circle, fill=green!20, scale=.6] at (\x, \y) {+};
        }
    }
    \draw[red, dashed] (-.8, -2) to (-.8, 1);
    \begin{scope}[yshift=-.5cm]
     \foreach \y in {-0.25, -0.55, 0, 0.55, 0.25}{
        \node[draw, circle, fill=red!30, scale=.85] at (-.8, \y) {-};
    }
    \end{scope}
    \draw(0,-1.5)--(0,1.5);
    \node(c) at (1, -1.7) {$\{\mu_1, \mu_2\}^\perp$};
    \draw[->] (c) to[out=180,in=0] (0.05, -1.3);

    \end{tikzpicture}}
\resizebox{4.8cm}{4.8cm}{
\begin{tikzpicture}
\node {\includegraphics[width=5cm]{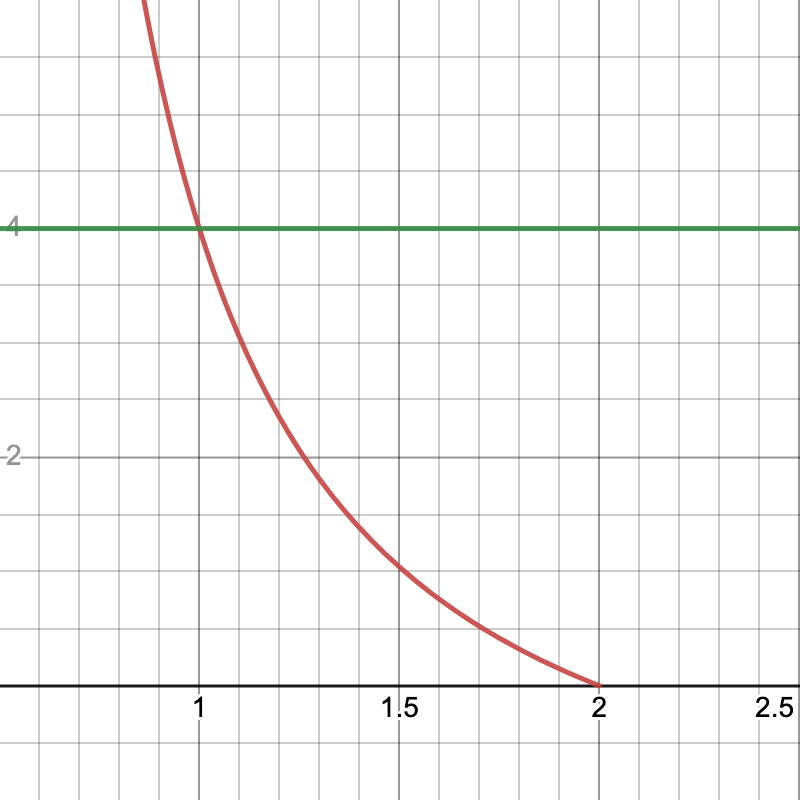}};
\node(a)[red,fill=white] at (2,-.5) {$\kgenx$};
\node at (2.8, -1.8) {$h$};
\node(b)[green!50!black] at (1.7,1.5) {$\kphenx = 4$};
\draw[->, green!50!black] (b) to[out=100,in=80] (0.5,1.1);
\draw[->,red] (a) to[out=180,in=60] (0,-.9);
\end{tikzpicture}}
    \resizebox{6.6cm}{5.4cm}{
    \usetikzlibrary{patterns}

\begin{tikzpicture}

    \draw[gray!50, thin, step=0.5] (-1,-1) grid (5,4);
    \draw[very thick,->] (-1,0) -- (5.2,0) node[right] {$d/n$};
    \draw[very thick,->] (0,-1) -- (0,4.2) node[above] {$1/\sigma^2$};


    \fill[green!70!yellow!20] (1,0.45) -- (5,2.05) -- (5, 4)  -- (1, 4) -- cycle;

    \fill[red,opacity=0.3] (1,0.35) -- (5,1.95) -- (5, 0)  -- (1, 0) -- cycle;

    \fill[pattern color=violet!60,opacity=1, pattern=crosshatch] (1,0.45) -- (5,2.05) -- (5, 4)  -- (2.05, 4) -- (1, 1.9) -- cycle;

    \draw [dashed, thick] (0,0) -- node[below,sloped] {$\kappa = \kappa_{gen}$} (5,2);
    \draw [dashed, thick] (0,0) -- node[above,sloped] {$\kappa = \kappa_{uc}$} (2,4);

    \draw [dashed] (1,0) -- (1,4);

   \node at (1, -0.25) {\small $\frac{d}{n} \geq c$};

   \node(b)[green!50!black] at (2,4.5) {Generalization};
\draw[->, green!50!black] (b) to (1.5,3.7);
\draw[->, green!50!black] (b) to (2.5,3);

\node(c)[red!50!black] at (3,0.3) {No Generalization};
\draw[violet,fill=white] (2.2, 2.3) rectangle (3.8, 2.7);
\node(c)[purple!50!black] at (3,2.5) {UC Fails};

\end{tikzpicture}}
    \caption{Left: Quadratic XOR Problem. Middle: $\kgenx$ (red) and $\kphenx$ (green) as a function of $h$. Right: Regions in which theorems hold. As shown in this figure, our results only hold when there is sufficient overparameterization, that is, $d \geq cn$ for a constant $c$.}
    \label{fig:xor}
\end{figure}

Before proceeding, we make some comments on our models and compare it to related work.

{\textbf{Large dimension assumption.}} In both the linear and non-linear settings, our focus is an overparameterized regime where the dimension $d$ is at least a constant factor times larger than $n$, the number of training samples. Such an assumption is mild relative to the assumptions made in related work, which require $d = \omega(n)$ (see Remark~\ref{rem:lin_class} for a detailed discussion of work on linear models; for neural networks, the work of \cite{frei2022benign} and \cite{cao2022benign} assume that $d \geq n^2$ or stronger). When the dimension is sufficiently large (in particular, at least $\omega(n)$), with high probability, the max-margin solution coincides with the min-norm regression solution (see~\cite{hsu2021proliferation}), meaning the max-margin solution can be analyzed via a closed-form expression. Our work is fundamentally different from the work on linear classification which operates in the $d = \omega(n)$ regime, because in our setting when $d = \Theta(n)$, these two solutions do not coincide. 
Note that $n \gg d$, then uniform convergence analyses should apply both for the linear and non-linear problems (see \cite{wei2019regularization} which gives a UC-based margin bound for the XOR problem). There remains an interesting open regime when $d$ is close to $n$ where UC may not apply, but our results do not hold.


{\textbf{Distribution of Covariates.}}
Many works on linear classification study more general data models which allow arbitrary decay of the eigenvalues of the covariance matrix (eg. \cite{muthukumar2021classification, wang2020binary, cao2021risk}), or variance in the signal direction, that is, $x^T\mu \neq y$ (eg. \cite{shamir2022implicit}). We work with a simpler distribution, which is still challenging, because it defies existing analyses built on UC or closed-form solutions. Our results can be extended to a setting where the orthogonal features are sub-Gaussian, since the key property we leverage is concentration of the $n \times n$ Gram matrix containing the dot product of all pairs of orthogonal components of the data points. This extension would be similar to the data distribution of studied in \cite{frei2022benign} on neural networks. We expect, however, that having significant variance in the signal direction would affect the threshold at which generalization occurs.

\subsection{Background and Definitions on Uniform Convergence}

In this subsection, we provide some definitions to quantitatively reason about the limits of UC bounds. We also provide some definitions and background on margin bounds. Many of the definitions are based on those from \citet{nagarajan2019uniform}.

For any loss function $\mathcal{L}: \mathbb{R} \times \mathbb{R} \rightarrow \mathbb{R}$, and a hypothesis $h$ mapping from a domain $\mathcal{X}$ to $\mathbb{R}$, we define the test loss on a distribution $\mathcal{D}$ to be $\mathcal{L}_{\mathcal{D}}(h) := \mathbb{E}_{(x, y) \sim \mathcal{D}}\mathcal{L}(h(x), y)$. For a set of examples $S = \{(x_i, y_i)\}_{i \in [n]}$, we define $\mathcal{L}_{S}(h) := \mathbb{E}_{i \in [n]}\mathcal{L}(h(x_i), y_i)$ to be the empirical loss over the samples.

Unless otherwise specified, we will use $\mathcal{L}$ to denote the $0/1$ loss, which equals $1$ if and only if the signs of the two labels disagree, that is, $\mathcal{L}(y, y') = \textbf{1}(\on{sign}(y) \neq \on{sign}(y'))$.


Typically in machine learning one considers a global hypothesis class $\mathcal{G}$ that an algorithm may explore (e.g., the set of all two-layer neural networks). A uniform convergence bound, defined below, may hold over a smaller subset $\mathcal{H}$ of $\mathcal{G}$, eg. the subset of networks with bounded norm.  

\begin{definition}[Uniform Convergence Bound]\label{def:uc}
A \em two-sided \em uniform convergence bound with parameter $\epsilon_{\on{unif}}$ for a problem class $\Omega$, a set of hypotheses $\mathcal{H}$, and loss $\mathcal{L}$ is a bound that guarantees that for any $\mathcal{D} \in \Omega$,
\begin{align}
    \Pr_{S \sim \mathcal{D}^n}[\sup_{h \in \mathcal{H}}|\mathcal{L}_{\mathcal{D}}(h) - \mathcal{L}_{S}(h) | \geq \epsilon_{\on{unif}}] \leq \frac{1}{4}.
\end{align}
A \em one-sided \em uniform convergence bound with parameter $\epsilon_{\on{unif}}$ for a problem class $\Omega$, a set of hypotheses $\mathcal{H}$, and loss $\mathcal{L}$ is a bound that guarantees that for any $\mathcal{D} \in \Omega$, 
\begin{align}
    \Pr_{S \sim \mathcal{D}^n}[\sup_{h \in \mathcal{H}}\mathcal{L}_{\mathcal{D}}(h) - \mathcal{L}_{S}(h) \geq \epsilon_{\on{unif}}] \leq \frac{1}{4}.
\end{align}
\end{definition}

\begin{remark}
More generally one can require a uniform convergence bound to hold with probability $1 - \delta$, but because we are not concerned with the dependence on $\delta$ in our work, we simplify by allowing the bound to fail with probability $1/4$. This is the largest value for which our results go through.
\end{remark}


A uniform convergence bound can be customized to algorithms by choosing $\mathcal{H}$ to depend on the implicit bias of an algorithm. For instance, if an algorithm $\mathcal{A}$ favors low-norm solutions, one could choose $\mathcal{H}$ to be the set of all classifiers with bounded norm. Of course, if $\mathcal{H}$ is too small, it may not be useful for proving generalization, because $\mathcal{A}$ will never output a solution in $\mathcal{H}$. We formalize the notion of choosing a useful algorithm-dependent set $\mathcal{H}$ as follows.


\begin{definition}[Useful Hypothesis Class]\label{def:useful}
A hypothesis class $\mathcal{H}$  is \em useful \em with respect to an algorithm $\mathcal{A}$ over a problem class $\Omega$ if for any $\mathcal{D} \in \Omega$, 
\begin{align}
    \Pr_{S \sim \mathcal{D}^n}[\mathcal{A}(S) \in \mathcal{H}] \geq \frac{3}{4}.
\end{align}
\end{definition}

\begin{remark}
The key difference between our Definitions~\ref{def:uc} and \ref{def:useful} and the definition of Algorithm-Dependent Uniform Convergence bound in \citet{nagarajan2019uniform} is that we require that the bound holds for a class of distributions $\Omega$, and not a single distribution $\mathcal{D}$. We believe this is a natural definition, since for most problems, the designer of the generalization bound would not know in advance the ground truth distribution, but might know that their data comes from some problem class, e.g., linearly separable distributions.
\end{remark}

\begin{remark}
All of our impossibility results concern the broader definition of one-sided uniform convergence bounds. We define two-sided uniform convergence bounds to highlight the difference between our results and the impossibility results of \citet{nagarajan2019uniform}, which consider two-sided UC bounds. 
Thus our conclusions are stronger, but this is only possible because we make a stronger assumption by considering a class of distributions. Note, under the same assumptions as \citet{nagarajan2019uniform} (where $\Omega$ is just a single distribution), our techniques would rule out out two-sided UC bounds, as they do.
\end{remark}

More generally, we can have generalization bounds that do not yield the same generalization guarantee for all elements of $\mathcal{H}$. Instead, their guarantee scales with some property of the hypothesis $h$ and the sample $S$.  We call these \em data-dependent \em bounds. Such bounds are useful if the favorable property is satisfied with high probability by the algorithm of interest.

One specific type of data-dependent bound depends on the margin achieved by the classifier on the training sample. We recall the definition of a margin:

\begin{definition}[Margin]
The \em margin \em $\gamma(h, S)$ of a classifier $h$ on a sample $S$ equals $\min_{(x, y) \in S} yh(x)$.
\end{definition}

 In certain parameterized hypothesis classes it is useful to define a normalized margin. If $f_W$ is $h$-homogeneous, that is, $f_{cW}(x) = c^hf_W(x)$ for any positive scalar $c$, then we define the \em normalized margin \em
\begin{align}
    \bar{\gamma}(f_W, S) := \frac{\gamma(f_W, S)}{\|W\|^h} = \gamma(f_{W/\|W\|}, S),
\end{align}
where we define the norm $\|W\|$ to equal $\sqrt{\mathbb{E}_{i \in [m]}[\|w_i\|^2]}$, where $w_i$ is the $i$'th column of $W$. 

We will use $\gamma^*(S)$ to denote the maximum normalized margin. When we are discussing the linear problem, we let $\gamma^*(S)$ be the max-margin over all vectors $w \in \mathbb{R}^{d}$  with norm $1$, that is $\gamma^*(S) := \sup_{w: \|w\|_2 \leq 1} \gamma(S, f_w)$.  In the XOR problem, we use $\gamma^*(S)$ to denote the max-margin over all weight matrices $W \in \mathbb{R}^{m \times d}$  with norm $1$, that is $\gamma^*(S) := \sup_{W: \|W\| \leq 1} \gamma(S, f_W)$. 

Most classical margin bounds prove that the generalization gap can be bounded by a term that scales inversely linearly or quadratically in the margin~\citep{koltchinskii2002empirical,kakade2009complexity}. Such bounds usually rely on proving uniform convergence for a continuous loss that upper bounds the $0/1$ loss. As we will show in the next section, such bounds are also intrinsically limited in regimes where UC fails on the $0/1$ loss.

In contrast to this, in our work, we prove bounds for classifiers that achieve near-maximal margins. 
\begin{definition}\label{def:maxmargin}
Let $\eps > 0$.  A classifier $h$ is a \emph{$(1 - \eps)$-max-margin solution} for $S$ if 
$ \gamma(h,S) \geq (1 - \epsilon)\gamma^*(S). $
\end{definition}
We refer to a bound that holds for $(1-\epsilon)$-max-margin solutions as a \emph{extremal margin bound.} 

\section{Main Results}\label{sec:results}
In the following section, we state our main results for the linear and quadratic problems, and provide intuition for our findings. As illustrated in Figure~\ref{fig:thresholds}, our results show different possibilities for a near max-margin solution depending on the size of $\kappa := \frac{n}{d\sigma^2}$, a signal-to-noise parameter, where $\sigma$, $d$ are as in Section~\ref{sec:model}. When $\kappa$ is smaller than some threshold $\kgen$ we are not guaranteed to have learning: even a near max-margin solution may not generalize. When $\kappa$ exceeds $\kgen$ by an absolute constant and when $\sigma^2 \ll 1$, our results show that any near max-margin solution generalizes well. Finally, we show that if $\kappa$ is smaller than a second threshold $\kphen$, then uniform convergence approaches will fail to guarantee generalization. 

The exact thresholds $\kgen$ and $\kphen$ depend on the problem class of interest, but in both the linear setting and the nonlinear setting we study, we show that $\kphen > \kgen$. Thus we observe a regime where uniform convergence fails, but generalization still occurs for near max-margin solutions.

For the linear problem, we define the universal constants
\begin{align}\label{def:kgenphen_linear}
    \kgenl := 0 \textup{ and }  \kphen^{\on{linear}} := 1.
\end{align}

For the XOR problem with activation $\on{relu}^h$, for $h \in [1,2)$, we define the constants
\begin{align}\label{def:kgenphen}
    \kgenx := \text{ the solution to }\:\: 2^{\frac{1}{h}}\sqrt{\frac{2}{\kappa}}=\sqrt{\frac{\kappa}{4 + \kappa}}+\sqrt{\frac{16}{\kappa\left(4 + \kappa\right)}} \textup{ and } \kphenx := 4.
\end{align}
The constants are pictured in Figure~\ref{fig:xor}(right) as a function of $h$. Observe that for $h \in (1, 2)$, we have $\kgenx < \kphenx$, and $\kgenx > 0$. When $h = 1$ and the activation is relu, we have $\kgenx = \kphenx$, and thus we do not expect to have a regime where uniform convergence fails, but max-margin solutions generalize. We elaborate more intuitively on why $h > 1$ allows for generalization without UC in Section~\ref{sec:sketch}. 


Our first theorem states that when $\kappa > \kgen$, any near-max-margin solution generalizes.

\begin{restatable}[Extremal-Margin Generalization for Linear Problem]{theorem}{lineargen}\label{thm:linear_gen}Let $\delta > 0$. There exist constants $\epsilon = \epsilon(\delta)$ and $c = c(\delta)$ such that the following holds. For any $n, d, \sigma$ and $\mathcal{D} \in \Omega_{\sigma, d}^{\on{linear}}$ satisfying $\kgenl + \delta \leq \kappa \leq \frac{1}{\delta}$, and $\frac{d}{n} \geq c$, then with probability $1 - 3e^{-n}$ over the randomness of a training set $S \sim \mathcal{D}^n$, for any $w \in \mathbb{R}^d$ that is a $(1 - \epsilon)$-max-margin solution (as in Definition~\ref{def:maxmargin}), we have 
$ \mathcal{L}_{\mathcal{D}}(f_w) \leq e^{-\frac{n}{36d\sigma^4}} + e^{-n/8}.$
\end{restatable}

Attentive readers may observe that since $\kgenl = 0$, Theorem~\ref{thm:linear_gen} can guarantee asymptotic generalization for some sequences of parameters $(n_i, d_i, \sigma_i)_{i \geq 1}$ even when $\kappa_i = \frac{n_i}{d_i \sigma_i^2} = o_{i \rightarrow \infty}(1)$, as long as $\sigma_i^2$ decays fast enough. In Theorem~\ref{thm:linear_full} in the appendix, we state a more detailed version of this theorem which states the exact dependence of $c$ and $\epsilon$ on $\delta$, yielding precise results for $\kappa = o(1)$.

\begin{remark}[Comparison with work on linear classification]\label{rem:lin_class}
Theorem~\ref{thm:linear_gen} is similar to the results of \cite{cao2021risk, wang2020binary, muthukumar2021classification,  shamir2022implicit}, which study the generalization of the max-margin solution in binary linear classification. These works consider more general data distributions than we do, but make stronger overparameterization assumptions. Among the these works, \cite{wang2020binary} requires the mildest overparameterization assumption, that $d \geq \Theta(n \log(n))$. When this assumption is met, their results (restricted to our setting) yield the same generalization for the the max-margin solution as we do.
\end{remark}

We prove a similar generalization result for XOR problem learned on two-layer neural networks.

\begin{restatable}[Extremal-Margin Generalization for XOR on Neural Network]{theorem}{xorgen}\label{thm:xor_gen}
Let $h \in (1,2)$, and let $\delta > 0$. There exist constants $\epsilon = \epsilon(\delta)$ and $c = c(\delta)$ such that the following holds. For any $n, d, \sigma$ and $\mathcal{D} \in \Omega_{\sigma, d}^{\on{XOR}}$ satisfying $\kappa = \frac{n}{d\sigma^2} \geq \kgenx + \delta$ and $\frac{d}{n} \geq c$, then with probability $1 - 3e^{-n/c}$ over the randomness of a training set $S \sim \mathcal{D}^n$, for any two-layer neural network with activation function relu$^h$ and weight matrix $W$ that is a $(1 - \epsilon)$-max-margin solution (as in Definition~\ref{def:maxmargin}), we have
$ \mathcal{L}_{\mathcal{D}}(f_W) \leq e^{-\frac{1}{c\sigma^2}}.$
\end{restatable}
\begin{remark}
So long as $\delta$ is a universal constant (ie. it doesn't depend on $n$, $d$, or $\sigma$), both $c$ and $\epsilon$ are universal constants. While our proofs do not optimize for the dependence of $\epsilon$ and $c$ on $\delta$, an examination of the proofs yields that this the dependence is inversely polynomial in $\frac{1}{\delta}$. 
\end{remark}
This theorem guarantees meaningful results whenever $\sigma$ is small enough. To see this, note that the assumptions of the theorem require that $\frac{d}{n} \in \left[c, \frac{1}{\sigma^2(\kgenx + \delta)}\right]$. If $\sigma$ is small enough (in terms of $\delta$), this interval is non-empty. Further, the generalization guarantee is good if $\sigma$ is small enough (since $\exp(-1/(c\sigma^2))$ tends to $0$ as $\sigma$ approaches $0$).
\mg{Is that helpful? Otherwise see whats commented out below for more explanation.}

If $\kappa < \kgen$, it is possible that a near-max margin solution does not generalize at all. Since $\kgen = 0$ in the linear setting, we only state this result for the XOR problem.

\begin{restatable}[Region where Max-Margin Generalization not Guaranteed]{theorem}{xornogen}\label{thm:xor_no_gen}
Suppose $\kappa < \kgenx$. For any $\epsilon > 0$, there exists a constant $c = c(\kappa, \epsilon)$ such that if $\frac{d}{n} \geq c$, then for any $\mathcal{D} \in \Omega_{\sigma, d}^{\on{XOR}}$, with probability $1 - 3e^{-n/c}$ over $S \sim \mathcal{D}^n$, there exists some $W$ with $\|W\| = 1$ and $\gamma(f_W, S) \geq (1 - \epsilon)\gamma^*(S)$ such that $\mathcal{L}_\mathcal{D}(f_W) = \frac{1}{2}$.
\end{restatable}

Theorems~\ref{thm:xor_gen} and \ref{thm:xor_no_gen} demonstrate that in the XOR problem, there is a threshold in $\kappa$ above which generalization occurs. As long as $\kappa$ is above this threshold, we achieve generalization when $\sigma^2 \ll 1$.

The next theorem states that when $\kappa < \kphen$,  uniform convergence bounds over useful hypothesis classes will be vacuous, that is, their generalization guarantee must be arbitrarily close to $1$. For brevity, we state our results for the linear and XOR neural network settings together; we state the more complicated XOR result in full and then mention how the linear result differs.

\begin{restatable}[One sided UC Bounds are Vacuous]{theorem}{linearphen}\label{thm:one_sided_impossible_linear} 
Fix $h \in (1, 2)$, and suppose $\kgenx < \kappa < \kphenx$. For any $\delta > 0$, there exist strictly positive constants $\epsilon = \epsilon(\kappa, \delta)$ and $c = c(\kappa, \delta)$ such that the following holds. Let $\mathcal{A}$ be any algorithm that outputs a $(1-\epsilon)$-max-margin two-layer neural network $f_W$ for any $S \in (\mathbb{R}^d \times \{1,-1\})^n$. Let $\mathcal{H}$ be any concept class that is useful for $\mathcal{A}$ on $\Omega_{\sigma, d}^{h, \on{XOR}}$ (as in Definition~\ref{def:useful}).
Suppose that $\epsilon_{\on{unif}}$ is a uniform convergence bound for the XOR problem $\Omega_{\sigma, d}^{h, \on{XOR}}$: that is, for any $\mathcal{D} \in \Omega_{\sigma, d}^{h, \on{XOR}}$,  $\epsilon_{\on{unif}}$ satisfies $$\Pr_{S \sim \mathcal{D}^n}[ \sup_{h \in \mathcal{H}} \mathcal{L}_{\mathcal{D}}(h) - \mathcal{L}_S(h)  \geq \epsilon_{\on{unif}} ] \leq 1/4.$$
Then if $\frac{d}{n} \geq c$ and $n > c$ we must have $\epsilon_{\on{unif}} \geq 1 - \delta.$
 
A similar result holds for the linear problem with $\kgenl < \kappa < \kphenl$ and the problem class $\Omega_{\sigma, d}^{\on{linear}}$. In this case the result holds for a universal constant $c$ (which does not depend on $\kappa$), for $\eps \leq \frac{\kappa(\kphenl - \kappa)^2}{c}$, and for $\frac{d}{n} \geq \frac{c}{\kappa^2(\kphenl - \kappa)^4}$. We achieve the guarantee that $\eps_{\on{unif}} \geq 1 - e^{-\frac{n}{36d\sigma^2}} - exp(-n/8)$.
\end{restatable}

\begin{remark}[Comparison to \cite{nagarajan2019uniform}]
Theorem~\ref{thm:one_sided_impossible_linear} assumes that the concept class $\mc{H}$ must be useful for all of $\Omega_{\sigma, d}^{h, \on{XOR}}$ (resp. 
$\Omega_{\sigma, d}^{\on{linear}}$). We state our theorem this way to parallel our upper bounds which yield bounds which are useful for all of $\Omega_{\sigma, d}^{h, \on{XOR}}$ or
$\Omega_{\sigma, d}^{\on{linear}}$. In contrast, \citet{nagarajan2019uniform} only requires that $\mathcal{H}$ is useful for a single distribution. It is immediate from our proof that our result still holds under the weaker assumption that $\mc{H}$ is useful for an $\Omega$ containing only two distributions with opposite ground truths: $\mc{D}_{\mu, \sigma, d}$ and $\mc{D}_{\mu, \sigma, d}$ in the linear case and $\mc{D}_{\mu_1, \mu_2, \sigma, d}$ and $\mc{D}_{\mu_2, \mu_1, \sigma, d}$ in the neural network case. Note also, if we only assume $\mc{H}$ is useful for a single distribution, our techniques yield the same impossibility result as \citet{nagarajan2019uniform}, which rules out two-sided UC bounds. (This requires a slight modification to our proofs).

The same observation that $\Omega$ needs to contain only two distributions holds also for the next two theorems (\ref{thm:margin_impossible_linear} and \ref{thm:margin_impossible}) on the impossibility of polynomial margin bounds.
\end{remark}


In addition to ruling out uniform convergence bounds, we can also show the limitations of margin bounds which achieve an error that scales inversely polynomially with $\gamma(h, S)$. 
The following results state that if $\kappa < \kphen$, then certain types of margin bounds cannot yield better that constant test loss on even the max-margin solution.

\mg{Work on the paragraph below.}
The crux of the following two theorems is showing that there exists a ``bad" solution with a large classification margin (a constant fraction times the max-margin) which does not generalize. This is harder to do than just finding a a ``bad'' solution with zero classification error that doesn't generalize, which is the crux of Theorem~\ref{thm:one_sided_impossible_linear}. Ultimately, we are able to find such a solution by showing that the generalizing component of the max-margin solution is substantially smaller than the overfitting component in the ``junk'' feature directions.

\begin{restatable}[Polynomial Margin Bounds Fail for Linear Problem]{theorem}{linmargin}\label{thm:margin_impossible_linear} 
Suppose $\kgenl < \kappa < \kphenl$. There exists a universal constant $c$ such that the following holds. Let $\epsilon =  \frac{\kappa(\kphen - \kappa)^2}{c}$, and let $\mathcal{A}$ be any algorithm so that $\mathcal{A}(S)$ outputs a $(1-\epsilon)$-max-margin solution $f_w$ for any $S \in (\mathbb{R}^d \times \{1,-1\})^n$.
Let $\mathcal{H}$ be any concept class that is useful for $\mathcal{A}$ (as in Definition~\ref{def:useful}).
Suppose that there exists an polynomial margin bound of integer degree $p$ for the linear problem $\Omega_{\sigma, d}^{\on{linear}}$: that is, for any $\mathcal{D} \in \Omega_{\sigma, d}^{\on{linear}}$,  there is some $G$ that satisfies
$$\Pr_{S \sim \mathcal{D}^n}\left[\sup_{h \in \mathcal{H}} \mathcal{L}_{\mathcal{D}}(h) - \mathcal{L}_S(h)\geq \frac{G}{\gamma(h, S)^p}\right]  \leq \frac{1}{4}$$
Then for any $\mathcal{D} \in \Omega_{\sigma, d}^{\on{linear}}$, if $\frac{d}{n} \geq \frac{c}{\kappa^2(\kphen - \kappa)^4}$, with probability $\frac{1}{2} - 3e^{-n}$ over $S \sim \mathcal{D}^n$, the margin bound is weak even on the max-margin solution, that is, $\frac{G}{\gamma^*(S)^p} \geq \max\left(\frac{1}{c}, 1 - e^{-\frac{\kappa}{36\sigma^2}} - e^{-n/8} - \frac{3\kappa}{c}\right)^p$, which is more than an absolute constant.
\end{restatable}
This theorem says that even on the max-margin solution, the margin guarantee can be no better than a constant whenever $\kappa < \kphenl$. That is, no polynomial margin bound will be able to show that the test error of the max-margin solution is less than an absolute constant. We know however from Theorem~\ref{thm:linear_gen} that in this same regime, the test error of the max-margin solution can be arbitrarily small for small enough $\sigma$. Thus the any polynomial margin bound cannot predict this behaviour. \mg{Cut some of this?}

The attentive reader again may notice that if $\kappa \rightarrow 0$ as $n$ and $d$ grow, but generalization occurs, any such margin bound is vacuous, in that $\frac{G}{\gamma^*(S)^p} \rightarrow 1$.

We achieve a similar result in the XOR setting.
\begin{restatable}[Polynomial Margin Bounds Fail for XOR on Neural Network]{theorem}{xormargin}\label{thm:margin_impossible}
Fix an integer $p \geq 1$, and suppose $\kgenx < \kappa < \kphenx$. For any $\epsilon > 0$, there exists $c = c(\kappa, p, \epsilon)$ such that the following holds. Let $\mathcal{H}$ be any hypothesis class such that for all $\mathcal{D} \in \Omega_{\sigma, d}^{\on{XOR}}$,
$$ \Pr_{S \sim \mathcal{D}^n}[ \text{all $(1-\eps)$-max-margin two-layer neural networks $f_W$ for $S$ lie in $\mathcal{H}$}] \geq 3/4. $$
Suppose that there exists an polynomial margin bound of degree $p$ for the XOR problem $\Omega_{\sigma, d}^{\on{XOR}}$: that is, for any $\mathcal{D} \in \Omega_{\sigma, d}^{\on{XOR}}$, there is some $G$ that satisfies
$$ \Pr_{S \sim \mathcal{D}^n}\left[\sup_{h \in \mathcal{H}} \mathcal{L}_{\mathcal{D}}(h) - \mathcal{L}_S(h)\geq \frac{G}{\gamma(h, S)^p}\right]  \leq \frac{1}{4}.$$
Then for any $\mathcal{D} \in \Omega_{\sigma, d}^{\on{XOR}}$, if $\frac{d}{n} \geq c$ and $n \geq c$, with probability $\frac{1}{2} - 3e^{-n/c}$
over $S \sim \mathcal{D}^n$, on the max-margin solution, the generalization guarantee is no better than $\frac{1}{c}$, that is, $\frac{G}{\gamma^*(S)^p} \geq \frac{1}{c}$.
\end{restatable}
\begin{remark}
The polynomial margin impossibility results is slightly weaker for the XOR problem. Namely, the hypothesis class $\mathcal{H}$ we consider is larger in the XOR problem: it must contain with probability $\frac{3}{4}$ \em any \em near max-margin solution, instead of just the one output by $\mathcal{A}$.
\end{remark}
The combination of our generalization results and our margin possibility results suggest a phase transition in how the margin size affects generalization. If the margin is near-maximal, Theorems~\ref{thm:linear_gen} and \ref{thm:xor_gen} show that we achieve generalization. Meanwhile, the proof of Theorems~\ref{thm:margin_impossible_linear} and \ref{thm:margin_impossible} suggest that solutions achieving a constant factor of the maximum margin may not generalize.

The proofs of all of our results concerning the linear problem are given in Section~\ref{sec:linear_proofs}. The proofs for the XOR problem are in Section~\ref{sec:xor_proofs}.


\newcommand{\wg}{w_{\textup{g}}}
\newcommand{\wb}{w_{\textup{b}}}

\paragraph{Key intuitions.} We demonstrate the gist of the analysis for the linear problems with some simplifications. It turns out that two special solutions merit particular attention: (i) the good solution $\wg = \mu$ that generalizes perfectly, and (ii) the bad overfitting solution $\wb = \frac{1}{\sqrt{nd}\sigma}\sum_{j} y_j \xi_j$ that memorizes the ``junk'' dimension of the data. In the this section, we assume that the dimension $d$ is large enough relative to $n$ such that we have $\xi_i^T\wb \approx \frac{1}{\sqrt{nd}\sigma}y_i|\xi_i|^2 = y_i\sqrt{\frac{d\sigma^2}{n}}$ for all $i$. \footnote{This occurs when $d \gtrsim n\log(n)$. For smaller $d$, we can instead choose $\wb$ to be the min-norm vector satisfying $\xi_i^T\wb = y_i$ for all $i$, and the intutions proceed identically.}
\mg{If we want to put in some intuition for why we need $d > cn$, this would be the place.}
We examine the margin of the two solutions and have
\begin{align}\label{eq:bad_good}
\bar{\gamma}(\wg, S) = 1  \textup{ and }\bar{\gamma}(\wb,S) \approx \sqrt{\frac{d\sigma^2}{n}}.
\end{align}

At first glance, one might conclude that when $\bar{\gamma}(\wg, S) < \bar{\gamma}(\wb, S)$, the max margin solution will be $\wb$, which does not generalize. 
However, our key observation is that any (near) max margin solution $w$ always contains a mixture of both $\wg$ and $\wb$. When the $\wg$ component is small but non-trivial and the $\wb$ component is large, the solution can simultaneously generalize but contain a large enough overfitting component to preclude UC.

More concretely, suppose we consider the margin of a linear mixture $w = \alpha \wg + \beta \wb$ satisfying $\alpha^2+\beta^2=1$ so that $\|w\|_2 =1$. It is easy to see that the margin on the training set is 
\begin{align}\label{eq:margins_add}
\bar{\gamma}(w, S) = \alpha \bar{\gamma}(\wg, S) + \beta\bar{\gamma}(\wb, S)
\end{align}
Meanwhile, the margin on an test example $x$ is only slightly affected by $\wb$:
\begin{align}
\bar{\gamma}(w, x) \approx \alpha \bar{\gamma}(\wg, S) \pm \beta \wb^Tx \approx \alpha\bar{\gamma}(\wg, S) \pm \beta\bar{\gamma}(\wb, S)\sqrt{\frac{n}{d}}.
\end{align}
The effect $\wb^Tx$ of the bad solution on the test sample is is smaller than $\bar{\gamma}(w, S)$ by a $\sqrt{\frac{n}{d}}$ factor because $x$ is a high dimensional random vector, and thus mostly orthogonal to $\wb$. 
\mg{Optional: Indeed, $(\sum_j y_j \xi_j)^T x \approx \pm \frac{|\sum_j y_j \xi_j||\xi|}{\sqrt{d}} \approx \frac{\sqrt{n}|\xi|^2}{\sqrt{d}}$, while for any $j$, $(\sum_j y_j \xi_j)^T \xi_j \approx |\xi|^2.$}
Therefore, even if the margin on the training set mostly stems from the bad overfitting solution, that is, $\alpha \bar{\gamma}(\wg, S)  < \beta\bar{\gamma}(\wb, S)$, the model may still generalize as long as $\alpha \bar{\gamma}(\wg, S) \ge  \beta\bar{\gamma}(\wb, S)\sqrt{\frac{n}{d}}$.


The optimal $\alpha,\beta$ satisfying $\alpha^2 + \beta^2 = 1$ that maximize the margin turns out to be proportional to the original margin: $\frac{\alpha}{\beta} =  \frac{\bar{\gamma}(\wg, S)}{\bar{\gamma}(\wb, S)}$. Therefore, we have 
$\frac{\alpha \bar{\gamma}(\wg, S)}{\beta\bar{\gamma}(\wb, S)} = \frac{\bar{\gamma}(\wg, S)^2}{\bar{\gamma}(\wb, S)^2}$. In other words, we should expect reasonable generalization of near-max margin solutions as long as $\frac{\bar{\gamma}(\wg, S)}{\bar{\gamma}(\wb, S)} > (\frac{n}{d})^{1/4}$, which by equation~\ref{eq:bad_good} occurs whenever $\frac{n}{d\sigma^4} \gg 1$. 

We now explain why UC will fail when $\bar{\gamma}(\wg, S) < \bar{\gamma}(\wb,S)$. With high probability over $S \sim \mathcal{D}^n$, there is a ``oppositite''
dataset $S'$,\footnote{This opposite-mapping is similar to the phenomenon described in the work of \citet{nagarajan2019uniform} with some modifications to show impossibility results for one-sided UC bounds.} 
which flips the signal component of each example (namely $S \ni (y_j \mu + \xi_j, y_j) \mapsto (-y_j \mu + \xi_j, y_j) \in S'$), where: (1) the max-margin classifier $f_{w'}$ learned on $S'$ correctly classifies $S$ but (2) the test error $\mathcal{L}_{\mathcal{D}}(f_{w'})$ is large.

In more detail, because we are in a regime where the classifier mostly memorizes the ``junk'' dimension $\xi$ instead of fitting to the signal, and the junk dimension is the same in $S$ and $S'$, $f_{w'}$ will correctly classify $S'$. However, $f_{w'}$ should generalize on the corresponding ``opposite'' distribution $\mathcal{D}'$, and thus the test error on $\mathcal{D}$ will be large.

A uniform convergence bound with parameter $\eps_{\on{unif}}$ should satisfy on sets $S$: $\eps_{\on{unif}} \geq \mathcal{L}_{\mathcal{D}}(f_{w'}) - \mathcal{L}_S(f_{w'})$. This $\eps_{\on{unif}}$ large because $\mathcal{L}_{\mathcal{D}}(f_{w'})$ is large and $\mathcal{L}_S(f_{w'})$ is zero.

\begin{remark}[Comparison with benign overfitting]
\label{remark:benign}
In both benign overfitting settings and our setting,  memorizing the data and learning the signal coexist, but there are some key differences. Benign overfitting in a linear regression setting (e.g. ~\cite{bartlett2020benign}) typically focuses on the minimum norm \em interpolant, \em which may not be the ideal solution in the presence of noisy data. For instance, if instead we use a $\ell_2$-regularized loss to favor smaller norm solutions, overfitting can be avoided. Such a trade-off is not possible in our setting with clean data, because the overfitting component already has the smallest norm (with the constraint that average margin $\ge 1$), and the generalizing component has a larger norm. This property provably dooms any UC bound, and is fundamentally caused by the lack of data.   
\end{remark}

\section{Conclusion}
In this work, we give novel generalization bounds in settings where uniform convergence provably fails. We use a unified approach of leveraging the extremal margin in both a linear classification setting and a non-linear two-layer neural network setting. Our work provides insight on why memorization can coexist with generalization. 

Going beyond our results, it is important to find broader tools for understanding the regime near the boundary of generalization and no generalization. We conclude with several concrete open directions in this vein. One question is how to prove generalization without UC when $d < n$, but the model itself (e.g. a neural network) is overparameterized, and thus can still overfit to the point of UC failing. A second direction asks if we can prove similar results in the non-linear network setting for the solution found by gradient descent, if this solution is not a near max-margin solution. Indeed, in the non-convex optimization landscape, it not guaranteed that that a max-margin solution will be found by gradient descent.

\section*{Acknowledgments}
We thank Jason Lee for helpful discussions. MW acknowledges the support of NSF Grant CCF-1844628 and a Sloan Research Fellowship. TM is supported by NSF IIS 2045685.

\small
\bibliographystyle{plainnat}
\bibliography{all,secondary}

\begin{appendices}
\listofappendices

\section{Proof Overview}\label{sec:sketch}

In our proof overview, we focus on the linear problem. While basic steps and intuitions remain the same for the more complicated neural network problem, we add explanation of where we need additional techniques or insights.

The starting observation is that any solution $w$ can be decomposed into a \em signal \em component and a \em overfitting \em component. For the linear problem, lets call those components $u$ and $v$ respectively, where $u$ is in the subspace containing $\mu$, and $v$ is orthogonal to $\mu$, such that $w = u + v$. Conveniently in the linear problem, we have $f_w(x) = w^Tx = f_{u}(x) + f_{v}(x)$. The proof of our main results can be divided into three main parts, which are sketched in the next three subsections.

\subsection{The overfitting component only slightly affects generalization.} Since the ``junk'' features (orthognal to $\mu$) are high dimensional and have smaller variance, on a random new sample drawn from the population distribution $\mathcal{D}$, we have $f_w(x) \approx f_{u}(x)$. We formalize the affect on generalization in the following lemma, which shows that non-trivial generalization can occur so long as $\|v\|_2 \lesssim \frac{1}{\sigma}u^T\mu$. Notice that if $\sigma^2 \ll 1$, this means the $v$ component can be a great deal larger than the $u$ component without affecting generalization.

\begin{restatable}{lemma}{linsmall}\label{lemma:linear_spur_small_2}
Fix a distribution $\mathcal{D}_{\mu, \sigma, d} \in \Omega_{\sigma, d}^{\on{linear}}$. For $w \in \mathbb{R}^d$, let $w = u + v$ as above. Let $q := \frac{u^T\mu}{\|v\|_2}$. Then for $x \sim \mathcal{D}_{\mu, \sigma, d}$,
$$\Pr[|f_v(x)| \geq yf_u(x)] \leq 2e^{-\frac{q^2}{8\sigma^2}} + \exp(-8d).$$
\end{restatable}

We show a similar lemma for the two-layer neural network, proved in Section~\ref{sec:xor_technical_proofs_spurious}.
Recall that for a two-layer  neural network with weight matrix $W \in \mathbb{R}^{m \times d}$, we define $f_W(x) = \mathbb{E}_{i \in [m]}a_i\phi(w_i^Tx)$, where the weights $a_i$ are fixed, and $w_i$ is the $i$th row of $W$. Recall that $\phi(z) = \max(0,z)^h$ for $h \in (1, 2)$.

\begin{restatable}{lemma}{xorsmall}\label{lemma:xor_spur_small}
Fix a distribution $\dfull \in \omegax$. For $W \in \mathbb{R}^{m \times d}$, let $W = U + V$ where $V$ is orthogonal to the subspace containing $\mu_1$ and $\mu_2$.  Then for some universal constant $c$, for any $t \geq 1$, with probability at least $1 - e^{-ct}$, on a random sample $x \sim \mathcal{D}_{\mu_1, \mu_2, \sigma, d}$,
   $$ |f_{W}(x) - f_{U}(x)| \leq \left(8\|U\| + 3\right)(t + 1)\sigma^2\|V\|^2 + 2\left((t + 1)\sigma^2\|V\|^2\right)^{\frac{h}{2}}.$$
\end{restatable}

\subsection{Any near-max-margin solution should leverage both signal and overfitting components. }

A max-margin solution aims to maximize the minimum margin of any training example while holding the norm of the solution constant. To get a sense of why such a solution must leverage both signal and overfitting components, we consider first what would happen if we used a signal or overfitting component alone.

In the linear problem, setting $v$ to be zero and only considering the signal direction, it is easy to check that the max-margin solution is achieved by setting $u = \mu$, leading to a margin $\gamma_g$ of $1$. We call this good solution $\wg$.

If we optimize in the direction orthogonal to $\mu$ alone and set $u$ to zero, the max-margin solution can be shown to found by choosing $v$ to be very near the vector $\frac{\sum_j y_j \xi_j}{\|\sum_j y_j \xi_j\|_2}$, which achieves a margin $\gamma_b$ of roughly $\frac{\sigma}{\sqrt{n}}$. We call this bad overfitting solution $\wg$.

Depending on the choice of $\sigma$, the margin $\gamma_b$ might be larger that $\gamma_g$. Fortunately, this does not preclude generalization. Indeed, we will show that combining these two solutions achieves an even larger margin! Consider constructing the solution $\hat{w} = \alpha \wg + \beta \wb$ and $\alpha^2 + \beta^2 = 1$. It is easy to check since $f_{\hat{w}}(x) = f_{u}(x) + f_{v}(x)$ that $\hat{w}$ achieves a margin of $\alpha + \beta$. The following simple optimization program characterizes the optimal trade-off between $\alpha$ and $\beta$:
\begin{align}\label{eq:linear_opt}
    \max \: &\alpha\gamma_g + \beta\gamma_b \\
    \alpha^2 &+ \beta^2 \leq 1.
\end{align}
Analyzing this program shows that optimal values are achieved by choosing $\frac{\alpha}{\beta} = \frac{\gamma_g}{\gamma_b}$, which suggests that the max-margin solution will include a significant component both of $\wg$ and $\wb$. 

While this alone is not enough to prove that \em any \em near-max-margin solution has a significant component of $\mu$, we can extend this argument to show that any solution that achieves a margin larger than $\gamma_b$ must include some component in the signal direction, and in particular, this component must be in the $\mu$-direction. This is formalized in Lemma~\ref{linear:tech_lemma}

The linear problem had two nice properties which unfortunately we will not be able to leverage in the non-linear problem:

\begin{enumerate}
    \item It is easy to understand the affect of linearly combining solutions from the signal space and the junk space. That is, for any $w = u + v$, we have $f_w(x) = f_{u}(x) + f_{v}(x)$.
    \item Any component in the signal subspace which improves the margin is guaranteed to stand alone as a good enough solution. That is, for any $v$, if the margin of $f_{u+ v}$ is better than the margin of $f_v$, then it must be the case that $f_u$ generalizes. This is because any component in the signal subspace which improves the margin must be in the direction of $\mu$.
\end{enumerate}

In the XOR problem, the failure of (1) to hold is challenging because it turns out that adding the max-margin solutions in the signal space (which we call the ``good'' solution $W_g$) to the max-margin solution in the orthogonal space (which we call the ``bad'' solution $W_b$) has the affect of partially \em cancelling \em each other's margins out. \mg{Can go into more detail here to show this, but maybe later is better.} Ultimately, we will resolve this by showing that we can construct an alternate bad solution $W_b'$ that does not cancel out the margin at all when combined with the good solution $W_g$. Even though alone $W_b'$ has a slightly worse margin than $W_b$, we show that the benefit from combining $W_g$ and $W_b'$ overcomes this loss. This benefit scales with the convexity of the activation $\phi$ is the positive region. This is why we require that $h$, the power of the activation, is strictly greater than $1$.

The failure of (2) to hold is challenging because there are many weight matrices $U$ in the $\mu_1, \mu_2$ subspace that may improve the margin, but not stand alone well. For instance, $U$ might just use the $\mu_1$ direction and not the $\mu_2$ direction, and thus still improve the margin, but not stand well alone. Fortunately, we can rule out this behavior by arguing, similarly to before, that any solution that does not use both signal directions equally cannot exceed a certain (non-maximal) margin. This argument takes the form of a series of lemmas presented in section~\ref{sec:reduction}. \mg{Seems too messy to put in line; though there might be a creative way to rephrase it.}

Even more challenging however is the potential to have a component $W_1$ that uses its component in the $u_1$-direction (or analogously in the $u_2$-direction) \em only to improve the margin on points from one of the two positive clusters \em (see Figure~\ref{fig:xor}). We are able to rule out this behavior by reducing our understanding of max-margin solutions on the neural net to a simpler 3-variable optimization problem which concerns a single neuron $w \in \mathbb{R}^d$ and a pair of training examples $x_j$ and $x_{j'}$ from the two positive clusters (see Figure~\ref{fig:xor}(left)), that is, $x_j = \mu_1 + \xi_j$, and $x_{j'} = -\mu_1 + \xi_{j'}$.

\begin{definition}[Trivariate Subproblem]\label{trivariate_sketch}
\begin{align}
    \max \: & \phi(b + c) +  \phi(-b + d): \\
    & b^2 + \frac{\kappa}{4} (c^2 + d^2) \leq 1
\end{align}
\end{definition}
In this problem, the variable $b$ represents $\mu^Tw$, the strength of the signal component. The variable $c$ represents $w^T\xi_j$, and $d$ represents $w^T\xi_{j'}$, the strengths of the overfitting components. This $3$-variable optimization problem can be viewed as an analog of the optimization problem in eq.~\ref{eq:linear_opt}. 

The following lemma argues that if the neuron $w$ is on average good for both $x_j$ and $x_{j'}$, ie., the objective is large, then \em for both \em $x_j$ and $x_{j'}$, a constant fraction of the activation must be explained by the component of $w$ in the $\mu_1$ direction.

\begin{lemma}[Simplification of Lemma~\ref{lemma:trivariate_analysis}]
Suppose $\phi(x) = \max(0, x)^{h}$ for $1 < h < 2$ and $\kappa > \kgen^{\text{XOR}, h}$. Then for any a sufficiently small constant $\epsilon = \epsilon(\kappa)$, any $(1 - \epsilon)$-optimal solution to the program in Definition~\ref{trivariate_sketch} satisfies
\begin{enumerate}
    \item $\phi(b) \geq \Omega(1)\phi(b + c)$; and
    \item $\phi(-b) \geq \Omega(1)\phi(-b + d)$.
\end{enumerate}
\end{lemma}
We note that this lemma is also the key step in arguing that signal solutions and in the overfitting solutions can be combined efficiently enough that using a non-zero component in the signal-subspace is effective.

In Section~\ref{sec:xor_proofs}, we flesh out this argument in detail, making rigorous the reduction from the full neural net and the full data-set to a the sub-problem in Definition~\ref{trivariate}.

\subsection{Any near-max-margin solution should have a large enough overfitting component to preclude UC.}

In order to show the failure of standard uniform convergence bounds on the problems we consider, we must argue that the overfitting component of any near-max-margin solution roughly exceeds the size of the signal component, as outlined in the Intuition section. The argument described in the previous subsection guarantees this: for this linear problem it is enough to show that $\gamma_b \geq \gamma_g$, and for the non-linear problem, we will need to leverage the full version of  Lemma~\ref{lemma:trivariate_analysis}, which shows that $c$ and $d$ are large relative to $b$.

A larger overfitting component than signal component ensures the phenomenon similar to the one described in the work of Nagarajan and Kolter: if the signal component of the data were changed, but the junk feature stayed the same, the data would still be classified correctly by the classifier learned on the original data. Showing this phenomenon is the key step in proving Theorems~\ref{thm:one_sided_impossible_linear}. To prove the margin lower bounds, we also need to show that the this ``opposite'' dataset is still classified with a large margin by the original classifier.

\section{Auxiliary Lemmas}\label{sec:lemmas}
If $M$ is a matrix, we use $\|M\|_2$ to denote the spectral norm of $M$.

We will use the following lemma throughout, on the concentration of random covariance matrices in both the linear and XOR problems.
\begin{lemma}[Concentration of Random Covariance Matrix]\label{lemma:concentration}
There exists a universal constant $C_{\ref{lemma:concentration}}$ such that the following holds. Suppose $d > C_{\ref{lemma:concentration}}^2n$ and $\Xi \in \mathbb{R}^{d \times n}$, where each column of $\Xi$ is a random vector distributed uniformly on the sphere of radius $\sqrt{d}\sigma$ in $d$ dimensions. There exists a universal constant $C_{\ref{lemma:concentration}}$ such that with probability at least $1 - 3e^{-n}$ the following two events hold:
\begin{enumerate}
    \item $\|\frac{1}{\sigma^2d}\Xi^T\Xi - I\|_2 \leq C_{\ref{lemma:concentration}}\sqrt{\frac{n}{d}}$
    \item For any $c \in \mathbb{R}^n$, if $v \in \mathbb{R}^d$ is the minimum-norm vector $v$ satisfying $\Xi^Tv = c$ exists and has $\|v\|_2^2 \in \frac{\|c\|_2^2}{\sigma^2d}\left[\frac{1}{1 + C_{\ref{lemma:concentration}}\sqrt{\frac{n}{d}}}, \frac{1}{1 - C_{\ref{lemma:concentration}}\sqrt{\frac{n}{d}}}\right]$. 
\end{enumerate}
The lemma still holds if the rows of $\Xi$ are instead drawn i.i.d. from a sphere of dimension $d - a$ of radius $\sqrt{d - a}\sigma$ in any subspace of dimension $d - a$, for $a \in \{1, 2\}$.
\end{lemma}
\begin{proof}
We prove this using a similar result for matrices with i.i.d. entries. Observe that $\Xi = ZD$, where $Z \in \mathbb{R}^{d \times n}$ is a random matrix with i.i.d. normal entries from $\mathcal{N}(0, \sigma^2)$, and $D \in \mathbb{R}^{n \times n}$ is the diagonal matrix with $D_{jj} = \frac{\sigma\sqrt{d}}{\|z_j\|_2}$, where $z_j$ is the $j$th column of $Z$.

Thus 
\begin{align}
    \left\|\frac{1}{\sigma^2d}\Xi^T\Xi - I\right\|_2 &=  \left\|\frac{1}{\sigma^2d}DZ^TZD - I\right\|_2\\
    &\leq \|D\|_2^2\left\|\frac{1}{\sigma^2d}Z^TZ - D^{-2}\right\|_2\\
    &\leq  \|D\|_2^2\left(\left\|\frac{1}{\sigma^2d}Z^TZ - I\right\|_2 + \left\|D^{-2} - I\right\|_2\right)
\end{align}
\citet{vershynin2018high} (see Theorem 3.1.1 and the discussion thereafter), states that for any $j$, for some universal constant $C$, with probability $1 - e^{-2n}$, $|\frac{\|z_j\|^2}{\sigma^2 d} - 1| \leq C\sqrt{\frac{n}{d}}$. \citet{vershynin2018high} (Ex. 4.7.3) also guarantees that with probability $1 - 2e^{-n}$, $\left\|\frac{1}{\sigma^2d}Z^TZ - I\right\|_2 \leq C\sqrt{\frac{n}{d}}$. Unioning over the the first event occurring for all $j$, and the matrix concentration even happening, we have (for some new constant $C$), with probability 
\begin{align}
    1 - 2e^{-n} - ne^{-2n} \geq 1 - 3e^{-n},
\end{align}
$\left\|\frac{1}{\sigma^2d}\Xi^T\Xi - I\right\|_2 \leq C\sqrt{\frac{n}{d}}$.

For the second conclusion, it well-known that the min-norm solution to overparameterized linear regression problem $X^Ta = b$ satisfies $\argmin_{a : X^Ta = b} \|a\|_2 = X(X^TX)^{\dagger}b$, where $\dagger$ denotes the pseudo-inverse (see eg. \citet{bartlett2020benign}, page 5). Observe that $\Xi^T\Xi$ is invertible since $d > C_{\ref{lemma:concentration}}^2n$ and thus $C_{\ref{lemma:concentration}}\sqrt{\frac{n}{d}} < 1$. It follows by that we can solve explicitly for $v$, yielding $v = \Xi(\Xi^T\Xi)^{-1}c$. Hence,
\begin{align}
    \|v\|_2^2 = c^T(\Xi^T\Xi)^{-1}\Xi^T\Xi(\Xi^T\Xi)^{-1}c = c^T(\Xi^T\Xi)^{-1}c \in \frac{\|c\|_2^2}{\sigma^2d}\left[\frac{1}{1 + C_{\ref{lemma:concentration}}\sqrt{\frac{n}{d}}}, \frac{1}{1 - C_{\ref{lemma:concentration}}\sqrt{\frac{n}{d}}}\right].
\end{align}
To see that the lemma applies if the rows of $\Xi$ are orthogonal to some subspace, it suffices to assume the subspace is spanned by $e_1$ or $e_1$ and $e_2$. Thus one can view the matrix $\Xi$ as being in $\mathbb{R}^{d - a \times n}$, and the conclusion follows by replacing $d$ with $d - a$. The constant $C_{\ref{lemma:concentration}}$ can be adjusted so that the conclusion written holds for $a \in \{1, 2\}$.
\end{proof}

\begin{lemma}\label{lem:activation}
Let $\phi(x) = \max(0, x)^h$ for $h \in (1, 2)$. For any values $s, t$, we have $\phi(s + t) \leq (\phi(s) + \phi(t))2^{h-1}$
\end{lemma}
\begin{proof}
By homogeneity of $\phi$,
\begin{align}
    \frac{\phi(s) + \phi(t)}{\phi(s + t)} = \frac{\phi(\frac{s}{s + t}) + \phi(\frac{t}{s + t})}{\phi(1)},
\end{align}
so we need to show that $\phi(a) + \phi(b) \geq 2^{1 - h}$ for any $a + b = 1$. Since $\phi$ is convex, subject to this linear constraint, by the KKT condition, the minimum is attained when $\phi'(a) = \phi'(b)$ which occurs when $a = b = \frac{1}{2}$. (Note, we cannot have $\phi'(a) = \phi'(b) = 0$, since at least one of $a$ and $b$ must be positive.
\end{proof}

\section{Proofs for Linear Problem}\label{sec:linear_proofs}

Throughout this section, since we are only concerned with the linear problem, we will abbreviate $\Omega = \omegal$, $\kphen = \kphenl$, and $\kgen = \kgenl$.

\subsection{Technical Lemmas}

Throughout the following subsection, we assume $\mathcal{D}_{\mu, \sigma, d}$ is fixed. For a vector $w \in \mathbb{R}^d$, we let $u = \mu \mu^Tw$ and $v = (I - \mu\mu^T)w$, such that $w = u + v$.

Our training data is given by the matrix $(X, y)$, where $X \in \mathbb{R}^{d \times n}$ and $y \in \mathbb{R}^n$, where $x_j = Xe_j$, and $(x_j, y_j)$ denotes the $j$th training sample. Recall that we have $x_j = z_j + \xi_j$, where $z_j = \mu y_j$, $y_j \sim \on{Uniform}(-1, 1)$, and $\xi_j$ is uniformly distributed on the sphere of radius $\sigma\sqrt{d - 1}$ in the $d - 1$ dimensions orthogonal to $\mu$. We use $\Xi \in \mathbb{R}^{d \times n}$ to denote the matrix with columns $\xi_j$.

\begin{lemma}\label{lem:min}
On the event that the conclusion of Lemma~\ref{lemma:concentration} holds for $\Xi$, for any $v$, 
\begin{equation}
     \min_j y_jv^T\xi_j \leq \frac{\|v\|_2\sqrt{\left(1 + C_{\ref{lemma:concentration}}\sqrt{\frac{n}{d}}\right)}}{\sqrt{\kappa}}
\end{equation}
\end{lemma}
\begin{proof}
Let $\gamma := \min_j y_jv^T\xi_j$ and suppose $\gamma > 0$; otherwise the lemma is immediate. Observe that $\|v\|_2$ must be larger than the norm of least norm vector (call it $w$) achieving $w^T\xi_jy_j \geq \gamma$ for all $j$, which implies that $|w^T\xi_j| \geq \gamma$ for all $j$. Let $c := w^T\Xi$, such that $\|c\|^2 \geq n\gamma^2$. By Lemma~\ref{lemma:concentration}, on this event of the lemma, 

\begin{align}
    \|v\|^2 \geq \|w\|^2 \geq \frac{1}{1 + C_{\ref{lemma:concentration}}\sqrt{\frac{n}{d}}}\frac{\|c\|^2}{d\sigma^2} \geq \frac{1}{1 + C_{\ref{lemma:concentration}}\sqrt{\frac{n}{d}}}\frac{n\gamma^2}{d\sigma^2} \geq \frac{\kappa}{1 + C_{\ref{lemma:concentration}}\sqrt{\frac{n}{d}}}\left(\min_j y_jv^T\xi_j\right)^2.
\end{align}
The conclusion follows.
\end{proof}

Let $\kgen := 0$, and let $\kphen := 1$.

The following is our main technical lemma. It shows that if the margin is near-optimal and $\kappa > \kgen$, then $w$ must have a large component in the $\mu$ direction. If additionally $\kappa < \kphen$, then the spurious component must explain more than half of the margin on every data point, or even more if $\kappa$ is very small.
\begin{lemma}[Main Technical Lemma]\label{linear:tech_lemma}
For any $\kappa > \kgen$, there exist a universal constant $c$ such that if $\epsilon \leq \frac{1}{c}\min\left(\kappa, \frac{1}{\kappa}\right)$ and $\sqrt{\frac{n}{d}} \leq \frac{1}{c}\min\left(\kappa, \frac{1}{\kappa}\right)$, with probability $1 - 3e^{-n}$ over $(X, y)$, for any $(1 - \epsilon)$-margin maximizing solution $w$ with $\|w\| = 1$ we have
\begin{align}
    \frac{w^T\mu}{\|v\|_2\sqrt{\kappa}} \geq \frac{2\sqrt{2}}{3}.
\end{align}
If additionally $\kappa < \kphen$ and $\epsilon \leq \frac{1}{c}\min\left(\kappa, \frac{1}{\kappa}, (\kphen - \kappa)^2\right)$ and $\sqrt{\frac{n}{d}} \leq \frac{1}{c}\min\left(\kappa, \frac{1}{\kappa}, (\kphen - \kappa)^2\right)$, then for every $j$,  
\begin{align}
    y_jv^T\xi_j \geq  \max\left(1 + \frac{1}{c}, \frac{1}{2\kappa}\right)w^T\mu =  \max\left(1 + \frac{1}{c}, \frac{1}{2\kappa}\right)y_jw^Tz_j.
\end{align}
\end{lemma}
\begin{proof}
We condition on the events that the outcome of Lemma~\ref{lemma:concentration} (and hence  Lemma~\ref{lem:min}) hold of $\Xi$, which occurs with probability at least $1 - 3e^{-n}$.
By Lemma~\ref{lem:min},
\begin{align}\label{eq:margin_lin}
\gamma(w, S) &= w^T\mu + \min_j y_jv^T\xi_j\\
    &\leq w^T\mu + \frac{\|v\|_2\sqrt{\left(1 + C_{\ref{lemma:concentration}}\sqrt{\frac{n}{d}}\right)}}{\sqrt{\kappa}}.
\end{align}
Further, we can lower bound the max-margin by constructing a solution $\hat{w}$ in the following way:

Let $\hat{w} = \frac{\hat{u} + \hat{v}}{\|\hat{u} + \hat{v}\|}$, where $\hat{u} = n\mu$, and $\hat{v}$ is the min-norm solution to $\hat{v}^T\Xi = (d\sigma^2)y$. Since the conclusion of  Lemma~\ref{lemma:concentration} holds, we have $\|\hat{v}\|^2 \leq \frac{n\sigma^2 d}{1 - C_{\ref{lemma:concentration}\sqrt{\frac{n}{d}}}}$. 

Thus we have 
\begin{align}\label{eq:lin_max_margin}
     \gamma^*(S) &\geq\left(1 - C_{\ref{lemma:concentration}}\sqrt{\frac{n}{d}}\right)\frac{n + d\sigma^2}{\sqrt{n^2 + n\sigma^2d}}\\
     &= \left(1 - C_{\ref{lemma:concentration}}\sqrt{\frac{n}{d}}\right)\sqrt{1 + \frac{1}{\kappa}}.
\end{align}

We prove the first conclusion of the lemma first. Putting together Equations~\ref{eq:margin_lin} and \ref{eq:lin_max_margin}, we have
\begin{align}
    w^T\mu + \frac{\|v\|_2\sqrt{\left(1 + C_{\ref{lemma:concentration}}\sqrt{\frac{n}{d}}\right)}}{\sqrt{\kappa}} \geq (1 - \epsilon)\gamma^*(S) \geq (1 - \epsilon)\sqrt{1 + \frac{1}{\kappa}}\left(1 - C\sqrt{\frac{n}{d}}\right).
\end{align}
Thus for some (different) universal constant $C$, we have 
\begin{align}
    w^T\mu + \frac{\|v\|_2}{\sqrt{\kappa}} &\geq \left(1 - \epsilon - C\sqrt{\frac{n}{d}}\right)\sqrt{1 + \frac{1}{\kappa}}\\
    &= \left(1 - \epsilon - C\sqrt{\frac{n}{d}}\right)\sqrt{1 + \frac{1}{\kappa}}\left(\sqrt{(w^T\mu)^2 + \|v\|^2}\right)
\end{align}

For the remainder of the proof, let $\epsilon' := \epsilon + C\sqrt{\frac{n}{d}}$. Letting $q = \frac{w^T\mu}{\|v\|_2}$, we have
\begin{align}
    q + \frac{1}{\sqrt{\kappa}} \geq \left(1 - \eps'\right)\sqrt{1 + \frac{1}{\kappa}}\left(\sqrt{1 + q^2}\right).
\end{align}
Squaring and rearranging terms, we have
\begin{align}
    q^2\left(1 - \left(1 - \eps'\right)^2\left(1 + \frac{1}{\kappa}\right)\right) + q\left(\frac{2}{\sqrt{\kappa}}\right) + \left(\frac{1}{\kappa} - \left(1 - \eps'\right)^2\left(1 + \frac{1}{\kappa}\right)\right)\geq 0,
\end{align}
or equivalently, \begin{align}\label{eq:quad}
    a\left(\frac{q}{\sqrt{\kappa}}\right)^2 + 2\left(\frac{q}{\sqrt{\kappa}}\right) + c \geq 0,
\end{align}
where 
\begin{align}
    a &:= \left( \kappa\left(1 - \left(1 - \eps'\right)^2\right) - \left(1 - \eps'\right)^2\right)\\
    c &:= \frac{1}{\kappa} - \left(1 - \eps'\right)^2\left(1 + \frac{1}{\kappa}\right),
\end{align}
\begin{claim}
For a small enough constant $\delta$, for any $\kappa, n, d, \epsilon > 0$ such that $\sqrt{\frac{n}{d}} \leq \frac{\delta\min(\kappa, 1/\kappa)}{2C}$ and $\epsilon \leq \frac{\delta\min(\kappa, 1/\kappa)}{2}$, we have $a < 0$, and $c \leq -\frac{2\sqrt{2}}{3}$.
\end{claim}
\begin{proof}
Not that the conditions of the claim imply that $\eps' \leq \delta \min(\kappa, 1/\kappa)$. Thus for a small enough $\delta$, we have 
\begin{align}
    \left(1 - \eps'\right)^2 \geq \left(1 - \delta\min\left(\kappa, \frac{1}{\kappa}\right)\right)^2 \geq 1 - 3\delta\min\left(\kappa, \frac{1}{\kappa}\right).
\end{align}
Thus for a small enough $\delta$,
\begin{align}
     c &= \frac{1}{\kappa} - \left(1 - \eps'\right)^2\left(1 + \frac{1}{\kappa}\right)\\
     &\leq \frac{1}{\kappa} - \left(1 - 3\delta\min\left(\kappa, \frac{1}{\kappa}\right)\right)\left(1 + \frac{1}{\kappa}\right)\\
     &\leq -1 + 3\delta \min\left(\kappa, \frac{1}{\kappa}\right)\left(1 + \frac{1}{\kappa}\right)\\
     &\leq -1 + 6\delta < -\frac{2\sqrt{2}}{3},
\end{align}
and similarly,
\begin{align}
    a &= \kappa\left(1 - \left(1 - \eps'\right)^2\right) - \left(1 - \eps'\right)^2\\
    &\leq \kappa\left(3\delta\min(\kappa, 1/\kappa)\right) - \left(1 - 3\delta\min(\kappa, 1/\kappa)\right)\\
    &\leq -1 + 6\delta.
\end{align}

\end{proof}
Viewing Equation~\ref{eq:quad} as a quadratic function $f$ of $\frac{q}{\sqrt{\kappa}}$, if the claim above holds, then
\begin{align}
    \frac{q}{\sqrt{\kappa}} \geq -\frac{f(0)}{f'(0)} = \frac{-c}{2} \geq \frac{\sqrt{2}}{3}.
\end{align}
Since $q$ was defined to be $\frac{w^T\mu}{\|v\|}$, this proves the first part of the lemma for $\epsilon, \sqrt{\frac{n}{d}} \leq \frac{\min(\kappa, \frac{1}{\kappa})}{c}$ for a constant $c$ large enough.

We perform a similar argument for the second conclusion. Observe that by plugging in the contents of Equation~\ref{eq:margin_lin} into Equation~\ref{eq:lin_max_margin}, we have for some constant $C$ (whose value changes throughout this equation, but does not depend on $\kappa$),
\begin{align}
    w^T\mu + \min_j y_jv^T\xi_j &\geq \left(1 - \epsilon - C\sqrt{\frac{n}{d}}\right)\sqrt{1 + \frac{1}{\kappa}}\sqrt{(w^T\mu)^2 + \|v\|^2}\\
    &\geq \left(1 - \epsilon - C\sqrt{\frac{n}{d}}\right)\sqrt{1 + \frac{1}{\kappa}}\sqrt{(w^T\mu)^2 + \frac{\kappa\left(\min_j y_jv^T\xi_j\right)^2}{1 + C_{\ref{lemma:concentration}}\sqrt{\frac{n}{d}}}}\\
    &\geq \left(1 - \epsilon - C\sqrt{\frac{n}{d}}\right)\sqrt{1 + \frac{1}{\kappa}}\sqrt{(w^T\mu)^2 + \kappa\left(\min_j y_jv^T\xi_j\right)^2}
\end{align}
Let $r := \frac{\min_j y_jv^T\xi_j}{w^T\mu}$. Dividing through by $w^T\mu$ and squaring, we obtain:
\begin{align}
    (r + 1)^2 \geq \left(1 - \epsilon - C\sqrt{\frac{n}{d}}\right)^2\left(1 + \frac{1}{\kappa}\right)(1 + \kappa r^2),
\end{align}
Rearranging, and multiplying by $\kappa$, we have 
\begin{align}
    (r\kappa)^2\left(\frac{1}{\kappa} - \left(1 + \frac{1}{\kappa}\right)\left(1 - \epsilon - C\sqrt{\frac{n}{d}}\right)^2\right) + 2(r\kappa) + \left(\kappa - \left(\kappa + 1\right)\left(1 - \epsilon - C\sqrt{\frac{n}{d}}\right)^2\right) \geq 0
\end{align}

Now, for $\kappa < \kphen = 1$, for a small enough constant $\delta$ (independent of $\kappa$), if $\sqrt{\frac{n}{d}} \leq \frac{\delta\kappa(\kphen - \kappa)^2}{2C}$ and $\epsilon \leq \frac{\delta\kappa(\kphen - \kappa)^2}{2}$, we $q = r\kappa$, we have
\begin{align}
    q^2\left(\frac{1}{\kappa} - \left(1 + \frac{1}{\kappa}\right)\left(1 - 3\delta(\kphen - \kappa)^2\kappa)\right)\right) + 2q + \left(\kappa - \left(\kappa + 1\right)\left(1 - 3\delta(\kphen - \kappa)^2\kappa\right)\right) \geq 0,
\end{align}
so 
\begin{align}
    -q^2\left(1 - 6\delta(\kphen - \kappa)^2\right) + 2q  - (1 - 6\delta(\kphen - \kappa)^2) \geq 0,
\end{align}
or 
\begin{align}
    -q^2 + \frac{2q}{\left(1 - 6\delta(\kphen - \kappa)^2\right)}  - 1 \geq 0.
\end{align}
Let $x = 6\delta(\kphen - \kappa)^2$. The smallest root of this equation is given by 
\begin{align}
    \frac{-(2 + x) + \sqrt{(2 + x)^2 - 4}}{-2} = \frac{2 + x - \sqrt{4x - x^2}}{2} \geq 1 - \sqrt{x},
\end{align}
so $q \geq 1 - \sqrt{6\delta}(\kphen - \kappa)$.

Thus for a constant $\delta$ small enough, we have
\begin{align}
    \frac{\min_j y_jv^T\xi_j}{w^T\mu} = \frac{q}{\kappa} > \frac{1 + \sqrt{6\delta}(\kappa - 1)}{\kappa} > \max\left(1, \frac{1}{2\kappa}\right).
\end{align}

\end{proof}

The following lemma shows that the influence of the $v$ on the label of a test example is small.
\linsmall*

\begin{proof}
For any $x$, we have $f_u(x) = u^T\mu$. Now $f_v(x)$ is distributed like $\|v\|\sigma\sqrt{d - 1}\frac{X}{\sqrt{X^2 + Y}}$, where $X \sim \mathcal{N}(0, 1)$, and $Y$ is a Chi-square random variable with $d - 2$ degrees of freedom. Thus we can bound the probability that $|f_v(x)| \geq yf_u(x)$ by the probability that $\|v\|\sigma X \geq \frac{1}{2}yf_u(x)$ plus the probability that $X^2 + Y$ is smaller than $\frac{d - 1}{4}$.

Thus we have 
\begin{align}
    \Pr[|f_v(x)| \geq yf_u(x)] & \leq  \Pr[|\mathcal{N}(0, \sigma^2\|v\|^2_2)| \geq \frac{1}{2}yu^T\mu] + \Pr\left[X^2 + Y \leq \frac{d}{4}\right]\\
    &= \Pr[|\mathcal{N}(0, \sigma^2)| \geq |q|/2] + \exp(-d/8)\\
    &\leq 2e^{-\frac{q^2}{8\sigma^2}} + \exp(-d/8).
\end{align}

\end{proof}

\subsection{Proof of Main Results}
We use the results of the previous subsection to prove our main results in the linear setting.

First we prove our linear generalization result, Theorem~\ref{thm:linear_gen}, which we restate for the reader's convenience.

\lineargen*

We prove the following slightly stronger result, which implies Theorem~\ref{thm:linear_gen}, and gives the exact dependence of $c$ on $\delta$.

\begin{theorem}\label{thm:linear_full}
There exists a universal constant $c$ such that the following holds. For any $n, d, \sigma$ and $\mathcal{D} \in \Omega_{\sigma, d}^{\on{linear}}$ such that $\kappa = \frac{n}{d\sigma^2} > \kgen^{\on{linear}}$ and $\frac{d}{n} \geq c\max\left(\frac{1}{\kappa^2}, \kappa^2\right)$, with probability $1 - 3e^{-n}$ over the randomness of a training set $S \sim \mathcal{D}^n$, for any $w \in \mathbb{R}^d$ that is a $(1 - \epsilon)$-max-margin solution (as in Definition~\ref{def:maxmargin}), we have 
$ \mathcal{L}_{\mathcal{D}}(f_w) \leq e^{-\frac{n}{36d\sigma^4}} + e^{-n/8},$
where $\epsilon = \frac{1}{c}\min\left(\frac{1}{\kappa}, \kappa\right)$.
\end{theorem}

\begin{proof}
The proof follows from combining Lemmas~\ref{linear:tech_lemma} and \ref{lemma:linear_spur_small_2}. Let $c$ be the universal constant from Lemmas~\ref{linear:tech_lemma}. For any $\kappa > \kgen$, by Lemmas~\ref{linear:tech_lemma}, for constants $\epsilon = \frac{1}{c}\min\left(\kappa, \frac{1}{\kappa}\right)$, if $\frac{d}{n} > c^2\max\left(\kappa^2, \frac{1}{\kappa^2}\right)$, with probability $1 - 3e^{-n}$, for any $w$ which is a $(1 - \eps)$ max-margin solution, we have 
\begin{align}
    \frac{w^T\mu}{\|v\|_2 \kappa} \geq \frac{\sqrt{2}}{3}.
\end{align}
Now appealing to Lemma~\ref{lemma:linear_spur_small_2}, this means that  $\mathcal{L}_{\mathcal{D}}(f_w) \leq 2e^{-\frac{\kappa}{36\sigma^2}} + \exp(-d/8) \leq 2e^{-\frac{\kappa}{36\sigma^2}} + \exp(-n/8).$
\end{proof}

To prove our impossibility results, for any $\md = \md_{\mu, d, \sigma} \in \Omega$, we define the following mappings $\psi$ and $\bar{\psi}$. For $(x, y) \in \mathbb{R}^d \times \{-1, 1\}$, where $x = \mu y + \xi$, define $\psi((x, y)) = (-\mu y + \xi, y)$. This mapping swaps the signal direction of the example, but maintains the label and junk component, and we will use it for our margin lower bound. Define $\bar{\psi}((x, y)) = (\mu y - \xi, y)$, which swaps the signal direction of the example, but maintains the label and junk component. We will use this for our UC lower bound. 

For a set of training examples $S$, let $\psi(S)$ (resp. $\bar{\psi}(S)$) be the set where each example is mapped via $\psi$ (resp. $\bar{\psi}$). Finally define $\psi(\mathcal{D}) := \md_{-\mu, d, \sigma}$ . Thus $\psi(\mathcal{D})$ is the distribution with the opposite signal direction, and $\bar{\psi}(\mathcal{D}) = \mc{D}$. It is immediate to check that for any classifier $w \in \mathbb{R}^d$, $\mathcal{L}_{\md}(f_w) = 1 - \mathcal{L}_{\psi(\md)}(f_w)$.  Note that $\psi$ and $\bar{\psi}$ implicitly depend on $\md$ through the parameter $\mu$. If it is not clear from context that we are speaking about a specific $\md$, we will use $\psi_{\md}$ or $\bar{\psi}_{\md}$ to denote the mapping associated with $\md$.

We now prove the the linear part of Proposition~\ref{thm:one_sided_impossible_linear}. We restate a version which just includes the linear part, and gives more precise dependence of $c$ and $\epsilon$ on the distance between $\kappa$ and the boundaries $\kphen$ and $\kgen$.

\begin{proposition}[UC Bounds are Vacuous for Linear Problem (From Proposition~\ref{thm:one_sided_impossible_linear})]
There exists a universal constant $c$ for which the following holds. For any $n, d, \sigma$ and $\mc{D} \in \Omega_{\sigma, d}^{\on{linear}}$ such that $\kgenl \leq \kappa \leq \kphenl$, $\frac{d}{n} \geq \frac{c}{\kappa^2(\kphen - \kappa)^4}$, the following holds for any $\eps \leq \frac{\kappa(\kphen - \kappa)^2}{c}$. Let $\mathcal{A}$ be any algorithm that outputs $w \in \mathbb{R}^d$ which is a $(1-\epsilon)$-max-margin solution for any $S \in (\mathbb{R}^d \times \{1,-1\})^n$. Let $\mathcal{H}$ be any hypothesis class that is useful for $\mathcal{A}$ on $\md$ (as in Definition~\ref{def:useful}). Suppose that $\epsilon_{\on{unif}}$ is a uniform convergence bound for $\mc{D}$ and $\mc{H}$, that is, $$\Pr_{S \sim \mathcal{D}^n}[ \sup_{h \in \mathcal{H}} |\mathcal{L}_{\mathcal{D}}(h) - \mathcal{L}_S(h)| \geq \epsilon_{\on{unif}} ] \leq 1/4.$$
Then $\epsilon_{\on{unif}} \geq 1 - e^{-\frac{n}{36d\sigma^2}} - e^{-n/8}.$
\end{proposition}


\begin{proof}
Let $T_{\md} \subset 2^{(\mathbb{R}^d \times \{-1, 1\})^n}$ be the set of training sets $S$ on which the conclusion of Lemma~\ref{linear:tech_lemma} holds for $S$. Thus $\Pr_{S \sim {\md}^n}[S \in T_{\md}] \geq 1 - 3e^{-n}$. Let $H \subset 2^{(\mathbb{R}^d \times \{-1, 1\})^n}$ be the set of training sets $S$ on which $\mathcal{A}(S) \in \mathcal{H}$. Thus  $\Pr_{S \sim {\md}^n}[S \in H] \geq \frac{3}{4}$.

Let $T'_{{\md}}$ be the set on which 
\begin{align}
    |\mathcal{L}_{\md}(h) - \mathcal{L}_{\phi(S)}(h)| \leq \eps_{\on{unif}} \qquad \forall h \in \mathcal{H},
\end{align}
where $\phi := \bar{\psi}_{\md}$. By assumption, $\Pr_{S \sim {\md}^n}[\phi(S) \in T'_{\md}] = \Pr_{S \sim {\md}^n}[S \in T'_{\md}] \geq \frac{3}{4}$. By a union bound, for $n \geq 2$,
\begin{align}
    \Pr_{S \sim {\md}^n}[S \in T'_{\md} \land S \in T_{\md} \land S \in H] \geq 1 - \left(1 - \frac{3}{4}\right) - \left(1 - \frac{3}{4} + 3e^{-n}\right) = \frac{1}{2} - 3e^{-n} > 0.
\end{align}
This is because the distribution of $\bar{\psi}(S)$ with $S \sim \md^n$ is the same as the distribution of $n$ samples from  $\bar{\psi}({\md}) = \md$. 

Let $S$ be any set for which the three event above hold, ie.,
\begin{align}
  S \in T'_{\md} \land S \in T_{\phi({\md})} \land S \in H. 
\end{align}
With $f_w = \mathcal{A}(S)$, we have by combining the results of Lemmas~\ref{linear:tech_lemma} and \ref{lemma:linear_spur_small_2} that $\mathcal{L}_{\md}(f_w) \leq e^{-\frac{n}{36d\sigma^2}} - e^{-n/8}$. Further, by the second conclusion of Lemma~\ref{linear:tech_lemma}, we know that $\mathcal{L}_{\phi(S)}(f_w) = 1$, since $f_w$  misclassifies every point in $S$. It follows that $\eps_{\on{unif}} \geq 1 - e^{-\frac{n}{36d\sigma^2}} - e^{-n/8}$.
\end{proof}

Finally we prove Proposition~\ref{thm:margin_impossible_linear} via a similar technique, but using the mapping $\psi$ instead of $\bar{\psi}$.

\linmargin*

We prove the following slightly stronger result, which implies Proposition~\ref{thm:margin_impossible_linear}, and gives the conditions depending on the distance between $\kappa$ and the boundaries $\kphen$ and $\kgen$.

\begin{proposition}
There exists a universal constant $c$ such that the following holds. For any $n, d, \sigma$ and $\mathcal{D} \in \Omega_{\sigma, d}^{\on{linear}}$ such that $\kappa < \kphenl$ and $\frac{d}{n} \geq \frac{c}{\kappa^2(\kphen - \kappa)^4}$, the following holds. Let $\epsilon =  \frac{\kappa(\kphen - \kappa)^2}{c}$, and let $\mathcal{A}$ be any algorithm so that $\mathcal{A}(S)$ outputs a $(1-\epsilon)$-max-margin solution $f_w$ for any $S \in (\mathbb{R}^d \times \{1,-1\})^n$.
Let $\mathcal{H}$ be any hypothesis class that is useful for $\mathcal{A}$ (as in Definition~\ref{def:useful}) on both $\mc{D}_{\mu, \sigma}^{\on{linear}}$ and $\mc{D}_{-\mu, \sigma}^{\on{linear}}$.
Suppose that there exists an polynomial margin bound of integer degree $p$: that is, there is some $G$ that satisfies for $\tilde{D} \in \{\mc{D}, \psi(\mc{D})\}$, 
$$\Pr_{S \sim \tilde{\mathcal{D}}^n}\left[\sup_{h \in \mathcal{H}} \mathcal{L}_{\tilde{\mathcal{D}}}(h) - \mathcal{L}_S(h)\geq \frac{G}{\gamma(h, S)^p}\right]  \leq \frac{1}{4}.$$
Then with probability $\frac{1}{2} - 3e^{-n}$ over $S \sim \mathcal{D}^n$, the margin bound is weak even on the max-margin solution, that is, $\frac{G}{\gamma^*(S)^p} \geq \max\left(\frac{1}{c}, 1 - e^{-\frac{\kappa}{36\sigma^2}} - e^{-n/8} - \frac{3\kappa}{c}\right)^p$, which is more than an absolute constant.
\end{proposition}

\begin{proof}
For any $\md \in \Omega$, let $T_{\md} \subset 2^{(\mathbb{R}^d \times \{-1, 1\})^n}$ be the set of training sets $S$ on which the conclusion of Lemma~\ref{linear:tech_lemma} holds for ${\md}$ and $S$. Thus for any ${\md} \in \Omega$, $\Pr_{S \sim {\md}^n}[S \in T_{\md}] \geq 1 - 3e^{-n}$. Let $H \subset 2^{(\mathbb{R}^d \times \{-1, 1\})^n}$ be the set of training sets $S$ on which $\mathcal{A}(S) \in \mathcal{H}$. Thus for any ${\md} \in \Omega$, $\Pr_{S \sim {\md}^n}[S \in H] \geq \frac{3}{4}$.

For any $\md \in \Omega$, let $T'_{{\md}}$ be the set on which 
\begin{align}
    \mathcal{L}_{\md}(h) \leq \mathcal{L}_S(h) + \frac{G}{\gamma(h, S)^p} \qquad \forall h \in \mathcal{H}.
\end{align}
By assumption, for any ${\md} \in \Omega$, $\Pr_{S \sim {\md}^n}[S \in T'_{\md}] \geq \frac{3}{4}$. 

Now fix any $\md = \md_{\mu, \sigma, d}\in \Omega$. By a union bound, with $\psi = \psi_{\md}$,
\begin{align}
    \Pr_{S \sim {\md}^n}[S \in T'_{\md} \land \psi(S) \in T_{\psi({\md})} \land \psi(S) \in H] \geq 1 - \left(1 - \frac{3}{4}\right) - \left(1 - \frac{3}{4} + 3e^{-n}\right) = \frac{1}{2} - 3e^{-n}.
\end{align}
This is because the distribution of $\psi(S)$ with $S \sim \md^n$ is the same as the distribution of $n$ samples from  $\psi({\md})$. 

Let $S$ be any set for which the three events above hold, ie.,
\begin{align}
  S \in T'_{\md} \land \psi(S) \in T_{\psi({\md})} \land \psi(S) \in H. 
\end{align} With $f_w = \mathcal{A}(\psi(S))$, we have by combining the results of Lemmas~\ref{linear:tech_lemma} and \ref{lemma:linear_spur_small_2} that $\mathcal{L}_{\psi({\md})}(f_w) \leq e^{-\frac{n}{9d\sigma^2}}$, and thus $\mathcal{L}_{{\md}}(f_w) \geq 1 - e^{-\frac{n}{9d\sigma^2}}$. Further, by the conclusion of Lemma~\ref{linear:tech_lemma}, we know that for $(x_j, y_j) \in S$, with $\mu$ being the direction of the distribution $\psi({\md})$, for some $C > 0$,

\begin{align}
    y_jv^T\xi_j \geq \left(1 + \frac{1}{C}, \frac{1}{2\kappa}\right) w^T\mu.
\end{align}

Observe that $\gamma^*(\psi(S)) \leq \frac{1}{1 - \epsilon}(w^T\mu + y_jv^T\xi_j)$, and thus letting $b = w^T\mu$ and $a = \min_j y_jv^T\xi_j$, we have

\begin{align}
    \frac{\gamma(f_w, S)}{\gamma^*(\psi(S))} \geq (1 - \epsilon)\frac{a - b}{a + b} > (1 - \epsilon)\frac{\max\left(1 + \frac{1}{C}, \frac{1}{2\kappa}\right) - 1}{\max\left(1 + \frac{1}{C}, \frac{1}{2\kappa}\right) + 1} \geq (1 - \epsilon)\max\left(\frac{1}{3C}, 1 - 4\kappa\right) \geq \max\left(\frac{1}{4C}, 1 - 2\epsilon - 8\kappa\right),
\end{align}
since $C$ is a constant and $\epsilon$ is sufficiently small.

It follows that for any such $S$, we must have 
\begin{align}
    G \geq (1 - e^{-\frac{n}{9d\sigma^4}})\gamma(f_w, S)^p \geq (1 - e^{-\frac{n}{9d\sigma^4}})\max\left(\frac{1}{4C}, 1 - 2\epsilon - 8\kappa\right)^p\gamma^*(\psi(S))^p.
\end{align}

Thus for the distribution $\psi({\md})$, with probability at least $\frac{1}{2} - 3e^{-n}$,  the margin bound yields a generalization guarantee no better than 
\begin{align}
    (1 - e^{-\frac{n}{9d\sigma^4}})\max\left(\frac{1}{4C}, 1 - 2\epsilon - 8\kappa\right)^p 
    &\geq \max\left(\frac{1}{c}, 1 - e^{-\frac{\kappa}{9\sigma^2}} - 9\kappa\right)^p,
\end{align}
where we have assumed $c$ is a sufficiently large constant, and plugged in the assumption that $\epsilon \leq \frac{\kappa}{c}$.

\end{proof}

\subsection{Additional Linear Results: Beyond Extremal Margins}

While the extremal margin bounds we just presented are promising, in some cases natural solutions found via SGD do not achieve a near-optimal margin. For instance, this may be due to early stopping or non-convexity of the loss landscape. For example, in the linear setting, if we run SGD on the logistic loss and stop after some fixed number of iterations, then in general we get solutions of the form $\sum \lambda_i x_i$, where $\lambda_i$ is roughly in the direction of $y_i$, but there may be a lot of variance among the $\lambda_i$'s. Unless the $\lambda_i$ are well-concentrated, we will not achieve an near-extremal margin.

Fortunately, in the linear setting, if we additionally assume the data-dependent condition that the solution lies in the \em span \em of the training data, which is satisfied by any solution found via first order methods, then it suffices to show a much weaker margin condition in order to achieve generalization. This weaker condition only requires that the \em average \em margin is large enough. We achieve guarantees bonding the test loss that are analogous to those in Theorem~\ref{thm:linear_gen} whenever the average margin is on the same order as the maximum margin. 

\begin{restatable}[Generalization with Linear Span Condition.]{theorem}{linspan}\label{thm:span}
For any $\mathcal{D} \in \Omega_{\sigma, d}^{\on{linear}}$, if $\kappa > \kgen^{\on{linear}}$ and $\frac{d}{n} \geq 4C_{\ref{lemma:concentration}}^2$, then with probability $1 - 3e^{-n}$ over $S \sim \mathcal{D}^n$, for any $w \in \on{Span}(\{x_i\})$, if $\bar{\gamma} := \frac{\mathbb{E}_i y_i f_w(x_i)}{\|w\|_2} \geq 4C_{\ref{lemma:concentration}}\sigma$, we have
\begin{align}
    \mathcal{L}_{\mathcal{D}}(f_w) \leq 2e^{-\frac{\min\left(1, \frac{1}{\kappa^2}\right)\bar{\gamma}^2}{8\sigma^2}}.
\end{align}
If particular, it $\kappa \leq \kphen$, this test loss bound equals $2e^{-\frac{\bar{\gamma}^2}{8\sigma^2}}$.
\end{restatable}

To prove Theorem~\ref{thm:span}, we begin with the following lemma, which shows that for classifier $w$ in the span with a good average empirical margin, it must be that $\mu^Tw$ is large.

Recall that our training data is given by $(X, y)$ where $X \in \mathbb{R}^{d \times n}$ and $y \in \mathbb{R}^n$, and $\Xi \in \mathbb{R}^{d \times n}$ equals $X - \mu y^T$.

\begin{lemma}\label{lemma:dot}
Let $\delta = C_{\ref{lemma:concentration}}\sqrt{\frac{n}{d}}$, and assume that $\delta \leq \frac{1}{2}$. On the event that the conclusion of Lemma~\ref{lemma:concentration} holds for $\Xi$, for any $a \in \mathbb{R}^n$, with $w = Xa$ we have 
\begin{align}w^T\mu \geq \frac{y^T\hat{y} - 2\delta \sigma\sqrt{d}\|y\|_2\|w\|_2}{\sigma^2d + \|y\|^2},
\end{align}
where $y \in \mathbb{R}^n$ is the vector with entries $y_j$, $\hat{y}$ is the vector with entries $w^Tx_j$.
\end{lemma}

\begin{proof}[Proof of Lemma~\ref{lemma:dot}]

We expand
\begin{align}
    \mu^Tw &= \mu^TXa \\
    &= \frac{1}{\sigma^2d}\mu^TXX^TXa + \mu^TX\left(I - \frac{X^TX}{\sigma^2 d}\right)a\\
    &= \frac{1}{\sigma^2d}y^T\hat{y} + y^T\left(I - \frac{X^TX}{\sigma^2 d}\right)a\\
    &= \frac{1}{\sigma^2d}y^T\hat{y} + y^T\left(I - \frac{\Xi^T\Xi}{\sigma^2 d}\right)a - \frac{1}{\sigma^2d}y^T\left(\|\mu\|^2yy^T\right)a\\
    &= \frac{1}{\sigma^2d}y^T\hat{y} + y^T\left(I - \frac{\Xi^T\Xi}{\sigma^2 d}\right)a - \frac{1}{\sigma^2d}\|y\|^2\left(y^Ta\right)
\end{align}

Here in the second equality, we used the fact that $\mu^TX = y$, and $X^TXa = X^Tw = \hat{y}$, and in the third, we used that $\Xi = X - \mu y^T$.

Since $\mu^Tw = y^Ta$, we can gather these terms and rearrange, yielding
\begin{align}\label{eq:muw}
    \mu^Tw\left(1 + \frac{\|y\|^2}{\sigma^2d}\right) = \frac{1}{\sigma^2d}y^T\hat{y} + y^T\left(I - \frac{\Xi^T\Xi}{\sigma^2 d}\right)a.
\end{align}

Since we have conditioned on the conclusion of Lemma~\ref{lemma:concentration} holding, we have that $\|I - \frac{\Xi^T\Xi}{\sigma^2 d}\|_2 \leq C_{\ref{lemma:concentration}}\sqrt{
\frac{n}{d}}$. 
To bound $\|a\|$, we write
\begin{align}
    \|a\|^2 = \frac{1}{\sigma^2 d}a^T\Xi^T\Xi a + a^T\left(I - \frac{1}{\sigma^2d}\Xi^T\Xi\right)a,
\end{align}
and thus since $\|w\| = \|Xa\| \geq \|\Xi a\|$ and by the conclusion of Lemma~\ref{lemma:concentration}, we can rearrange to get 
\begin{align}
    \|a\|^2 \leq \frac{1}{1 - \|I - \frac{1}{d\sigma^2}\Xi^T\Xi\|_2}\frac{\|a\Xi\|^2}{\sigma^2d} \leq \frac{1}{1 -  C_{\ref{lemma:concentration}}\sqrt{
\frac{n}{d}}}\frac{\|w\|^2}{\sigma^2d}
\end{align}

Thus plugging in $\delta = C_{\ref{lemma:concentration}}\sqrt{
\frac{n}{d}}$, since we have assumed $\delta \leq \frac{1}{2}$, we have
\begin{align}
     y^T\left(I - \frac{\Xi^T\Xi}{\sigma^2 d}\right)a \leq \|y\|\|w\|\frac{1}{(\sigma^2 d)^{\frac{1}{2}}}\frac{\delta}{\sqrt{1 - \delta}} \leq 2\|y\|\|w\|\frac{1}{(\sigma^2 d)^{\frac{1}{2}}}
     \delta.
\end{align}

Plugging this into equation~\ref{eq:muw} and dividing both sides by $1 + \frac{\|y\|^2}{\sigma^2d}$ we have
\begin{align}
    \mu^Tw \geq \frac{\frac{1}{\sigma^2d}y^T\hat{y} - 2\|y\|\|w\|\frac{1}{(\sigma^2 d)^{\frac{1}{2}}}}{1 + \frac{\|y\|^2}{\sigma^2d}} = \frac{y^T\hat{y} - 2\delta \sigma \sqrt{d}\|y\|\|w\|}{\sigma^2d + \|y\|^2},
\end{align}
which is the desired conclusion.

\end{proof}

We now prove Theorem~\ref{thm:span}, which we restate here.

\linspan*

\begin{proof}[Proof of Theorem~\ref{thm:span}]
Condition on the event that the conclusion of Lemma~\ref{lemma:concentration} holds, which occurs with probability $1 - 3e^{-n}$. Since $\frac{d}{n} \geq 4C_{\ref{lemma:concentration}}^2$, this implies that the conclusion of Lemma~\ref{lemma:dot} holds. Then for any $w \in \on{Span}(\{x_i\})$, we have 

\begin{align}
w^T\mu &\geq \frac{y^T\hat{y} - 2C_{\ref{lemma:concentration}}\sqrt{\frac{n}{d}} \sigma\sqrt{d}\|y\|_2\|w\|_2}{\sigma^2d + \|y\|^2} \\
&\geq \frac{\|w\|}{\sigma^2d + n}\left(n\bar{\gamma} - 2C_{\ref{lemma:concentration}}\sigma n\right)\\
&\geq \frac{1}{2}\|w\|\min(1, \kappa)\bar{\gamma}.
\end{align} 

Thus with $v = w - \mu w^T\mu$, we have
\begin{align}
    \frac{w^T\mu}{\|v\|\kappa} \geq \frac{w^T\mu}{\|w\|\kappa} \geq \frac{1}{2}\min\left(1, \frac{1}{\kappa}\right)\bar{\gamma}
\end{align}

Now appealing to Lemma~\ref{lemma:linear_spur_small_2}, this means we have $\mc{L}_{\mc{D}} \leq 2e^{-\frac{\min\left(1, \frac{1}{\kappa^2}\right)\bar{\gamma}^2}{8\sigma^2}}$.
\end{proof}

\section{Proofs for XOR 2-layer Neural Network Problem}\label{sec:xor_proofs}

Throughout this section, since we are only concerned with the XOR problem, we will abbreviate $\Omega = \omegax$, $\kphen = \kphenx$, and $\kgen = \kgenx$.

In subsection~\ref{sec:reduction}, we present a series of technical lemmas. In subsection~\ref{sec:xor_thms}, we prove our main theorems for the XOR problem, assuming the technical lemmas. In subsection~\ref{sec:xor_technical_proofs}, we prove the technical lemmas.

\subsection{Technical Lemmas}\label{sec:reduction}

\paragraph{Notation.}
Throughout the following subsection, we assume $\md = \mathcal{D}_{\mu_1, \mu_2, \sigma, d} \in \Omega$ is fixed. For a weight matrix $W \in \mathbb{R}^{m \times d}$, we define $U = W \Pi_{\mathrm{span}(\mu_1, \mu_2)}$ and $V = W \Pi_{\mathrm{span}(\mu_1, \mu_2)^\perp}$, where $\Pi_T$ is the orthogonal projector onto $T$. For $i \in [m]$, let $w_i \in \mathbb{R}^d$, $u_i \in \mathbb{R}^d$ and $v_i \in \mathbb{R}^d$ denote the rows of $W$, $U$ and $V$ respectively. We use $\mathbb{E}$ to denote the expectation over $i$ uniformly in $[m]$. Let $H_+: \{i : a_i > 0\}$, and $H_-: \{i : a_i < 0\}$

Recall that our samples $x_1, \ldots, x_n$ are of the form $x_j = z_j + \xi_j$, where $z_j \in \{\mu_1, -\mu_1, \mu_2, -\mu_2\}$ and $\xi_j \perp \mathrm{span}(\mu_1, \mu_2)$. Let $\mathcal{P}_1$, $\mathcal{P}_{-1}$, $\mathcal{N}_1$ and $\mathcal{N}_{-1}$ denote the four clusters of points, that is, 
\begin{align}
    \mathcal{P}_1 &= \{j \in [n]: z_j = \mu_1 \}\\
    \mathcal{P}_{-1} &= \{j \in [n]: z_j = -\mu_1\} \\
    \mathcal{N}_1 &= \{j \in [n]: z_j = \mu_2 \}\\
    \mathcal{N}_{-1} &= \{j \in [n]: z_j = -\mu_2\}
\end{align}
Let $\mathcal{P}= \mathcal{P}_1 \cup \mathcal{P}_{-1}$, and let $\mathcal{N} = \mathcal{N}_1 \cup \mathcal{N}_{-1}$. Let $\Xi \in \mathbb{R}^{d \times n}$ be the matrix with $j$th column $\xi_j$. Let $n_{\on{min}} := \min\left(|\mc{N}_{1}|, |\mc{N}_{-1}|, |\mc{P}_1|, |\mc{P}_{-1}|\right)$, and $n_{\on{max}} := \max\left(|\mc{N}_{1}|, |\mc{N}_{-1}|, |\mc{P}_1|, |\mc{P}_{-1}|\right)$ such that we expect $n_{\on{min}}$ and $n_{\on{max}}$ to be close to $\frac{n}{4}$, as per the Lemma~\ref{lemma:cluster_size} below.

Assume throughout the following section that $h \in (1, 2)$ is fixed, and recall that we have defined the activation $\phi(z) = \max(0, z)^h$.

\begin{lemma}\label{lemma:cluster_size}
For any $\beta > 0$, with probability as least $1 - 8e^{-8n\beta^2}$ over $S \in \md^n$, for all clusters $C \in \mc{P}_{1}, \mc{P}_{-1}, \mc{N}_{1}, \mc{N}_{-1}$, we have 
\begin{align}
    \left||C|- \frac{n}{4}\right| \leq \beta n,
\end{align}
and thus for $\beta \leq \frac{1}{8}$,
\begin{align}
    \frac{n_{\on{max}}}{n_{\on{min}}} = \frac{\max\left(|\mc{P}_{1}|, |\mc{P}_{-1}|,  |\mc{N}_{1}|, |\mc{N}_{-1}|\right)}{\min\left(|\mc{P}_{1}|, |\mc{P}_{-1}|,  |\mc{N}_{1}|, |\mc{N}_{-1}|\right)} \leq 1 + 16\beta.
\end{align}
\end{lemma}
This lemma follows immediately from Hoeffding’s inequality on Bernoulli random variables: To prove it, one can apply Hoeffding’s inequality four times (once for each cluster), and take a union bound.

\subsubsection{Overview of Technical Lemmas.}
Our goal will be to analyze near-optimal solutions to the following optimization program which is defined using the $n$ training examples.

\begin{definition}[Opt 1]\ \newline
Parameter: $P_1$ \newline
Variables: $W \in \mathbb{R}^{m \times d}$
\begin{align}
    &\max \gamma \\
    & \mathbb{E}_i a_i\phi(w_i^Tx_j)y_j \geq \gamma \qquad  \forall j\\
    & \mathbb{E}_i{\|w_i\|^2} \leq P_1
\end{align}
\end{definition}

In this subsection, we define a chain of optimization programs beginning from Opt 1. Each subsequent optimization problem becomes simpler and involves fewer variables. The chaining lemmas in this section typically show two conclusions:
\begin{enumerate}
    \item If a solution is near-optimal for the $i$th optimization program in the chain, then that solution can be transformed into a [series of] solutions that are [mostly] near-optimal for the $(i + 1)$th optimization program in the chain.
    \item An optimal solution to the $i$th optimization program in the chain can be converted into a near-optimal solution to the $(i - 1)$th optimization program in the chain. 
\end{enumerate}
Ultimately, in Lemma~\ref{lemma:trivariate_analysis} we study the optimal solution to the final simplest optimization program, which only includes $3$ variables. From this, using the first conclusion of the lemmas, we are able to chain back through the optimization programs to analyze certain properties of any $W$ which is a near-max-margin solution. This analysis ultimately leads to Lemmas~\ref{lemma:xor_gen} and Lemma~\ref{lemma:phen}, which are our main tools in proving generalization and the impossibility of UC bounds.

The second conclusion of the lemmas allows us to construct a near-max-margin solution $\hat{W}$ which satisfies certain properties allowing us to show the the limitations of margin bounds, and the failure of generalization for near-max-margin solutions when $\kappa < \kgen$. This is captured in Lemmas~\ref{margin_phen} and \ref{lemma_no_gen}.
 
\paragraph{Chain of Optimization Programs.}
We define the chain of Optimization Programs. Unless otherwise specified, all variables and parameters lie in $\mathbb{R}$. Like Opt 1, these programs all assume that the set of training examples $S = \{(x_j, y_j)\}_{j \in [n]}$ is fixed.

\begin{definition}[Opt 2]\ \newline
Parameter: $P_2$ \newline
Variables: $\{c_{ij}\}_{i \in [m], j \in [n]}$, $\{s_i\}_{i \in H_+}, \{t_i\}_{i \in H_-}$

\begin{align}
    &\max \: \gamma \\
    & \frac{1}{2}\mathbb{E}_{i \in H_+} \phi(s_i + c_{ij})  \geq \gamma \qquad  \forall j \in \mathcal{P}_1\\
    & \frac{1}{2}\mathbb{E}_{i \in H_+} \phi(-s_i + c_{ij}) \geq \gamma \qquad  \forall j \in \mathcal{P}_{-1}\\
    & \frac{1}{2}\mathbb{E}_{i \in H_-} \phi(t_i + c_{ij}) \geq \gamma \qquad  \forall j \in \mathcal{N}_1\\
    &  \frac{1}{2}\mathbb{E}_{i \in H_-} \phi(-t_i + c_{ij}) \geq \gamma \qquad  \forall j \in \mathcal{N}_{-1}\\
    & \frac{1}{2}\mathbb{E}_{i \in H_+}\left[s_i^2 + \frac{1}{d\sigma^2}\sum_{j \in \mathcal{P}}{c_{ij}^2}\right] + \frac{1}{2}\mathbb{E}_{i \in H_-}\left[t_i^2 + \frac{1}{d\sigma^2}\sum_{j \in \mathcal{N}}{c_{ij}^2}\right] \leq P_2
\end{align}

\end{definition}

\begin{definition}[Opt 3]\label{def:opt_3}\ \\
Parameter: $P_3$\newline
Variables: $\{c_{ij}\}_{i \in H, j \in S_{1} \cup S_{-1}}$,  $\{b_i\}_{i \in H}$
\begin{align}
    &\max \: \gamma \\
    & \frac{1}{2}\mathbb{E}_{i \in H} \phi(b_i + c_{ij})  \geq \gamma \qquad  \forall j \in S_1\\
    & \frac{1}{2}\mathbb{E}_{i \in H} \phi(-b_i + c_{ij}) \geq \gamma \qquad  \forall j \in S_{-1}\\
    &\mathbb{E}_{i \in H}\left(b_i^2 + \frac{1}{d\sigma^2}\sum_{j \in S_1 \cup S_{-1}}{c_{ij}^2}\right) \leq P_3,
\end{align}
where sets $S_{1}, S_{-1} \subset[n] $ with $|S_{1}| = |S_{-1}| = n_{\on{min}}$, and $H \subset [m]$ with $|H| = \frac{m}{2}$.
\end{definition}

\begin{definition}[Opt 4]\label{program:no_j}\ \\
Parameter: $P_4$\newline
Variables: $\{c_{i}\}_{i \in H}$,  $\{d_{i}\}_{i \in H}$, $\{b_i\}_{i \in H}$
\begin{align}
    &\max \: \gamma \\
    & \frac{1}{2}\mathbb{E}_{i \in H} \phi(b_i + c_{i})  \geq \gamma \\
    & \frac{1}{2}\mathbb{E}_{i \in H} \phi(-b_i + d_{i}) \geq \gamma \\
    &\mathbb{E}_{i \in H}\left(b_i^2 + \frac{n_{\on{min}}}{d\sigma^2}(c_i^2 + d_i^2)\right) \leq P_4,
\end{align}
where $H \subset [m]$ with $|H| = \frac{m}{2}$.
\end{definition}

\begin{definition}[Opt 5: Trivariate Simplification]\label{trivariate}\ \\
Parameter: $P_5$\newline
Variables: $c, d, b$
\begin{align}
    &\max \frac{1}{4}\left(\phi(b + c) +  \phi(-b + d)\right) \\
    & b^2 +  \frac{\hat{\kappa}}{4} (c^2 + d^2) \leq P_5,
\end{align}
where $\hat{\kappa} = \frac{4n_{\on{min}}}{d\sigma^2}$.
\end{definition}

For $i \in \{1, 2, 3, 4, 5\}$, let $D_i$ denote the domain of parameters and variables in the program Opt $i$. For an instance of program Opt $i$ in $D_i$, we say it is $(1 - \epsilon)$-optimal if the objective value given by the variables is at least $(1 - \epsilon)$ times the maximum objective value for the parameters in the instance. We use $\ins_i \in D_i$ to denote an instance of the program Opt $i$. When such an instance is fixed, we will freely use the names of the parameters and the variables associated with Opt $i$ to refer to the variables and parameters of $\ins_i$. For instance, if $\ins_1 \in D_1$, then $W$ is the variable associated with $\ins_1$.

\subsubsection{Chaining Lemmas.}

Throughout the following section, we globally assume that $\frac{d}{n} \geq 4C_{\ref{lemma:concentration}}^2$ and $\epsilon \leq \frac{1}{4}$, such that on the condition that Lemma~\ref{lemma:concentration} holds, we have $C_{\ref{lemma:concentration}}\sqrt{\frac{n}{d}} \leq \frac{1}{2}$. Such conditions are assumed in the theorems that follow from these lemmas, so there is no harm in making the assumption now. All of the lemmas in this section are proved in Section~\ref{sec:xor_technical_proofs_chain}.

\begin{restatable}[Opt 1 $\leftrightarrow$ Opt 2]{lemma}{mapone}\label{lemma:opt_1_relax_1}
Assume the conclusion of Lemma~\ref{lemma:concentration} holds for $\Xi$.

Define the mapping $\psi_{12}: D_1 \rightarrow D_2$ as follows. Given input $\ins_1$ with variable $W = U + V$, output:
\begin{itemize}
    \item $P_2 = P_1\left(1 + C_{\ref{lemma:concentration}}\sqrt{\frac{n}{d}}\right)$
    \item $c_{ij} = v_i^T\xi_j$ for all $i \in [m]$, $j \in [n]$
    \item $s_i = \mu_1^Tw_i$ for all $i \in H_+$
    \item $t_i = \mu_2^Tw_i$ for all $i \in H_-$
\end{itemize}

Define the mapping $\psi_{21}: D_1 \rightarrow D_2$ as follows: Given input $\ins_2$, output $\ins_1$ as follows:
\begin{itemize}
    \item $P_1 = \frac{P_2}{1 - C_{\ref{lemma:concentration}}\sqrt{\frac{n}{d}}}$.
    \item For all $i \in H_+$, let $u_i = s_i\mu_1$, and choose $v_i$ to be the min-norm vector such that $v_i^T\xi_j = c_{ij}$ for all $j \in \mathcal{P}$ and $v_i^T\xi_j = 0$ for all $j \in \mathcal{N}$.
    \item For all $i \in H_-$, let $u_i = t_i\mu_2$, and choose $v_i$ to be the min-norm vector such that $v_i^T\xi_j = c_{ij}$ for all $j \in \mathcal{N}$ and $v_i^T\xi_j = 0$ for all $j \in \mathcal{P}$.
\end{itemize}

Then with $\eps' = \sqrt{1 - (1 - \eps)\left(1 - C_{\ref{lemma:concentration}}\sqrt{\frac{n}{d}}\right)^h}$,
\begin{enumerate}
    \item If $\ins_1 \in D_1$ is $(1 - \epsilon)$-optimal, then $\psi_{12}(\ins_1)$ is $1 - \eps'$-optimal on Opt 1 and has objective value at most $\frac{1}{1 - \eps'}$ larger than the objective of $\ins_1$.
    \item 
    If $\ins_2 \in D_2$ is $(1 - \epsilon)$-optimal, then $\psi_{21}(\ins_2)$ is $1 - \eps'$-optimal on Opt 1.
\end{enumerate}
\end{restatable}

The following lemma states that in a near-optimal solution to Opt 1, most of the contribution to the norm constraint comes from the the variables that get used in the mapping $\psi_{12}$ to Opt 2.
\begin{lemma}\label{lemma:small_orthogonal}
Assume the conclusion of Lemma~\ref{lemma:concentration} holds for $\Xi$. Then any solution $W = U+V$ to Opt 1 with $\|W\| = 1$ that is $(1 - \epsilon)$-optimal must satisfy:
\begin{enumerate}
    \item  $\frac{1}{2}\mathbb{E}_{i \in H_+}\left[\|\mu_2^Tw_i\|^2 + \frac{1}{d\sigma^2}\sum_{j \in \mathcal{N}}{(v_i^T\xi_j)^2}\right] + \frac{1}{2}\mathbb{E}_{i \in H_-}\left[\|\mu_1^Tw_i\|^2 + \frac{1}{d\sigma^2}\sum_{j \in \mathcal{P}}{(v_i^T\xi_j)^2}\right] < (\epsilon'_{\ref{lemma:small_orthogonal}})^2$
    \item For at least a $1 - \epsilon'_{\ref{lemma:small_orthogonal}}$ fraction of the data points $j$, we have $\frac{1}{2}\mathbb{E}_{i: \on{sign}(a_i) = -y_j}\left[(v_i^T\xi_j)^2\right] \leq \frac{1}{\kappa}\cdot \epsilon'_{\ref{lemma:small_orthogonal}}$.
\end{enumerate}
where $\epsilon'_{\ref{lemma:small_orthogonal}} = \epsilon'_{\ref{lemma:small_orthogonal}}(\epsilon) := \sqrt{2C_{\ref{lemma:concentration}}\sqrt{\frac{n}{d}} + 2\epsilon}$.
\end{lemma}

\begin{restatable}[Opt 2 $\leftrightarrow$ Opt 3]{lemma}{maptwo}\label{lemma:relax_1_relax_2}
Define the mapping $\psi_{23}: D_2 \rightarrow D_3 \times D_3$ as follows. Given input $\ins_2$, output $\ins_3^{(1)}$, $\ins_3^{(2)}$, where for $\ins_3^{(1)}$:
\begin{itemize}
    \item  $H := H_+$ 
    \item $P_3 := \frac{1}{2}\mathbb{E}_{i \in H_+}\left(s_i^2 + \frac{1}{d\sigma^2}\sum_{j \in \mathcal{P}}c_{ij}^2\right)$
    \item Let $S_1$ be an arbitrary set of $n_{\on{min}}$ elements of $\mc{P}_1$, and let $S_{-1}$ be an arbitrary set of $n_{\on{min}}$ elements of $\mc{P}_{-1}$. Define $c_{ij}$ to be the same as in $\ins_2$ for all $j \in S_1 \cup S_{-1}$, $i \in H_+$.
    \item $b_i = s_i$ for $i \in H_+$,
\end{itemize}
and for $\ins_3^{(2)}$:
\begin{itemize}
    \item  $H := H_-$
    \item $P_3 := \mathbb{E}_{i \in H_-}\left(t_i^2 + \frac{1}{d\sigma^2}\sum_{j \in \mathcal{N}}c_{ij}^2\right)$
    \item Let $S_1$ be an arbitrary set of $n_{\on{min}}$ elements of $\mc{N}_1$, and let $S_{-1}$ be an arbitrary set of $n_{\on{min}}$ elements of $\mc{N}_{-1}$. Define $c_{ij}$ to be the same as in $\ins_2$ for all $j \in S_1 \cup S_{-1}$, $i \in H_-$.
    \item $b_i = t_i$ for $i \in H_-$,
\end{itemize}
Define the mapping $\psi_{32}: D_3 \rightarrow D_2$ as follows. Given input $\ins_3$, output $\ins_2$, where
\begin{itemize}
    \item $P_2 := P_3\left(\frac{n_{\on{max}}}{n_{\on{min}}}\right)$
    \item Define an arbitrary bisections $\pi_+$ and $\pi_-$ from $H_+$ and $H_-$ respectively to $H$. Similarly define surjections $\rho_{+, 1} : \mathcal{P}_1 \rightarrow S_1$, $\rho_{-, 1} : \mathcal{P}_{-1} \rightarrow S_{-1}$, $\rho_{+, 1} : \mathcal{N}_1 \rightarrow S_1$, $\rho_{-, 1} : \mathcal{N}_{-1} \rightarrow S_{-1}$, such that for $x \in \{1, -1\}$, $\mathbb{E}_{j \in \mc{P}_x}\mathbb{E}_{i \in H}c_{i\rho_{+ , x}(j)}^2 \leq \mathbb{E}_{j \in S_x}\mathbb{E}_{i \in H}c_{ij}^2$, and similarly, $\mathbb{E}_{j \in \mc{N}_x}\mathbb{E}_{i \in H}c_{i\rho_{- , x}(j)}^2 \leq \mathbb{E}_{j \in S_x}\mathbb{E}_{i \in H}c_{ij}^2$. For $i \in H_{+}$, define $s_i := b_{\pi_+(i)}$ and $c_{ij} := c_{\pi_+(i)\rho_{+, x}(j)}$ for all $x \in \{1, -1\}$ and $j \in \mathcal{P}_x$. For $i \in H_{-}$, define $t_i := b_{\pi_-(i)}$, and $c_{ij} := c_{\pi_-(i)\rho_{+, x}(j)}$ for all $x \in \{1, -1\}$ and $j \in \mathcal{N}_x$. Note that we here the $c_{ij}$ variables we are \em defining \em come belong to $\ins_2$, and they are define in terms of the $c_{ij}$ variables from $\ins_3$.
\end{itemize}
Then with $\eps' = \sqrt{1 - (1 - \eps)\left(\frac{n_{\on{max}}}{n_{\on{min}}}\right)^{-h}}$
\begin{enumerate}
    \item If $\ins_2 \in D_2$ is $(1 - \epsilon)$-optimal, then each instance of $\psi_{23}(\ins_2)$ is $(1 - \eps')$-optimal on Opt 3, and has objective at most $\frac{1}{1 - \eps'}$ times the objective of $\ins_2$.
    \item If $\ins_3 \in D_3$ is $(1 - \epsilon)$-optimal, then $\psi_{32}(\ins_3)$ is $(1 - \eps')$-optimal on Opt 2.
\end{enumerate}
\end{restatable}

\begin{restatable}[Opt 3 $\leftrightarrow$ Opt 4]{lemma}{mapthree}\label{lemma:relax_2_relax_3}
Define the mapping $\psi_{34}: D_3 \rightarrow D_4^{n_{\on{min}}}$ as follows. Arbitrarily choose a list of $n_{\on{min}}$ pairs $p = (j, j')\in S_1 \times S_{-1}$, such that each $j \in S_1$ appears in one pair, and each $j' \in S_{-1}$ appears in one pair. Given input $\ins_3$, output $\ins_3^{(p)}$ for each pair $p$ as follows:
\begin{itemize}
    \item $P_4 := \mathbb{E}_{i \in H}\left(b_i^2 + \frac{n_{\on{min}}}{d\sigma^2}(c_{ij}^2 + c_{ij'}^2)\right)$
    \item Keep $H$ and all the $b_i$ the same as in $\ins_3$.
    \item Put $c_i := c_{ij}$ and $d_i := c_{ij'}$ for all $i \in H$.
\end{itemize}
In reverse, define the mapping $\psi_{43}: D_4 \rightarrow D_3$ as follows:
\begin{itemize}
    \item Put $P_3 := P_4$.
    \item Keep $H$ and all the $b_i$ the same as in $\ins_3$.
    \item For all $i \in H$, put $c_{ij} := c_{i}$ for all $j \in S_1$ and $c_{ij} := d_i$ for all $j \in S_{-1}$.
\end{itemize}
Then:
\begin{enumerate}
    \item If $\ins_3 \in D_3$ is $(1 - \epsilon)$-optimal, then on at least a $1 - \sqrt{\eps}$ fraction of the $\frac{n}{4}$ instances of $\psi_{34}(\ins_4)$, the instance is $(1 - \sqrt{\eps})$-optimal on Opt 4 and has objective value at most $\frac{1}{1 - \sqrt{\eps}}$ larger than the objective of $\ins_3$.
    \item 
    If $\ins_4 \in D_4$ is $(1 - \epsilon)$-optimal, then $\psi_{43}(\ins_4)$ is $(1 - \sqrt{\eps})$-optimal on Opt 3.
\end{enumerate}
\end{restatable}

\begin{restatable}[Opt 5 $\leftrightarrow$ Opt 4]{lemma}{mapfour}\label{lemma:symmetrize}
define the mapping $\psi_{45}: D_4 \rightarrow D_5^{\frac{m}{2}}$ as follows. Given input $\ins_4$, output $\frac{m}{2}$ instances $\ins_5^{(i)}$ for $i \in H$ with
\begin{itemize}
    \item $P_5^{(i)} := b_i^2 + \frac{\hat{\kappa}}{4}(c_i^2 + d_i^2)$.
    \item $(b, c, d)^{(i)} = (b_i, c_i, d_i)$.
\end{itemize}
Define the mapping $\psi_{54}: D_5 \rightarrow D_4$ as follows. Given input $\ins_5$, output:
\begin{itemize}
    \item $P_4 := P_5$
    \item For half of the $i \in H$, put $(b_i, c_i, d_i) = (b, c, d)$.
    \item For the other half of the $i \in H$, put $(b_i, c_i, d_i) = (-b, d, c)$.
\end{itemize}
Then:
\begin{enumerate}
    \item If $\ins_4 \in D_4$ is $(1 - \epsilon)$-optimal, then the average objective value of the $\frac{m}{2}$ instances of $\psi_{45}(\ins_4)$ is at most $\frac{1}{1 - \eps}$ times larger than the objective of $\ins_4$.
    \item If $\ins_5 \in D_5$ is $(1 - \epsilon)$-optimal, then $\psi_{54}(\ins_5)$ is $(1 - \sqrt{\eps})$-optimal on Opt 4.
\end{enumerate}
\end{restatable}
Our main tool in proving these chaining lemmas is the following analysis lemma. We first state a definition.

\begin{definition}
An optimization program is $q$-homogeneous with with respect to a parameter $P$ if the optimal objective equals $CP^q$, for some fixed value $C$.
\end{definition}

\begin{lemma}\label{mappings}
Consider two optimization programs Opt A and Opt B with domains $D_A$ and $D_A$ with are $q$-homogeneous with respect to parameters $P_A$ and $P_B$ respectively. 
Suppose for some positive integer $k$, we have a mapping $\psi_{AB}: D_A \rightarrow D_B^k$ and $\psi_{BA}: D_B \rightarrow D_A$. Suppose for any feasible instances $\ins_A \in D_A$ and $\ins_B \in D_B$:
\begin{enumerate}
    \item All $k$ instances of $\psi_{AB}(\ins_A)$ are feasible, and have at least the same objective value as $\ins_A$. Similarly $\psi_{BA}(\ins_B)$ is feasible and has at least the same objective value as $\ins_B$.
    \item The average parameter $P_B$ of the $k$ instances of $\psi_{AB}(\ins_A)$ it at most $(1 + \delta)$ times the parameter $P_A$ of $\ins_A$.
    \item $P_A(\psi_{BA}(\ins_B)) \leq (1 + \delta)P_B(\ins_B)$. (Here the notation $P(\ins)$ refers to the parameter $P$ in an instance $\ins$.)
\end{enumerate}
Then letting $\eps' = \sqrt{1 - (1 - \eps)(1 + \delta)^{-2q}}$,
\begin{enumerate}
    \item If $\ins_A$ is $(1 - \epsilon)$-optimal, then for at least a $1 - \eps'$ fraction of the $k$ instances $\psi_{AB}(\ins_B)$ are $(1 - \eps')$-optimal and have objective value at most $\frac{1}{1 - \eps'}$ times the objective of $\ins_A$. 
    \item  If $\ins_B$ is $(1 - \eps)$-optimal, then $\psi_{BA}(\ins_B)$ is $(1 - \eps')$-optimal.
\end{enumerate}
In particular, if $\delta = 0$, $\eps' = \sqrt{\eps}$.
\end{lemma}

\subsubsection{Analysis of Opt 5: Trivariate Program.}
The lemmas in this section are proved in Section~\ref{sec:xor_technical_proofs_tri}.

Define $\gamma_0(\kappa) := 2\phi\left(\sqrt{\frac{2}{\kappa}}\right)$, and let $\gamma_*(\kappa) := \phi\left(\sqrt{\frac{\kappa}{4 + \kappa}} + \sqrt{\frac{16}{\kappa(4 + \kappa)}}\right)$. It is straightforward to check that $\gamma_0(\hat{\kappa})$ is $4$ times the optimum of Opt 5 when $P_5 = 1$, and we impose the additional constraint that $b = 0$. Similarly, $\gamma_*(\hat{\kappa})$ is $4$ times the optimum of Opt 5 when we impose the additional constraint that $d = 0$.

Recall that we have defined $\kgen$ to be the threshold at which $\gamma_*(\kgen) = \gamma_0(\kgen)$, and $\kphen$ to be the threshold in $\kappa$ at which $\sqrt{\frac{\kappa}{4 + \kappa}} = \sqrt{\frac{16}{\kappa(4 + \kappa)}}$. Observe that $\kphen = 4$.

The following lemma yields the optimal solution to the program Opt 5.
\begin{lemma}\label{tri_opt}
Let $k := \frac{\hat{\kappa}}{4}$ and assume $P_5 = 1$. If $k > \kgen/4$, then if we impose the additional constraint that $b \geq 0$, the supremum of Opt 5 (in Definition~\ref{trivariate}) and it is achieved uniquely at the point where $d = 0$, $b = \sqrt{\frac{k}{1 + k}}$, and $c = \sqrt{\frac{1}{k(1 + k)}}$. Outside of any neighborhood of this point, the supremum is bounded away from from the supremum of Opt 5.

If $k < \kgen/4$, then supremum is achieved by some optimal point with $b = 0$.
\end{lemma}

\begin{lemma}\label{lemma:trivariate_analysis}
There exists strictly positive constants $\epsilon = \epsilon(\hat{\kappa})$, $\eta = \eta(\hat{\kappa})$, and $q = q(\hat{\kappa})$, such that any $(1 - \epsilon)$-optimal solution to Opt 5 (in Definition~\ref{trivariate}), the following holds.
If $\hat{\kappa} > \kgen$, then
\begin{enumerate}
    \item $\phi(b) \geq \eta \phi(b + c)$; and
    \item $\phi(-b) \geq \eta \phi(-b + d)$.
\end{enumerate}
If additionally $\hat{\kappa} < \kphen$, then any such solution also satisfies
\begin{enumerate}
    \item $\phi(b) \leq \frac{1 - q}{2^h}\phi(b + c)$; and
     \item $\phi(-b) \leq \frac{1 - q}{2^h}\phi(-b + d)$.
\end{enumerate}
Finally, if $\hat{\kappa} < \kgen$, then at the optimum, $\phi(b) = \phi(-b) = 0$.
\mg{Can make $\eta = \left(\frac{1}{1 + \frac{4}{\kappa}}\right)^h$ minus arbitrarily small constant. (see proof) This might be useful for large $\kappa$. We can also show in this regime that $c^2 + d^2$ is small ( a factor of $\Theta(\kappa^2)$ smaller than $b^2$), which might be worth it if we want better generalization results for $\kappa = \Omega(1)$.}
\end{lemma}

The following lemma is a more tailored version of a chaining lemma between Opt 4 and Opt 5, which explicitly leverages the previous lemmas on the solution of Opt 5.
\begin{lemma}[Opt 4 $\rightarrow$ Opt 5]\label{lemma:relax_3_relax_4}
Suppose $\hat{\kappa} > \kgen$. There exists positive constants $\epsilon = \epsilon(\hat{\kappa})$, $\eta = \eta(\hat{\kappa})$, and $q = q(\hat{\kappa})$ such that for any $(1 - \epsilon)$-optimal solution to the program Opt 4 in Definition~\ref{program:no_j},
\begin{align}
    & \mathbb{E}_{i \in H} \phi(b_i)  \geq \frac{\eta}{2}\mathbb{E}_{i \in H} \phi(b_i + c_i);\\
    & \mathbb{E}_{i \in H} \phi(-b_i)  \geq \frac{\eta}{2}\mathbb{E}_{i \in H} \phi(-b_i + d_i).\\
\end{align}
If additionally $\hat{\kappa} < \kphen$, 
\begin{align}
    \mathbb{E}_{i \in H}[\phi(b_i)] &\leq \left(\frac{1 - q/2}{2^h}\right) \mathbb{E}_{i \in H}[\phi(b_i + c_i)];\\
    \mathbb{E}_{i \in H}[\phi(-b_i)] &\leq \left(\frac{1 - q/2}{2^h}\right) \mathbb{E}_{i \in H}[\phi(-b_i + d_i)].
\end{align}
Here $\eta(\hat{\kappa}), q(\hat{\kappa}) > 0$ are the constants from Lemma~\ref{lemma:trivariate_analysis}.
\end{lemma}

\subsection{Key Lemmas for Main Results.}

Putting together the results of Lemmas~\ref{lemma:opt_1_relax_1}-\ref{lemma:relax_3_relax_4}, we carry out the chain of reductions from Opt 1 though Opt 5 to achieve the following results.



\begin{lemma}[Generalization Lemma]\label{lemma:xor_gen}
Assume $\kappa > \kgen$. There exists strictly positive constants $\epsilon = \epsilon(\kappa)$ and $c = c(\kappa)$ and $\eta = \eta(\kappa)$ such that for any solution to Opt 1 which achieves a margin $\gamma$ that is at least $(1 - \epsilon)$-optimal, the following holds. For $\frac{d}{n} \geq c$, with probability $1 - 3e^{-n/c}$ over the training data, we have
\begin{enumerate}
    \item $\frac{1}{2}\mathbb{E}_{i \in H_+}[\phi(\mu_1^Tw_i)] \geq \frac{\gamma \eta}{2}$ 
    \item $\frac{1}{2}\mathbb{E}_{i \in H_+}[\phi(-\mu_1^Tw_i)] \geq \frac{\gamma \eta}{2}$ 
    \item $\frac{1}{2}\mathbb{E}_{i \in H_-}[\phi(\mu_2^Tw_i)] \geq \frac{\gamma \eta}{2}$ 
    \item $\frac{1}{2}\mathbb{E}_{i \in H_-}[\phi(-\mu_2^Tw_i)] \geq \frac{\gamma \eta}{2}$.
\end{enumerate}
\end{lemma}
\begin{proof}
Condition on the event in Lemma~\ref{lemma:concentration} holding for $\Xi$ and the event in Lemma~\ref{lemma:cluster_size} holding for $\beta = \frac{1}{c_1}$, for some constant $c_1 = c_1(\kappa) > 8$ to be chosen later. These events occur with probability at least $\min(0, 1 - 3e^{-n} - 8e^{-8n/c_1^2}) \geq 1 - 3e^{-n/(c_1^2 + 1)}$ for $n \geq 1$. If we begin with an instance $\ins_1$ which is $(1 - \epsilon)$-optimal on Opt 1, then:
\begin{enumerate}
    \item $\ins_2 := \psi_{12}(\ins_1)$ is $1 - \eps_2 := \sqrt{\left((1 - \epsilon)\left(1 - C_{\ref{lemma:concentration}}\sqrt{\frac{n}{d}}\right)^{h}\right)} \geq \sqrt{\left((1 - \epsilon)\left(1 - C_{\ref{lemma:concentration}}\sqrt{\frac{1}{c}}\right)^{h}\right)} $- optimal on Opt 2 (Lemma~\ref{lemma:opt_1_relax_1})
    \item Both instances $(\ins_3^{(1)}, \ins_3^{(2)}) := \psi_{23}(\ins_2)$ are $1 - \eps_3 := 1 - \sqrt{1 - (1 - \eps)\left(\frac{n_{\on{max}}}{n_{\on{min}}}\right)^{-h}} \geq 1 - \sqrt{1 - (1 - \eps)\left(1 + \frac{16}{c_1}\right)^{-h}}$-optimal on Opt 3 (Lemma~\ref{lemma:relax_1_relax_2})
    \item For at least one pair $p = (j, j') \in \mathcal{P}_1 \times \mathcal{P}_{-1}$, the instance $\ins_4^{(p)}$ of $\psi_{34}(\ins_3^{(1)})$ indexed by that pair is $1 - \eps_4 := 1 - \sqrt{\eps_3}$-optimal on Opt 4. (Lemma~\ref{lemma:relax_2_relax_3}). The same holds for some $(j, j') \in \mathcal{N}_1 \times \mathcal{N}_{-1}$ on $\psi_{34}(\ins_3^{(2)})$.
\end{enumerate}
Now choose $c_1$ large enough and $\epsilon$ small enough constants such that for $c \geq c_1$, we have $\hat{\kappa} = \left(\frac{n_{\on{min}}}{n_{\on{max}}}\right)\kappa \geq \frac{\kgen + \kappa}{2}$, and $\eps_4$ is less than the value $\min_{\kappa' \in \left[\frac{\kgen + \kappa}{2}, \kappa\right]} \epsilon(\kappa')$ from Lemma~\ref{lemma:relax_3_relax_4}. Thus applying Lemma~\ref{lemma:relax_3_relax_4}, we observe that on the pair $(j, j') \in \mathcal{P}_1 \times \mathcal{P}_{-1}$, we have 
\begin{align}
    & \mathbb{E}_{i \in H} \phi(b_i)  \geq \frac{\eta}{2}\mathbb{E}_{i \in H} \phi(b_i + c_i) \geq \frac{\eta \gamma}{2};\\
    & \mathbb{E}_{i \in H} \phi(-b_i)  \geq \frac{\eta}{2}\mathbb{E}_{i \in H} \phi(-b_i + d_i) \geq \frac{\eta \gamma}{2},\\
\end{align}
where $\gamma$ is the objective value of $\ins_4^{((j, j'))}$, and $\eta = \eta(\kappa) := \min_{\kappa' \in \left[\frac{\kgen + \kappa}{2}, \kappa\right]} \eta_{\ref{lemma:relax_3_relax_4}}(\kappa')$ where $\eta_{\ref{lemma:relax_3_relax_4}}(\cdot)$ is the positive constant called $\eta$ from Lemma~\ref{lemma:relax_3_relax_4}.
By definition of the mapping $\psi_{34}$, this means that in $\ins_3^{(1)}$,
\begin{align}
    & \mathbb{E}_{i \in H} \phi(b_{i})  \geq \frac{\eta \gamma}{2};\\
    & \mathbb{E}_{i \in H} \phi(-b_{i})  \geq \frac{\eta \gamma}{2},\\
\end{align}
where $\gamma$ is the objective value of $\ins_3^{(1)}$. By definition of the mapping $\psi_{23}$, this means that in $\ins_2$,
\begin{align}
    & \mathbb{E}_{i \in H_+} \phi(s_{i})  \geq \frac{\eta \gamma}{2};\\
    & \mathbb{E}_{i \in H_+} \phi(-s_{i})  \geq \frac{\eta \gamma}{2},\\
\end{align} where $\gamma$ is the objective value of $\ins_2$. Finally by definition of $\psi_{12}$, in $\ins_1$,
\begin{align}
\frac{1}{2}\mathbb{E}_{i \in H_+}[\phi(\mu_1^Tw_i)] &\geq \frac{\gamma \eta}{2}\\
\frac{1}{2}\mathbb{E}_{i \in H_+}[\phi(-\mu_1^Tw_i)] &\geq \frac{\gamma \eta}{2},
\end{align}
where $\gamma$ is the objective value of Opt 1. This yields the first two conclusions of the lemma. The second follows via an identical argument on the pair $(j, j') \in \mc{N}_{1} \times \mc{N}_{-1}$. Choosing $c = c_1^2 + 1$ yields the result with probability at least $1 - 3e^{-n/c}$.
\end{proof}

We can now put together the results of the chain of reductions to prove the following:

\begin{lemma}[Phenomenon Lemma]\label{lemma:phen}
Assume $\kgen < \kappa < \kphen$. For any constant $\delta > 0$, there exists strictly positive constants $\epsilon = \epsilon(\kappa, \delta)$ and $c = c(\kappa, \delta)$ and $q = q(\kappa)$ such that for any solution to Opt 1 which achieves a margin $\gamma$ that is at least $(1 - \epsilon)$-optimal, the following holds. For $\frac{d}{n} \geq c$, with probability $1 - 3e^{-n/c}$ over the training data, 
\begin{align}
    &\frac{1}{2}\mathbb{E}_{i \in H_+}[\phi(u_i^T\mu_1)] \leq  \frac{\gamma(1 - q/4)}{2^h};\\
    &\frac{1}{2}\mathbb{E}_{i \in H_-}[\phi(u_i^T\mu_2)] \leq  \frac{\gamma(1 - q/4)}{2^h},
\end{align} and for at least a $1 - \delta$ fraction of $j \in [n]$,
\begin{align}
    \frac{1}{2}\mathbb{E}_{i : \on{sign}(a_i) = y_j}[\phi(v_i^Tx_j)] \geq \frac{\gamma(1 + q/4)}{2^h}.
\end{align}
\end{lemma}
\begin{proof}
Condition on the event in Lemma~\ref{lemma:concentration} holding for $\Xi$ and the event in Lemma~\ref{lemma:cluster_size} holding for $\beta = \frac{1}{c_1}$, for some constant $c_1 = c_1(\delta, \kappa) > 8$ to be chosen later. These events occur with probability at least $\min(0, 1 - 3e^{-n} - 8e^{-8n/c_1^2}) \geq 1 - 3e^{-n/(c_1^2 + 1)}$. If we begin with an instance $\ins_1$ which is $(1 - \epsilon)$-optimal on Opt 1, then:
\begin{enumerate}
    \item $\ins_2 := \psi_{12}(\ins_1)$ is $1 - \eps_2 := \sqrt{\left((1 - \epsilon)\left(1 - C_{\ref{lemma:concentration}}\sqrt{\frac{n}{d}}\right)^{h}\right)} \geq \sqrt{\left((1 - \epsilon)\left(1 - C_{\ref{lemma:concentration}}\sqrt{\frac{1}{c}}\right)^{h}\right)} $- optimal on Opt 2 (Lemma~\ref{lemma:opt_1_relax_1})
    \item Both instances $(\ins_3^{(1)}, \ins_3^{(2)}) := \psi_{23}(\ins_2)$ are $1 - \eps_3 := 1 - \sqrt{1 - (1 - \eps)\left(\frac{n_{\on{max}}}{n_{\on{min}}}\right)^{-h}} \geq 1 - \sqrt{1 - (1 - \eps)\left(1 + \frac{16}{c_1}\right)^{-h}}$-optimal on Opt 3 (Lemma~\ref{lemma:relax_1_relax_2})
    \item Let $\{\ins_4^{((p)}\}_{p \in L_1} := \psi_{34}(\ins_3^{(1)})$ be the $n_{\on{min}}$ instances of Opt 4 indexed by some list $L_1$ of $n_{\on{min}}$ pairs $p \in \mathcal{P}_1 \times \mathcal{P}_{-1}$. Similarly let $\{\ins_4^{((p)}\}_{p \in L_2} := \psi_{34}(\ins_3^{(2)})$ for some list of $n_{\on{min}}$ pairs in $ \mathcal{N}_1 \times \mathcal{N}_{-1}$.
    For at least and $1 - \eps_4 := 1 - \sqrt{\eps_3}$ fraction of pairs $(j, j')$ in $L_1$, $\ins_4^{((j, j'))}$ is $1 - \eps_4$-optimal on Opt 4. The same hols for a $1 - \eps_4$ fraction of the pairs in $L_2$. (Lemma~\ref{lemma:relax_2_relax_3}). For each cluster in $\mc{P}_1, \mc{P}_{-1}, \mc{N}_1, \mc{N}_{-1}$, at least a $\frac{n_{\on{\min}}}{n_{\on{\max}}} \geq \frac{1}{1 + \frac{16}{c_1}}$ fraction of the points $j$ in that cluster appear in a pair in one of the lists $L_1$ or $L_2$.
\end{enumerate}
Now choose $c_1$ large enough and $\epsilon$ small enough constants such that for $c \geq c_1$, we have $\hat{\kappa} = \left(\frac{n_{\on{min}}}{n_{\on{max}}}\right)\kappa \geq \frac{\kgen + \kappa}{2}$, and $\frac{\eps_4}{1 + \frac{16}{c_1}}$ is both less than $\delta$ and than $\min_{\kappa' \in \left[\frac{\kgen + \kappa}{2}, \kappa\right]} \epsilon(\kappa')$, where $\eps(\cdot)$ is the function from Lemma~\ref{lemma:relax_3_relax_4}.

Applying this Lemma~\ref{lemma:relax_3_relax_4}, we observe that on at least a $1 - \eps_4$ fraction of pairs $(j, j')$ in the list $L_1$, we have
\begin{align}
    \mathbb{E}_{i \in H}[\phi(b_i)] &\leq \left(\frac{1 - q/2}{2^h}\right) \mathbb{E}_{i \in H}[\phi(b_i + c_i)];\\
    \mathbb{E}_{i \in H}[\phi(-b_i)] &\leq \left(\frac{1 - q/2}{2^h}\right) \mathbb{E}_{i \in H}[\phi(-b_i + d_i)],
\end{align}
where $q = q(\kappa) := \min_{\kappa' \in \left[\frac{\kgen + \kappa}{2}, \kappa\right]} q_{\ref{lemma:relax_3_relax_4}}(\kappa')$ where $q_{\ref{lemma:relax_3_relax_4}}(\cdot)$ is the positive constant called $q$ from Lemma~\ref{lemma:relax_3_relax_4}.
We first use this to prove the first conclusion the lemma for $j$ (the argument is analogous for $j'$). By definition of the mapping $\psi_{34}$, this means that in $\ins_3^{(1)}$,
\begin{align}
    \mathbb{E}_{i \in H}[\phi(b_i)] &\leq \left(\frac{1 - q/2}{2^h}\right) \mathbb{E}_{i \in H}[\phi(b_i + c_{ij})];
\end{align} 
By definition of the mapping $\psi_{23}$, this means that in $\ins_2$,
\begin{align}
    \mathbb{E}_{i \in H_+}[\phi(s_i)] &\leq \left(\frac{1 - q/2}{2^h}\right) \mathbb{E}_{i \in H}[\phi(s_i + c_{ij})];
\end{align} 
Finally, by definition of the mapping $\psi_{12}$, this means that in $\ins_1$,
\begin{align}\label{eq:part_1}
    \mathbb{E}_{i \in H_+}[\phi(\mu_1^Tx_j)] &\leq \left(\frac{1 - q/2}{2^h}\right) \mathbb{E}_{i \in H_+}[\phi(\mu_1^Tx_j + v_i^T\xi_j)] = \left(\frac{1 - q/2}{2^h}\right) \mathbb{E}_{i \in H_+}[\phi(w_i^Tx_j)].
\end{align}

Now by Lemma~\ref{lem:activation}, for any values $s, t$, we have $\phi(s + t) \leq (\phi(s) + \phi(t))2^{h-1}$, and thus
\begin{align}\label{phen_copy1}
    \mathbb{E}_{i \in H_+}[\phi(u_i^Tx_j) + \phi(v_i^T\xi_j)] &\geq 2^{- h + 1}\mathbb{E}_{i \in H_+}[\phi(w_i^Tx_j)]
\end{align}
So by Equation~\ref{eq:part_1}, we have
\begin{align}\label{phen_copy2}
    \mathbb{E}_{i \in H_+}[\phi(v_i^T\xi_j)] &\geq \left(\frac{1 + q/2}{2^h}\right)\mathbb{E}_{i \in H_+}[\phi(w_i^Tx_j)].
\end{align}

Now to relate $\mathbb{E}_{i \in H_+}[\phi(w_i^Tx_j)]$ to $\gamma$, observe that by Lemma~\ref{lemma:opt_1_relax_1}, \ref{lemma:relax_1_relax_2}, and \ref{lemma:relax_2_relax_3}, the objective value of $\ins_4^{(j, j')}$, which equals $\min\left(\frac{1}{2}\mathbb{E}_{i \in H_+}\phi(w_i^Tx_j), \frac{1}{2}\mathbb{E}_{i \in H_+}\phi(w_i^Tx_{j'})\right)$, is at least $\gamma$ and at most $\frac{\gamma}{(1 - \eps_2)(1 - \eps_3)(1 - \eps_4)}$. Now we also know by the first conclusion of Lemma~\ref{lemma:symmetrize} that $\psi_{45}(\ins_4^{(j, j')})$ produces some instances $\ins_5^{(i)}$ for $i \in H_+$ with objective values $\gamma^{(i)}$, for which
\begin{align}
    \mathbb{E}_{i \in H^+}\gamma^{(i)} \leq \frac{1}{1 - \eps_4}\min\left(\frac{1}{2}\mathbb{E}_{i \in H_+}\phi(w_i^Tx_j), \frac{1}{2}\mathbb{E}_{i \in H_+}\phi(w_i^Tx_{j'})\right).
\end{align}
Plugging in the fact that $\gamma^{(i)} = \frac{1}{4}\left(\phi(b_i + c_i) + \phi(-b_i + d_i)  \right) = \frac{1}{4}\left(\phi(w_i^Tx_j) + \phi(w_i^Tx_{j'})  \right)$, we have 
\begin{align}
    \frac{1}{4}\mathbb{E}_{i \in H^+}\left(\phi(w_i^Tx_j) + \phi(w_i^Tx_{j'})  \right) \leq \frac{1}{1 - \eps_4}\min\left(\frac{1}{2}\mathbb{E}_{i \in H_+}\phi(w_i^Tx_j), \frac{1}{2}\mathbb{E}_{i \in H_+}\phi(w_i^Tx_{j'})\right).
\end{align}
Since $\min(\alpha, \beta) \geq (1 - \rho)\frac{a + b}{2}$ implies that $a, b \in \max(a, b) \leq  \frac{1 + \rho}{1 - \rho}$, we must have that 
\begin{align}
    \frac{1}{2}\mathbb{E}_{i \in H_+}\phi(w_i^Tx_j) \in \gamma\left[1, \frac{1 + \eps_4}{(1 - \eps_2)(1 - \eps_3)(1 - \eps_4)^2}\right].
\end{align}
Thus for $\epsilon$ small enough in terms of $q$, we have $\eps_2$, $\eps_3$, and $\eps_4$ all small enough that from Equations~\ref{eq:part_1} and \ref{phen_copy2}, we have 
\begin{align}\label{eq:final}
    \mathbb{E}_{i \in H_+}[\phi(\mu_1^Tx_j)] &\leq \left(\frac{1 - q/2}{2^h}\right) \mathbb{E}_{i \in H_+}[\phi(w_i^Tx_j)] \leq \left(\frac{1 - q/4}{2^{h-1}}\right)\gamma\\
    \mathbb{E}_{i \in H_+}[\phi(v_i^T\xi_j)] &\geq \left(\frac{1 + q/2}{2^h}\right)\mathbb{E}_{i \in H_+}[\phi(w_i^Tx_j)] \geq \left(\frac{1 + q/4}{2^{h-1}}\right)\gamma.
\end{align}

The argument holds for $j'$ in the pair with $j$, and we can also repeat an analogous argument for the $1 - \eps_4$ fraction of pairs in the list of pairs $L_2$ from $\mc{N}_1 \times \mc{N}_{-1}$. Now at least a $1 - \delta$ fraction of examples $j$ lie in a pair $p$ from $L_1$ or $L_2$ for which $\ins_4^{(p)}$ is $(1 - \eps_4)$-optimal. This yields the second part of the lemma which involves specific data points. The first point follows for the fact that Equation\ref{eq:final} only needs to hold for a single example $j$ in each of the four clusters. Choosing $c = c_1^2 + 1$ yields the result with probability at least $1 - 3e^{-n/c}$.

\end{proof}

\begin{lemma}[No Generalization Lemma]\label{lemma_no_gen}
For any $\kappa < \kgen$ and $\epsilon > 0$, there exists some positive constant $c(\epsilon)$, such that if $\frac{d}{n} \geq c$, with probability at least $1 - 3e^{-n/c}$ over $S \sim \mathcal{D}^n$, there exists a classifier $W$ with $\|W\| = 1$ such that
\begin{enumerate}
    \item $\gamma(f_W, S) \geq (1 - \epsilon)\gamma^*(S)$
    \item $U = 0$.
\end{enumerate}
\end{lemma}
\begin{proof}
We work backwards from Opt 5 through Opt 1. Condition on the event in Lemma~\ref{lemma:concentration} holding for $\Xi$ and the event in Lemma~\ref{lemma:cluster_size} holding for $\beta = \frac{1}{c_1}$, for some constant $c_1(\epsilon) > 8$ to be chosen later. Given an optimal solution $\ins_5^*$ to Opt 5, we can construct an instance $\ins_1 = \psi_{21}(\psi_{32}(\psi_{43}(\psi_{54}(\ins_5))))$, which is $\eps'$-optimal over all solutions $W'$ with the same norm for $\eps' = \sqrt{1 - (1 - \hat{\eps})(1 - C_{\ref{lemma:concentration}}\sqrt{\frac{n}{d}})^h}$, where $\hat{\eps} = \sqrt{1 - \left(\frac{n_{\on{max}}}{n_{\on{min}}}\right)^{-h}} \leq \sqrt{1 - \left(1 + \frac{16}{c_1}\right)^{-h}}$. This can be seen via Lemmas~\ref{lemma:opt_1_relax_1}, \ref{lemma:relax_1_relax_2}, \ref{lemma:relax_2_relax_3}, \ref{lemma:symmetrize}, which show that at each step of the chain, we do not lose any optimality expect from from Opt 3 to Opt 2 and from Opt 2 to Opt 1.

Now recall from Lemma~\ref{lemma:trivariate_analysis} that since $\kappa < \kgen$, for large enough $c_1$, we have $\hat{\kappa} < \kgen$, and thus the optimal solution to Opt 5 has $b = 0$. Applying the four mappings above, in the instance $\ins_1$, the variable $W = U + V$ has $U = 0$. Taking $c_1$ large enough such that for $c \geq c_1$, we have $\eps' \leq \eps$. If we choose $c(\eps) = c_1^2 + 1$, then the desired events hold with probability at least $1 - 3e^{-n/c}$ (see eg. Lemma~\ref{lemma:xor_gen} for the computation). Scaling $W$ to have $\|W\| = 1$ concludes the lemma.
\end{proof}

\subsection{Proofs of Main Results}\label{sec:xor_thms}
Using Lemma~\ref{lemma:xor_gen}, Lemma~\ref{lemma:small_orthogonal}, and Lemma~\ref{lemma:xor_spur_small}, we can prove Theorem~\ref{thm:xor_gen}. We restate the theorem for the reader's convenience.

\xorgen*

\begin{proof}[Proof of Theorem~\ref{thm:xor_gen}]
First note that $\kgen = \kgen^{\on{XOR}, h}$.
Let $W = U + V$ be the decomposition of $W$ into the signal space and the orthogonal space, such that $V \perp \on{span}(\mu_1, \mu_2)$. It suffices to consider $W$ with $\|W\| = 1$. Let $\gamma$ be the margin achieved by $f_W$, and observe that $\gamma$ is at least a positive constant since we can achieve a margin of $\frac{1}{4}$ by choosing a solution that only uses components in the signal subspace.

Recall that $\|U\| < 1$ and $\|V\| < 1$, and consider a random $x \sim \mathcal{D}$.

Choosing to $c_0 = c(\kappa)$ and $\epsilon(\kappa)$ to be the values from Lemma~\ref{lemma:xor_gen}, if $\frac{d}{n} \geq c_0$, with probability $1 - 3e^{-n/c_0}$ over the training data (and not $x$), the conclusion of Lemma~\ref{lemma:xor_gen} and Lemma~\ref{lemma:small_orthogonal} hold, and thus we have for such a $W$:
\begin{enumerate}
    \item $\frac{1}{2}\mathbb{E}_{i : \on{sign}(a_i) = y}[\phi(u_i^Tx)] \geq \frac{\gamma \eta(\kappa)}{2}$ by Lemma~\ref{lemma:xor_gen}, since $\kappa > \kgen$.
    \item $\frac{1}{2}\mathbb{E}_{i : \on{sign}(a_i) = -y}[\phi(u_i^Tx)] \leq \left(2\epsilon + 2C_{\ref{lemma:concentration}}\sqrt{\frac{n}{d}}\right)^{\frac{h}{2}}$. This is by Lemma~\ref{lemma:small_orthogonal}, we have $\frac{1}{2}\mathbb{E}_{i : \on{sign}(a_i) = -y}[\|u_i^Tz\|^2] \leq 2\epsilon + 2C_{\ref{lemma:concentration}}\sqrt{\frac{n}{d}},$ and thus by he homogeneity of the activation, $\frac{1}{2}\mathbb{E}_{i : \on{sign}(a_i) = -y}[\phi(u_i^Tx)] \leq \left(2\epsilon + 2C_{\ref{lemma:concentration}}\sqrt{\frac{n}{d}}\right)^{\frac{h}{2}}\max_{X: \mathbb{E}[X^2] = 1}{\mathbb{E}[\phi(X)]}$.By Jenson's inequality $\max_{X: \mathbb{E}[X^2] = 1}{\mathbb{E}[\phi(X)]} \leq \max_{X: \mathbb{E}[X^2] = 1}{\left(\mathbb{E}[X^2]\right)^{\frac{h}{2}}} = 1.$
\end{enumerate}
Thus with probability $1 - 3e^{-n/c_0}$ over the training data, \begin{align}
    yf_{U}(x) =\mathbb{E}_{i : \on{sign}(a_i) = y}[\phi(u_i^Tx)] - \mathbb{E}_{i : \on{sign}(a_i) = -y}[\phi(u_i^Tx)]\geq \frac{\gamma \eta(\kappa)}{2} - \left(2\epsilon + 2C_{\ref{lemma:concentration}}\sqrt{\frac{n}{d}}\right)^{\frac{h}{2}}.
\end{align}

For some constant $c_1 = c_1(\kappa)$, by Lemma~\ref{lemma:xor_spur_small}, with probability $1 - e^{-\frac{1}{c_1\sigma^2}}$ over $x$, $|f_{W}(x) - f_{U}(x)| < \frac{\gamma \eta(\kappa)}{4}$. Here we plugged in $t = \frac{1}{100 (\gamma^\frac{2}{h} + \gamma) \left(\eta(\kappa)^{\frac{2}{h}} + \eta(\kappa)\right)\sigma^2}$ to Lemma~\ref{lemma:xor_spur_small}, and note that $\gamma$ is at least a constant.

Thus for $\frac{d}{n} \geq c_2$ for some $c_2 = c_2(\kappa)$, we have with probability at least $1 - 3e^{-n/c_0}$, 
$$\frac{\gamma \eta(\kappa)}{2} - \left(2\epsilon + 2C_{\ref{lemma:concentration}}\sqrt{\frac{n}{d}}\right)^{\frac{h}{2}} \geq \frac{\gamma \eta(\kappa)}{4},$$ and thus
the loss is at most $1 - e^{-\frac{1}{c_1\sigma^2}} = 1 - e^{-\frac{\kappa d}{c_1 n}}$. Choosing $c = \max(c_0, c_1, c_2)$ yields the theorem.
\end{proof}

We now use Lemma~\ref{lemma:phen}, Lemma~\ref{lemma:small_orthogonal}, and Lemma~\ref{lemma:xor_spur_small} to prove Theorem~\ref{thm:one_sided_impossible_linear}(the XOR part) and Theorem~\ref{thm:margin_impossible} on the limitations of uniform convergence and inverse margin bounds.

To prove these results, we will demonstrate a certain phenomenon where given a near max-margin classifier $f_W$ for a set $S$, the classifier $f_W$ correctly classifies a certain ``opposite'' dataset $\psi(S)$ while still achieving good training error (or margin) on this opposite dataset. 

We define this opposite-mapping in the following definition.
\begin{definition}
For $\md = \dfull \in \Omega$, define the map $\psi: \mathbb{R}^d \rightarrow \mathbb{R}^d$ to keep $\xi$ the same, but maps $z$ to be in an orthogonal direction, as follows:
\begin{align}
    \psi((x, y)) = \begin{cases}
       (\mu_2 + \xi, 1) & (x, y) = (\mu_1 + \xi, 1)\\
        (-\mu_2 + \xi, 1) & (x, y) = (-\mu_1 + \xi, 1)\\
        (\mu_1 + \xi, -1) & (x, y) = (\mu_2 + \xi, -1)\\
        (-\mu_1 + \xi, -1) & (x, y) = (-\mu_2 + \xi, -1). \\
    \end{cases}
\end{align}
\end{definition}
When it is clear that we have fixed $\md$, we will just use $\psi$ to denote this mapping. Otherwise, we will specify that we mean the mapping associated with $\md$ by denoting it $\psi_{\md}$.

We abuse notation and denote $(\psi(x), \psi(y)) := \psi(x, y)$, and for a set $S$, use $\psi(S)$ to denote the element-wise application of $\phi$. For $\md = \dfull$, we also denote $\psi(\md) = \md_{\mu_2, \mu_1, \sigma, d}$ to be the distribution with the opposite labeling ground truth. Thus the following claim is immediate:

\begin{claim}\label{claim:loss}
Fix $\md \in \Omega$. For any $W$, we have $\mc{L}_{\md}(f_W) = 1 - \mc{L}_{\psi(\md)}(f_W)$.
\end{claim}

Observe also that $\psi$ is a measure preserving bijection from $\mathcal{D}$ to $\psi(\md)$.

To prove Theorem~\ref{thm:one_sided_impossible_linear} we will use the following lemma, which shows that with high probability, any near-max-margin classifier does well on the ``opposite'' dataset, but has poor test loss on the opposite distribution.

\begin{lemma}\label{lemma:phen_real}
Suppose $\kgen \leq \kappa \leq \kphen$. Let $\mathcal{A}$ be any algorithm which returns a $(1-\epsilon)$-max margin solution. For a dataset $S \sim \mathcal{D}^n = \dfull^n$, consider the classifier $W = \mathcal{A}(S)$. For any constant $\delta > 0$, there exists constants $\epsilon(\delta, \kappa)$ and $c = c(\kappa, \delta)$ such that if $\epsilon \leq \epsilon(\delta, \kappa)$ and $\frac{d}{n} \geq c$, with probability at least $1 - 3e^{-n/c}$ over $S \sim \mathcal{D}^n$, we have 
\begin{enumerate}
    \item  $\mathcal{L}_{\mathcal{D'}}(f_W) \geq 1 - e^{-\frac{\kappa d}{cn}}$.
    \item $\mathcal{L}_{\psi(S)}(f_W) \leq \delta$.
\end{enumerate}
\end{lemma}
\begin{proof}
The first statement follows from Theorem~\ref{thm:xor_gen}, since the probability of classifying a example from $\mathcal{D}'$ correctly is the same as the probability of misclassifying a example from $\mathcal{D}$ (Claim~\ref{claim:loss}).

We now prove the second statement. We expand the margin on the examples in $\psi(S)$. For clarity, we will assume we are expanding on a example from $\psi(x_j)$ where $ j \in \mc{P}_1$, such that by definition of $\psi$, we have $\psi(x_j) = \mu_2 + \xi_j$. The same argument will apply to examples mapped from any other cluster by interchanging the roles of the four vectors $\mu_1, -\mu_1, \mu_2, -\mu_2$ and the two sets $H_+$ and $H_-$ accordingly.
\begin{align}
    \psi(y_j)f_W(\psi(x_j)) &= \mathbb{E}_i[y_j\phi(w_i^T\psi(x_j))]\\
    &= \frac{1}{2}\mathbb{E}_{i \in H_+ = y_j}[\phi(w_i^T\mu_2 + v_i^T\xi_j)] - \frac{1}{2}\mathbb{E}_{i \in H_-}[\phi(w_i^T\mu_2 + v_i^T\xi_j)]
\end{align}

By Lemma~\ref{lemma:phi_delta} (second statement), we have the following:
\begin{align}
    |\mathbb{E}_{i \in H_{+}}[\phi(w_i^T\mu_2 + v_i^T\xi_j)] & - \mathbb{E}_{i \in H_{+}}[\phi(v_i^T\xi_j)]|\\
    &\leq 2\mathbb{E}_{i \in H_{+}}\left[\phi(v_i^T\xi_j)^2 + 1\right]\sqrt{\mathbb{E}_{i \in H_{+}}\left[(w_i^T\mu_2)\right]} + h\left(\mathbb{E}_{i \in H_{+}}[(w_i^T\mu_2)^2]\right)^{\frac{h}{2}}.
\end{align}
Similarly appealing to Lemma~\ref{lemma:phi_delta} (first statement), we have 
\begin{align}
    |\mathbb{E}_{i \in H_{-}}[\phi(w_i^T\mu_2 + v_i^T\xi_j)] & - \mathbb{E}_{i \in H_{-}}[\phi(w_i^T\mu_2)]|\\
    &\leq \sqrt{\mathbb{E}_{i \in H_{-}}\left[4(w_i^T\mu_2)^2 + 2\right]}\sqrt{\mathbb{E}_{i \in H_{-}}\left[(v_i^T\xi_j)^2\right]} + h\left(\mathbb{E}_{i \in H_{-}}[(v_i^T\xi_j)^2]\right)^{\frac{h}{2}}\\
    &\leq \left(2\|U\|+ \sqrt{2}\right)\sqrt{\mathbb{E}_{i \in H_{-}}\left[(v_i^T\xi_j)^2\right]} + h\left(\mathbb{E}_{i \in H_{-}}[(v_i^T\xi_j)^2]\right)^{\frac{h}{2}}\\
\end{align}
Thus 
\begin{align}
    \psi(y_j)f_W(\psi(x_j)) &= \frac{1}{2}\mathbb{E}_{i \in H_+}[\phi(w_i^T\mu_2 + v_i^T\xi_j)] - \frac{1}{2}\mathbb{E}_{i \in H_- = -y_j}[\phi(w_i^T\mu_2 + v_i^T\xi_j)]\\
    &\geq \frac{1}{2}\mathbb{E}_{i \in H_+}[\phi(v_i^T\xi_j)] - \frac{1}{2}\mathbb{E}_{i \in H_-}[\phi(w_i^T\mu_2)] - \frac{1}{2}\mathcal{E}_j,
\end{align}
where 
\begin{align}
    \mathcal{E}_j &:= 2\mathbb{E}_{i \in H_{+}}\left[\phi(v_i^T\xi_j)^2 + 1\right]\sqrt{\mathbb{E}_{i \in H_{+}}\left[(w_i^T\mu_2)^2\right]} + h\left(\mathbb{E}_{i \in H_{+}}[(w_i^T\mu_2)^2]\right)^{\frac{h}{2}}\\
    &\qquad + 4\sqrt{\mathbb{E}_{i \in H_{-}}\left[(v_i^T\xi_j)^2\right]} + h\left(\mathbb{E}_{i \in H_{-}}[(v_i^T\xi_j)^2]\right)^{\frac{h}{2}},\\
\end{align}
where we have plugged in the fact that $\|U\| \leq \|W\| \leq 1$.

By the first and second conclusions of Lemma~\ref{lemma:small_orthogonal}, for $\kgen < \kappa < \kphen$, for at least a $1 - \epsilon_{\ref{lemma:small_orthogonal}}$ a set of examples $T \subset S$ of size at least $(1 - \epsilon'_{\ref{lemma:small_orthogonal}}n)$, we have if $j \in T$, $\frac{1}{2}\mathbb{E}_{i \in H_{+}}\left[(w_i^T\mu_2)^2\right] \leq (\epsilon'_{\ref{lemma:small_orthogonal}})^2$
and $\frac{1}{2}\mathbb{E}_{i \in H_-}\left[(v_i^T\xi_j)^2\right] \leq \frac{1}{\kappa}\cdot \epsilon'_{\ref{lemma:small_orthogonal}}$, where $\epsilon'_{\ref{lemma:small_orthogonal}} = \sqrt{2C_{\ref{lemma:concentration}}\sqrt{\frac{n}{d}} + 2\epsilon}$. Thus if $j \in T$, 
\begin{align}\mc{E}_j &\leq 2\mathbb{E}_{i \in H_{+}}\left[\phi(v_i^T\xi_j)^2 + 1\right]\sqrt{\frac{2\epsilon'_{\ref{lemma:small_orthogonal}}}{\kappa}} + h\left(\sqrt{2}\epsilon'_{\ref{lemma:small_orthogonal}}\right)^{h} + 4\sqrt{\frac{2\epsilon'_{\ref{lemma:small_orthogonal}}}{\kappa}} + h\left(\frac{2\epsilon'_{\ref{lemma:small_orthogonal}}}{\kappa}\right)^{\frac{h}{2}}\\
&\leq 8\mathbb{E}_{i \in H_{+}}\left[\phi(v_i^T\xi_j)^2 + 1\right]\sqrt{\frac{2\epsilon'_{\ref{lemma:small_orthogonal}}}{\kappa}} +4\epsilon'_{\ref{lemma:small_orthogonal}}
\end{align}
for $\epsilon'_{\ref{lemma:small_orthogonal}}$ small enough.

Now by Lemma~\ref{lemma:phen} applied to $S$, if $\epsilon \leq \epsilon(\kappa, \delta/2)$ and $\frac{d}{n} \geq c(\kappa, \delta/2)$, there exists a set $T' \subset S$ size at least $(1 - \frac{\delta}{2})n$ on which the second conclusion of the lemma holds.  Thus for the constant $q = q(\kappa)$ in Lemma~\ref{lemma:phen},
\begin{align}
    \frac{1}{2}\mathbb{E}_{i \in H_+}[\phi(u_i^T\mu_1)] &\leq \frac{(1 - q/4)}{2^h}\gamma,\\
    \frac{1}{2}\mathbb{E}_{i \in H_-}[\phi(u_i^T\mu_2)] &\leq \frac{(1 - q/4)}{2^h}\gamma,
\end{align} and for $j \in T'$,
\begin{align}
    \frac{1}{2}\mathbb{E}_{i \in H_+}[\phi(v_i^Tx_j)] \geq \frac{(1 + q/4)}{2^h}\gamma.
\end{align}
Thus
\begin{align}
    \psi(y_j)f_W(\psi(x_j)) &\geq \frac{1}{2}\mathbb{E}_{i \in H_+}[\phi(v_i^T\xi_j)] - \frac{1}{2}\mathbb{E}_{i \in H_-}[\phi(u_i^T\mu_2)] - \frac{1}{2}\mathcal{E}_j\\
    &\geq \frac{\gamma(1 + q/4)}{2^h} - \frac{\gamma(1 - q/4)}{2^h} - \frac{1}{2}\mathcal{E}_j
\end{align}
which is greater than zero for $\epsilon$ small enough and $c$ large enough in terms of $q$ and $\delta$. (In particular, we will need that $\mc{E}_j \leq \frac{\gamma q}{2^{h + 1}}$ and $\epsilon_{\ref{lemma:small_orthogonal}} \leq \frac{\delta}{2}$, and note that $\gamma$ is at least a constant).

Thus $f_W$ correctly classifies each example $\psi(x_j)$ for $j \in T \cap T'$, which is at least a $1 - \delta$ fraction of the examples in $\psi(S)$.
\end{proof}

We now prove Theorem~\ref{thm:one_sided_impossible_linear} for the XOR problem. We restate the theorem below, and only include the XOR part.

\begin{theorem}[One sided UC Bounds are Vacuous for XOR Problem]
Fix $h \in (1, 2)$, and suppose $\kgenx < \kappa < \kphenx$. For any $\delta > 0$, there exist strictly positive constants $\epsilon = \epsilon(\kappa, \delta)$ and $c = c(\kappa, \delta)$ such that the following holds. Let $\mathcal{A}$ be any algorithm that outputs a $(1-\epsilon)$-max-margin two-layer neural network $f_W$ for any $S \in (\mathbb{R}^d \times \{1,-1\})^n$. Let $\mathcal{H}$ be any concept class that is useful for $\mathcal{A}$ on $\Omega_{\sigma, d}^{h, \on{XOR}}$ (as in Definition~\ref{def:useful}).
Suppose that $\epsilon_{\on{unif}}$ is a uniform convergence bound for the XOR problem $\Omega_{\sigma, d}^{h, \on{XOR}}$: that is, for any $\mathcal{D} \in \Omega_{\sigma, d}^{h, \on{XOR}}$,  $\epsilon_{\on{unif}}$ satisfies $$\Pr_{S \sim \mathcal{D}^n}[ \sup_{h \in \mathcal{H}} \mathcal{L}_{\mathcal{D}}(h) - \mathcal{L}_S(h)  \geq \epsilon_{\on{unif}} ] \leq 1/4.$$
Then if $\frac{d}{n} \geq c$ and $n > c$ we must have $\epsilon_{\on{unif}} \geq 1 - \delta.$
\end{theorem}
\begin{proof}
Let $c = 3c_0$ and $\epsilon = 3\eps_0$ where $c_0$ and $\eps_0$ are the constants from Lemma~\ref{lemma:phen_real} for $\kappa$ and $\delta$.

For any $\md \in \Omega$, let $T_{\md} \subset 2^{(\mathbb{R}^d \times \{-1, 1\})^n}$ be the set of training sets $S$ on which the conclusion of Lemma~\ref{lemma:phen_real} holds for ${\md}$ and $S$. Thus for any ${\md} \in \Omega$, $\Pr_{S \sim {\md}^n}[S \in T_{\md}] \geq 1 - 3e^{-n/c_0}$ for some $c_0 = c(\kappa, \delta)$. Let $H \subset 2^{(\mathbb{R}^d \times \{-1, 1\})^n}$ be the set of training sets $S$ on which $\mathcal{A}(S) \in \mathcal{H}$. Thus for any ${\md} \in \Omega$, $\Pr_{S \sim \md^n}[S \in H] \geq \frac{3}{4}$.

For any $\md \in \Omega$, let $T'_{{\md}}$ be the set on which 
\begin{align}
    \mathcal{L}_{\md}(h) \leq \mathcal{L}_S(h) + \eps_{\on{unif}} \qquad \forall h \in \mathcal{H}.
\end{align}
By assumption, for any ${\md} \in \Omega$, $\Pr_{S \sim {\md}^n}[S \in T'_{\md}] \geq \frac{3}{4}$. By a union bound, for any ${\md} \in \Omega$, for $n \geq c = 3c_0$, with $\psi = \psi_{\md}$,
\begin{align}
    \Pr_{S \sim {\md}^n}[S \in T'_{\md} \land \psi(S) \in T_{\psi({\md})} \land \psi(S) \in H] \geq 1 - \left(1 - \frac{3}{4}\right) - \left(1 - \frac{3}{4} + 3e^{-n}\right) = \frac{1}{2} - 3e^{-n/c_0} > 0.
\end{align}
This is because the distribution of $\psi(S)$ for $S \sim \md^n$ is the same as the distribution of $n$ i.i.d. samples from $\psi({\md})^n$.

Let $S$ be any set for which the three events above hold, ie.,
\begin{align}
  S \in T'_{\md} \land \psi(S) \in T_{\psi({\md})} \land \psi(S) \in H. 
\end{align}
With $f_W = \mathcal{A}(\psi(S))$, by the first conclusion of Lemma~\ref{lemma:phen_real}, we have $\mathcal{L}_{\md}(f_W) \geq 1 - e^{-\frac{1}{c_0
\sigma^2}}$. Further, by the second conclusion of Lemma~\ref{lemma:phen_real}, we know that $\mathcal{L}_{S}(f_W) \leq \delta$. It follows that $\eps_{\on{unif}} \geq 1 -  e^{-\frac{1}{c_0
\sigma^2}}$. Since $c > c_0$, this yields the theorem.
\end{proof}

\begin{lemma}[Margin Lower Bound Lemma]\label{margin_phen}
For any $\kgen < \kappa < \kphen$ and $\epsilon > 0$, there exists some positive constants $q(\kappa) > 0$, $c = c(\kappa, \varepsilon)$ such that if $\frac{d}{n} \geq c$ with probability at least $1 - 3e^{-n/c}$ over $S \sim \mathcal{D}^n = \dfull^n$, there exists a classifier $W$ with $\|W\| = 1$ such that
\begin{enumerate}
    \item $\gamma(f_W, S) \geq (1 - \epsilon)\gamma^*(S)$
    \item $\mathcal{L}_{\psi(\mathcal{D})}(f_W) \geq 1 - e^{-\frac{1}{c\sigma^2}}$.
    \item $\gamma(f_W, \psi(S)) \geq q(\kappa)\gamma^*(S)$
\end{enumerate}
\end{lemma}
\begin{proof}
We work backwards from Opt 5 through Opt 1. Condition on the event in Lemma~\ref{lemma:concentration} holding for $\Xi$ and the event in Lemma~\ref{lemma:cluster_size} holding for $\beta = \frac{1}{c_0}$, for some constant $c_0(\kappa, \epsilon) > 8$ to be chosen later. We will eventually choose $c(\kappa, \eps) \geq c_0^2 + 1$, such that these events hold with probability at least $1 - 3e^{-n/c}$ (see eg. Lemma~\ref{lemma:xor_gen} for the computation).

Given an optimal solution $\ins_5^*$ to Opt 5, we can construct an instance $\ins_1 = \psi_{21}(\psi_{32}(\psi_{43}(\psi_{54}(\ins_5))))$, which is $\eps'$-optimal over all solutions $W'$ with the same norm for $\eps' = \sqrt{1 - (1 - \hat{\eps})(1 - C_{\ref{lemma:concentration}}\sqrt{\frac{n}{d}})^h}$, where $\hat{\eps} = \sqrt{1 - \left(\frac{n_{\on{max}}}{n_{\on{min}}}\right)^{-h}} \leq \sqrt{1 - \left(1 + \frac{16}{c_0}\right)^{-h}}$. This can be seen via Lemmas~\ref{lemma:opt_1_relax_1}, \ref{lemma:relax_1_relax_2}, \ref{lemma:relax_2_relax_3}, \ref{lemma:symmetrize}, which show that at each step of the chain, we do not lose any optimality expect from from Opt 3 to Opt 2 and from Opt 2 to Opt 1.

This will yield the first statement in the theorem for $c_0$ large enough in terms of $\epsilon$ and $\frac{d}{n} \geq c_0$. If we make $\eps'$ small enough (in terms of $\kappa$), then we know from Theorem~\ref{thm:xor_gen} that $\mc{L}_{\md}(f_W) \leq e^{-\frac{1}{c_1\sigma^2}}$ for some $c_1 = c_1(\kappa)$. This yields the second conclusion (as long as $c \geq c_1$), since the probability of classifying an example from $\psi(\md)$ correctly is equal to the probability of classifying an example from $\md$ incorrectly (Claim~\ref{claim:loss}).

We proceed to analyze the properties of $f_W$ to obtain the final conclusion.

Recall from Lemma~\ref{lemma:trivariate_analysis} that Since $\hat{\kappa} \leq \kappa < \kphen$, the optimal solution to Opt 5 has $\phi(b) \leq \frac{1 - q_1}{2^h}\phi(b + c)$ and $\phi(-b) \leq \frac{1 - q_1}{2^h}\phi(-b + d)$ for some constant $q_1 = q_1(\kappa)$. Let $\gamma_j$ be the margin $y_j f_W(x_j)$, and observe that by the symmetry of the backwards mapping $\gamma_j$ is the same for all points $j$. We call this value $\gamma$.

Applying the four mappings above, in the instance $\ins_1$, the variable $W = U + V$ satisfies for all $j \in \mc{P}$ and $i \in H_+$,
\begin{align}
\phi(y_jw_i^T\mu_1) \leq \frac{1 - q_1}{2^h}\phi(y_jw_i^T\mu_1 + w_i^T\xi_j),
\end{align}
and for all $j \in \mc{N}$ and $i \in H_-$,
\begin{align}
\phi(y_jw_i^T\mu_2) \leq \frac{1 - q_1}{2^h}\phi(y_jw_i^T\mu_2 + w_i^T\xi_j),
\end{align}
Further, by definition of the mapping $\psi_{21}$, for $i \in H_+$, we have $w_i^T\mu_2 = 0$ and $w_i^T\xi_j = 0$ for all $j \in \mc{N}$. Similarly, for $i \in H_-$, we have $w_i^T\mu_1 = 0$,  and $w_i^T\xi_j = 0$ for all $j \in \mc{P}$. 

Now we appeal to the fact that by Lemma~\ref{lem:activation}, for any values $s, t$, we have $\phi(s + t) \leq (\phi(s) + \phi(t))2^{h-1}$, and thus (repeating the argument in Equations~\ref{phen_copy1} and \ref{phen_copy2} of Lemma~\ref{lemma:phen}, which we omit the details of here) for all $j \in \mc{P}$ and taking expectation over $i \in H_+$,
\begin{align}
    \mathbb{E}_{i \in H_+}[\phi(v_i^T\xi_j)] &\geq \left(\frac{1 + q_1}{2^h}\right)\mathbb{E}_{i \in H_+}[\phi(w_i^Tx_j)] = \left(\frac{1 + q_1}{2^h}\right)(2\gamma).
\end{align}
Similarly for all $j \in \mc{N}$ and $i \in H_-$,
\begin{align}
    \mathbb{E}_{i \in H_+}[\phi(v_i^T\xi_j)] &\geq \left(\frac{1 + q_1}{2^h}\right)\mathbb{E}_{i \in H_+}[\phi(w_i^Tx_j)] = \left(\frac{1 + q_1}{2^h}\right)(2\gamma).
\end{align}
Finally, by inspecting the mapping in Lemma~\ref{lemma:symmetrize}, and the fact that all of the backwards mapping duplicate solutions to the simpler problems, we have the following symmetry property of $W$:
\begin{align}
    \mathbb{E}_{i \in H_+}[\phi(w_i^T\mu_1)] = \mathbb{E}_{i \in H_+}[\phi(-w_i^T\mu_1)] =     \mathbb{E}_{i \in H_-}[\phi(w_i^T\mu_2)] = \mathbb{E}_{i \in H_-}[\phi(-w_i^T\mu_2)].
\end{align}
We can now examine the margin on the flipped dataset $\psi(S)$. Without loss of generality, consider an example $\psi(x_j)$ where $j \in \mc{P}_1$, such that $\psi(x_j) = \mu_2 + \xi_j$.

\begin{align}
    \psi(y_j)f_W(\psi(x_j)) &= \mathbb{E}_i[y_ja_i\phi(w_i^T\psi(x_j))]\\
    &= \frac{1}{2}\mathbb{E}_{i \in H_+}[\phi(w_i^T\mu_2 + v_i^T\xi_j)] - \frac{1}{2}\mathbb{E}_{i \in H_-}[\phi(w_i^T\mu_2 + v_i^T\xi_j)]\\
    &= \frac{1}{2}\mathbb{E}_{i \in H_+}[\phi(v_i^T\xi_j)] - \frac{1}{2}\mathbb{E}_{i \in H_-}[\phi(w_i^T\mu_2)]\\
    &\geq \left(\frac{1 + q_1}{2^h}\right)(\gamma) - \left(\frac{1 - q_1}{2^h}\right)(\gamma)\\
    &= \frac{q_1\gamma}{2^{h-1}}.
\end{align}
Thus $\psi(y_j)f_W(\psi(x_j)) \geq \frac{q_1\gamma}{2^{h-1}} \geq (1 - \eps')\frac{q_1\gamma^*(S)}{2^{h-1}}$, and the conclusion follows by choosing $q = \frac{q_1}{2^{h}}$ since we have $\eps' \leq \frac{1}{2}$ for $c_0$ large enough.
\end{proof}

\xormargin*

To prove Theorem~\ref{thm:margin_impossible} we use Lemma~\ref{margin_phen}.
\begin{proof}[Proof of Theorem~\ref{thm:margin_impossible}]
Let $c = 3c_0$, where $c_0 = c(\kappa, \eps)$ is the constant from Lemma~\ref{margin_phen}.

For any $\md \in \Omega$, let $T_{\md} \subset 2^{(\mathbb{R}^d \times \{-1, 1\})^n}$ be the set of training sets $S$ on which the conclusion of Lemma~\ref{margin_phen} holds for ${\md}$ and $S$. Thus for any ${\md} \in \Omega$, $\Pr_{S \sim {\md}^n}[S \in T_{\md}] \geq 1 - 3e^{-n/c_0}$ for some constant $c$. Let $H \subset 2^{(\mathbb{R}^d \times \{-1, 1\})^n}$ be the set of training sets $S$ on which all $(1-\eps)$-max-margin two-layer neural networks $f_W$ for $S$ lie in $\mathcal{H}$. Thus for any ${\md} \in \Omega$, $\Pr_{S \sim {\md}^n}[S \in H] \geq \frac{3}{4}$.

For any $\md \in \Omega$, let $T'_{{\md}}$ be the set on which 
\begin{align}
    \mathcal{L}_{\md}(h) \leq \mathcal{L}_S(h) + \frac{G}{\gamma(h, S)^p} \qquad \forall h \in \mathcal{H}.
\end{align}
By assumption, for any ${\md} \in \Omega$, $\Pr_{S \sim {\md}^n}[S \in T'_{\md}] \geq \frac{3}{4}$. 

Now fix any $\md = \md_{\mu, \sigma, d}\in \Omega$. By a union bound, with $\psi = \psi_{\md}$,
\begin{align}
    \Pr_{S \sim {\md}^n}[S \in T'_{\md} \land \psi(S) \in T_{\psi({\md})} \land \psi(S) \in H] \geq 1 - \left(1 - \frac{3}{4}\right) - 3e^{-n/c_0} - \left(1 - \frac{3}{4}\right) = \frac{1}{2} - 3e^{-n/c_0}.
\end{align}
This is because the distribution of $\psi(S)$ with $S \sim \md^n$ is the same as the distribution of $n$ samples from  $\psi({\md})$. 

Let $S$ be any set for which the three events above hold, ie.,
\begin{align}
  S \in T'_{\md} \land \psi(S) \in T_{\psi({\md})} \land \psi(S) \in H. 
\end{align} 
Let $f_W$ be the classifier produced by Lemma~\ref{margin_phen} on input $\psi(S)$ and distribution $\psi({\md})$, such that:
\begin{enumerate}
    \item $\gamma(f_W, \psi(S)) \geq (1 - \epsilon)\gamma^*(\psi(S))$, and thus since $\psi(S) \in H$, we have $f_W \in \mathcal{H}$.
    \item $\mathcal{L}_{\md}(f_W) \geq 1 - e^{-\frac{1}{c_0\sigma^2}}$.
    \item $\gamma(f_W, S) \geq q\gamma^*(\psi(S))$ for some constant $q(\kappa)$.
\end{enumerate}

It follows that for any such $S$, we must have 
\begin{align}
    G \geq \left( 1 - e^{-\frac{1}{c_0\sigma^2}}\right)\gamma(f_W, S)^p \geq \left( 1 - e^{-\frac{1}{c_0\sigma^2}}\right)\gamma^*(\psi(S))^pq^p
\end{align}

Thus for the distribution $\psi({\md})$, with probability at least $\frac{1}{2} - 3e^{-n/c_0}$,  the margin bound yields a generalization guarantee no better than 
\begin{align}
    \left( 1 - e^{-\frac{1}{c_0\sigma^2}}\right)q^p. 
\end{align}
Taking $c = \max(\frac{1}{q^p\left( 1 - e^{-\frac{1}{c_0\sigma^2}}\right)}, c_0, \frac{c_0}{\kappa})$ yields the theorem. Note that $e^{-\frac{1}{c_0\sigma^2}} = e^{-\frac{\kappa d}{c_0n}}$, so for $\frac{d}{n} \geq \frac{c_0}{\kappa}$, $1 - e^{-\frac{1}{c_0\sigma^2}}$ is bounded away from $0$ and thus $c$ only depends on $\kappa$ and $\delta$ (since $c_0$ additionally depends on $c_0$).



\end{proof}

Finally, we prove  Theorem~\ref{thm:xor_no_gen}, which we restate.
\xornogen*

\begin{proof}[Proof of Theorem~\ref{thm:xor_no_gen}]
This follows directly from Lemma~\ref{lemma_no_gen}, since for any $\mathcal{D}_{\mu_1, \mu_2, \sigma, d}$, any classifier $f_W$ with $U = 0$ must have a test loss of exactly $\frac{1}{2}$.
\end{proof}

\subsection{Proof of Technical Lemmas}\label{sec:xor_technical_proofs}

Throughout the following section we assume $\mc{D}_{\mu_1, \mu_2, \sigma, d} \in \Omega$ is fixed, $h \in (1, 2)$, and we use the same notation defined in the notation section at the beginning of Section~\ref{sec:reduction}.

\subsubsection{Proof of Lemma~\ref{lemma:xor_spur_small}}\label{sec:xor_technical_proofs_spurious}

We begin by proving  Lemma~\ref{lemma:xor_spur_small}, for which we will need the following general analysis claim:
\begin{claim}\label{lemma:phi_delta}
For any random variables $a$ and $b$, with $\phi(x) = \max(0, x)^h$, we have
\begin{align}
    \left|\mathbb{E} \left[\phi(a + b) - \phi(b)\right]\right| &\leq \sqrt{\mathbb{E}\left[4a^2 + 2\right]}\sqrt{\mathbb{E}\left[b^2\right]} + h\left(\mathbb{E}[b^2]\right)^{\frac{h}{2}}
\end{align} and 
\begin{align}
     \left|\mathbb{E} \left[\phi(a + b) - \phi(b)\right]\right| &\leq 2\mathbb{E}\left[1 + \phi(a)\right]\sqrt{\mathbb{E}\left[b^2\right]} + h\left(\mathbb{E}[b^2]\right)^{\frac{h}{2}}.
\end{align}
\end{claim}
\begin{proof}
First note that for any $a, b$, we have:
\begin{align}\label{eq:phi1}
    |\phi(a + b) - \phi(a)| \leq \phi'(a + b)|b| \leq \left(\phi'(a) + \phi'(b)\right)|b| \leq \phi'(a)|b| + h|b|^h.
\end{align} and
\begin{align}\label{eq:phi2}
    \phi'(a)= \leq 2|a| + 1,
\end{align}
\begin{align}
     \left|\mathbb{E} \left[\phi(a + b) - \phi(b)\right]\right| &\leq \mathbb{E} \left[|\phi'(a)b|\right] + h\mathbb{E}[|b|^h]\\
    &\leq \sqrt{\mathbb{E} \left[(\phi'(a))^2\right]}\sqrt{\mathbb{E}\left[b^2\right]} + h\mathbb{E}[|b|^h]\\
    &\leq \sqrt{\mathbb{E} \left[(2|a| +  1)^2\right]}\sqrt{\mathbb{E}\left[b^2\right]} + h\mathbb{E}[|b|^h]\\
    &\leq \sqrt{\mathbb{E} \left[4a^2 + 2\right]}\sqrt{\mathbb{E}\left[b^2\right]} + h\mathbb{E}[|b|^h] \\
    &\leq \sqrt{\mathbb{E} \left[4a^2 + 2\right]}\sqrt{\mathbb{E}\left[b^2\right]} + h\left(\mathbb{E}[b^2]\right)^{\frac{h}{2}}.
\end{align}
Here we used Equation~\ref{eq:phi1} in the first inequality, Cauchy-Schwartz in the second, Equation~\ref{eq:phi2} in the third, Jenson's in the fourth, and Jensen's again in the fifth inequality.

If instead of Equation~\ref{eq:phi1}, we can obtain an alternative result.
\begin{align}\label{eq:phi2_alt}
    (\phi'(a))^2 = h^2\max(0, a)^{2h - 2} \leq 4(1 + \phi(a))
\end{align}
This yields
\begin{align}
     \left|\mathbb{E} \left[\phi(a + b) - \phi(b)\right]\right| &\leq \sqrt{\mathbb{E} \left[(\phi'(a))^2\right]}\sqrt{\mathbb{E}\left[b^2\right]} + h\mathbb{E}[|b|^h]\\
    &\leq 2\sqrt{\mathbb{E} \left[1 + \phi(a)\right]}\sqrt{\mathbb{E}\left[b^2\right]} + h\mathbb{E}[|b|^h]\\
    &\leq 2\mathbb{E}\left[1 + \phi(a)\right]\sqrt{\mathbb{E}\left[b^2\right]} + h\left(\mathbb{E}[b^2]\right)^{\frac{h}{2}}.
\end{align}
\end{proof}

We restate Lemma~\ref{lemma:xor_spur_small} for the reader's convenience.
\xorsmall*
\begin{proof}[Proof of Lemma~\ref{lemma:xor_spur_small}]
We can write $x = z + \xi$ for where $z \in \on{Span}(\mu_1, \mu_2)$ and $\xi \perp \mu_1, \mu_2$, such that by Claim~\ref{lemma:phi_delta} we have
\begin{align}
    |f_{W}(x) - f_{U}(x)| &= \left|\mathbb{E}_i \left[\phi(u_i^Tz + v_i^T\xi) - \phi(u_i^Tz)\right]\right| \\
    &\leq \sqrt{\mathbb{E}_i \left[4(u_i^Tz)^2 + 2\right]}\sqrt{\mathbb{E}_i\left[(v_i^T\xi)^2\right]} + h\left(\mathbb{E}_i[(v_i^T\xi)^2]\right)^{\frac{h}{2}}\\
    &\leq \sqrt{\mathbb{E}_i \left[8\|u_i\|^2 + 2\right]}\sqrt{\mathbb{E}_i\left[(v_i^T\xi)^2\right]} + h\left(\mathbb{E}_i[(v_i^T\xi)^2]\right)^{\frac{h}{2}},
\end{align}
where we have plugged $\|z\|_2 \leq \sqrt{2}$.

Now it suffices to get a high probability bound on $\mathbb{E}_i[(v_i^T\xi)^2] = \xi^T\mathbb{E}_i[v_iv_i^T]\xi$ for a random $\xi$. Let $M := \mathbb{E}_i[v_iv_i^T]$. We know by the Hanson-Wright Inequality that for some universal constant $c$,
\begin{align}
    \Pr\left[\xi^T\mathbb{E}_i[v_iv_i^T]\xi \geq \sigma^2\on{Tr}\left(M\right) + t\right] &\leq 2\exp\left(-c\min\left(\frac{t^2}{\|M\|_F^2}, \frac{t}{\|M\|_2} \right)\right) \\
    &\leq 2\exp\left(-c\min\left(\frac{\sigma^2t^2}{\on{Tr}\left(M\right)^2}, \frac{\sigma t}{\on{Tr}\left(M\right)} \right)\right),
\end{align}
where $\|\|_F$ denotes the Frobenius norm, and $\|\|_2$ denotes the spectral norm.
Thus for $t \geq 1$,
\begin{align}
    \Pr\left[\xi^T\mathbb{E}_i[v_iv_i^T]\xi \geq (t + 1)\sigma^2\|V\|^2 \right] &\leq 2\exp\left(-ct\right).
\end{align}
If follows that for any $t \geq 1$, with probability $1 - 2\exp\left(-ct\right)$,
\begin{align}
    |f_{W}(x) - f_{U}(x)| = \left(8\|U\| + 3\right)(t + 1)\sigma^2\|V\|^2 + 2\left((t + 1)\sigma^2\|V\|^2\right)^{\frac{h}{2}}.
\end{align}

\end{proof}

\subsubsection{Proof of Chaining Lemmas}\label{sec:xor_technical_proofs_chain}

\begin{proof}[Proof of Lemma~\ref{lemma:small_orthogonal}]
To prove the lemma, we will begin with a $(1 - \eps)$-solution $W$ to Opt 1. Assuming toward a contradiction that items (1) or (2) in the lemma statement do not hold, we will construct a solution $W''$ for Opt 1 that is more than a $1/(1-\eps)$-factor times better than $W$, contradicting the $(1 - \eps)$-optimality of $W$.
We condition on the event that the conclusion of Lemma~\ref{lemma:concentration} holds for $\Xi$. Given a solution $W$ to Opt 1, construct a solution $W'$ for Opt 1 as follows. First define $c_{ij} := w_i^T\xi_j$.
For $i \in H_+$, let $w'_i = \mu_1\mu_1^Tu_i + v'_i$, where $v'_i$ is the min-norm vector such that $(v'_i)^T\xi_j = c_{ij}$ for all $j \in \mc{P}$, and  $(v'_i)^T\xi_j = 0$ for all $j \in \mc{N}$. For $i \in H_-$, let $w'_i = \mu_2\mu_2^Tu_i + v'_i$, where $v'_i$ is the min-norm vector such that $(v'_i)^T\xi_j = c_{ij}$ for all $j \in \mc{N}$, and  $(v'_i)^T\xi_j = 0$ for all $j \in \mc{P}$. Note that all such $v_i'$ are guaranteed to exist since the conclusion of Lemma~\ref{lemma:concentration} holds.

Let $s_i = \|\mu_1\mu_1^Tu_i\|$ and $t_i = \|\mu_2\mu_2^Tu_i\|$.

Observe that by Lemma~\ref{lemma:concentration}, we have: 
\begin{align}
    \|w'_i\|^2 &\leq s_i^2 + \|v_i'\|^2 \leq s_i^2 + \left(1 + C_{\ref{lemma:concentration}}\sqrt{\frac{n}{d}}\right)\frac{1}{d\sigma^2}\sum_{j \in \mc{P}}{c_{ij}^2} \qquad \forall i \in H_+ \\
    \|w'_i\|^2 &\leq t_i^2 + \|v_i'\|^2 \leq t_i^2 + \left(1 + C_{\ref{lemma:concentration}}\sqrt{\frac{n}{d}}\right)\frac{1}{d\sigma^2}\sum_{j \in \mc{N}}{c_{ij}^2} \qquad \forall i \in H_- \\
    \|w_i\|^2 &\geq s_i^2 + t_i^2 + \|v_i\|^2 \geq s_i^2 + t_i^2 + \left(1 - C_{\ref{lemma:concentration}}\sqrt{\frac{n}{d}}\right)\frac{1}{d\sigma^2}\sum_{j \in \mc{P} \cup N}{c_{ij}^2}.\\
\end{align}
\begin{align}
    \mathbb{E}_i[\|w_i\|^2] &\geq \left(1 - C_{\ref{lemma:concentration}}\sqrt{\frac{n}{d}}\right)D + \frac{1}{1 + C_{\ref{lemma:concentration}}\sqrt{\frac{n}{d}}}\mathbb{E}_i[\|w'_i\|^2],
\end{align}
where \begin{align}
    D :=  \frac{1}{2}\mathbb{E}_{i \in H_+}\left[t_i^2 + \frac{1}{d\sigma^2}\sum_{j \in \mc{N}}{c_{ij}^2}\right] + \frac{1}{2}\mathbb{E}_{i \in H_-}\left[s_i^2 + \frac{1}{d\sigma^2}\sum_{j \in \mc{P}}{c_{ij}^2}\right].
\end{align}
Further observe that:
\begin{align}\label{eq:calc1}
    \phi((w'_i)^Tx_j) &= \phi(w_i^Tx_j) \qquad \forall i : a_i > 0, j \in \mc{P} \\
    \phi((w'_i)^Tx_j) &= 0 \leq \phi(w_i^Tx_j) \qquad \forall i : a_i < 0, j \in \mc{P} \\
    \phi((w'_i)^Tx_j) &= 0 \leq \phi(w_i^Tx_j) \qquad \forall i : a_i > 0, j \in \mc{N} \\
    \phi((w'_i)^Tx_j) &=  \phi(w_i^Tx_j) \qquad \forall i : a_i < 0, j \in \mc{N},\\
\end{align}
thus $W'$ satisfies the constraint that $\mathbb{E}_i a_ i\phi(w_i^T x_j) y_j \geq \gamma$. Indeed, we have by construction that  have for all $j \in \mc{P}$ that
\begin{align}\label{eq:calc2}
 \sum_i a_i \phi(w_i^T x_i) y_j &= \sum_{i \in H_+} \phi(w_i^T x_i) - \sum_{i \in H_-} \phi(w_i^T x_i) \\ &\geq \sum_{i \in H_+}\phi((w'_i)^T x_i) - \sum_{i \in H_-} \phi((w'_i)^Tx_i) \\
 &= \sum_i a_i\phi((w_i')^T x_i) y_j
 \end{align}
 and similarly for all $j \in \mc{N}$.
If $D \geq 2C_{\ref{lemma:concentration}}\sqrt{\frac{n}{d}} + 2\epsilon$, then
\begin{align}
    \mathbb{E}[\|w_i'\|^2] &\leq \mathbb{E}[\|w_i'\|^2] - \left(1 - C_{\ref{lemma:concentration}}\sqrt{\frac{n}{d}}\right)D\\ 
    &\leq 1 - \left(1 - C_{\ref{lemma:concentration}}\sqrt{\frac{n}{d}}\right)\left(2\eps  + 2C_{\ref{lemma:concentration}}\sqrt{\frac{n}{d}}\right)\\
    &\leq 1 - 2\epsilon,
\end{align}
where we have used the global assumptions that $\eps \leq 1/4$ and $C_{\ref{lemma:concentration}} \leq 1/2$. Thus we can scale $W'$ up by a factor of $\frac{\left(\mathbb{E}_i[\|w_i\|^2]\right)^{\frac{1}{2}}}{\left(\mathbb{E}_i[\|w_i'\|^2]\right)^{\frac{1}{2}}}$ to achieve a feasible solution $W''$ that has objective value $\frac{1}{(1 - 2\epsilon)^{\frac{h}{2}}} \geq \frac{1}{(1 - \epsilon)}$ times better than the solution given by $W$. This would contradict the $(1 - \epsilon)$-optimality of $W$, proving the first conclusion of the lemma.

For the second part, suppose for greater than a $\sqrt{2C_{\ref{lemma:concentration}}\sqrt{\frac{n}{d}} + 2\epsilon}$ fraction of data points we have $\frac{1}{2}\mathbb{E}_{i \in H_-}\left[(v_i^T\xi_j)^2\right] \geq \frac{1}{\kappa}\cdot \sqrt{2C_{\ref{lemma:concentration}}\sqrt{\frac{n}{d}} + 2\epsilon}$ (if $j \in \mc{P}$) or $\frac{1}{2}\mathbb{E}_{i \in H_+}\left[(v_i^T\xi_j)^2\right] \geq \frac{1}{\kappa}\cdot \sqrt{2C_{\ref{lemma:concentration}}\sqrt{\frac{n}{d}} + 2\epsilon}$ (if $j \in \mc{N}$) . This would imply that $D \geq \left(\frac{\sqrt{2C_{\ref{lemma:concentration}}\sqrt{\frac{n}{d}} + 2\epsilon}}{\kappa}\right)\frac{1}{d\sigma^2}\left(n\sqrt{2C_{\ref{lemma:concentration}}\sqrt{\frac{n}{d}} + 2\epsilon}\right)  = 2C_{\ref{lemma:concentration}}\sqrt{\frac{n}{d}} + 2\epsilon$, which as we saw above contradicts the $(1-\eps)$-optimality of $W$.




\end{proof}

\begin{proof}[Proof of Lemma~\ref{mappings}]


By homogeneity, there exists some value $C_B$ such that the optimum of any instance of Opt B with parameter $P_B$ equals $C_B P_B^{q}$. Thus by the properties of $\psi_{BA}$, given an optimal instance $\ins_B^* \in D_B$ with parameter $P_B$, we can construct and instance of Opt A with parameter at most $(1 + \delta)P_B$ and optimum at least $C_B P_B^{q}$. Thus for some value $C_A \geq C_B(1 + \delta)^{-q}$, the optimum of any instance of Opt A with parameter $P_A$ equals $C_A P_A^q$.

Suppose $\ins_A$ with parameter $P_A$ is $(1 - \eps)$-optimal and $\ins_B^{(1)}, \cdots , \ins_B^{(k)} := \psi_{AB}(\ins_A)$. Let $\gamma$ be the objective value of $\ins_A$. Define $P_B^{(1)} \cdots P_B^{(p)}$ to be the parameters $P_B$ of the $k$ instances respectively, and let $\gamma^{(p)}$ be their objective values. For $p \in [k]$, let $s(p)$ be the optimality of each $\ins_B^{(p)}$ times $\frac{\gamma}{\gamma^{(p)}}$. Then: $s(p) = \frac{\gamma^{(p)}}{C_B(P_B^{(p)})^{q}}\frac{\gamma}{\gamma^{(p)}}$, so 
\begin{align}
    s(p)C_B(P_B^{(p)})^{q} &\geq (1 - \eps)C_A(P_A)^q \geq (1 - \eps)(1 + \delta)^{-q}C_B(P_A)^q \quad \forall p\\
    \mathbb{E}_{p \in [k]}[P_B^{(p)}] &\leq (1 + \delta)P_A
\end{align}
Here the first inequality in the first line follows from the fact that the objective value achieved by $\ins_B^{(p)}$ is at least as large as the  objective value of $\ins_A$, which by assumption is at lease $(1 - \eps)$-optimal.

We now proceed by contradiction: Suppose for some set $S \subset [k]$ of size at least $k\eps'$, we have $s(p) \leq 1 - \eps'$. Then
\begin{align}
    \mathbb{E}_{p \in [k]}[P_B^{(p)}] &\geq \frac{1}{k}\sum_{p \in S}P_B^{(p)} + \frac{1}{k}\sum_{p \notin S}P_B^{(p)} \\
    &\geq \eps'\left((1 - \eps')^{-\frac{1}{q}}(1 - \epsilon)^\frac{1}{q}(1 + \delta)^{-1}P_A\right) + (1 - \eps')\left((1 - \epsilon)^\frac{1}{q}(1 + \delta)^{-1}P_A\right)\\
    &= P_A(1 - \eps)^{\frac{1}{q}}(1 + \delta)^{-1}\left(\eps'(1 - \eps')^{-\frac{1}{q}} + (1 - \eps')\right)
\end{align}
Thus if 
\begin{align}
    \left(\eps'(1 - \eps')^{-\frac{1}{q}} + (1 - \eps')\right) > (1 + \delta)^2(1 - \eps)^{-\frac{1}{q}}, 
\end{align}
we will have a contradiction, since the equation above will be strictly greater than $(1 + \delta)P_A$.

Choosing $\eps' = \sqrt{1 - (1 - \eps)(1 + \delta)^{-2q}}$, this produces the desired contradiction. Indeed, on can check that for all $\eps' \in (0, 1)$, we have
\begin{align}
    \eps'(1 - \eps')^{-\frac{1}{q}} + (1 - \eps') >  (1 - (\eps')^2)^{-\frac{1}{q}},
\end{align}
yielding the desired contradiction. Thus for at least a $1 - \eps'$ fraction of $p \in [k]$, we have $$\frac{\gamma}{\gamma^{(p)}} \times (\text{optimality of } \ins_B^{(p)}) \geq 1 - \eps',$$ which implies that each of these two terms are greater that $1 - \eps'$.

This proves the first conclusion.

To achieve the second conclusion, consider the mapping $\psi_{AB}': D_B \rightarrow D_A$ which maps $\ins_A$ to the instance of $\psi_{AB}(\ins_A)$ which has the smallest parameter $P_B$. Necessarily, this value at most $(1 + \delta)P_A$, since the average value of $P_B^{(p)}$ is at most $(1 + \delta)P_A$. Thus the pair of mappings $\psi_{BA}$ and $\psi_{AB}'$ and $\psi_B$ satisfy the conditions of the lemma, which we now apply with $k = 1$, and the roles of $A$ and $B$ reversed. The second conclusion follows.
\end{proof}


\begin{proof}[Proof of Lemma~\ref{lemma:opt_1_relax_1}]
Recall that we have conditioned on the event that for any $c \in \mathbb{R}^n$, the min-norm vector $v$ satisfying $\Xi^Tv = c$ has $\|v\|_2^2 \in \frac{\|c\|_2^2}{\sigma^2d}\left[\frac{1}{1 + C_{\ref{lemma:concentration}}\sqrt{\frac{n}{d}}}, \frac{1}{1 - C_{\ref{lemma:concentration}}\sqrt{\frac{n}{d}}}\right].$

Observe that the mappings $\psi_{12}(\ins_1)$ produces a feasible instance, since for all $i \in H_+$, 
\begin{align}
   s_i^2 + \frac{1}{d\sigma^2}\sum_{j \in \mc{P}}{c_{ij}^2} & \leq s_i^2 + \frac{1}{d\sigma^2}\sum_{j \in \mc{P} \cup \mc{N}}{c_{ij}^2}\\
   &\leq \|u_i\|^2 + \frac{1}{d\sigma^2}\left(1 + C_{\ref{lemma:concentration}}\sqrt{\frac{n}{d}}\right)\sigma^2d\|v_i\|^2 \leq \left(1 + C_{\ref{lemma:concentration}}\sqrt{\frac{n}{d}}\right)\|w_i\|^2.
\end{align}
A similar statement holds for $i \in H_-$, summing over $j \in \mc{N}$. Further, the objective value of $\psi_{12}(\ins_1)$ is at least the objective value of $\ins_1$.

The mapping $\psi_{21}$ always maintains the exact same objective value, and is feasible because for $i \in H_+$, $\|v_i\|^2 \leq \frac{1}{1 - C_{\ref{lemma:concentration}}\sqrt{\frac{n}{d}}}\frac{1}{d\sigma^2}\sum_{j \in \mc{P}}c_{ij}^2$, and a similar statement holds for $i \in H_-$.

Thus applying Lemma~\ref{mappings} twice (with Opt A = Opt 1 and Opt B = Opt 2, and then in reverse, and with $q = \frac{h}{2}$ and $1 + \delta = \frac{1}{1- C_{\ref{lemma:concentration}}\sqrt{\frac{n}{d}}}$) yields the result.



\end{proof}

\maptwo*

\begin{proof}[Proof of Lemma~\ref{lemma:relax_1_relax_2}]
It is easy to check by the definition of the mappings that if an instance $\ins_2 \in D_2$ is feasible, then so is $\psi_{23}(\ins_2)$. Likewise, if instance $\ins_3 \in D_3$ is feasible, then so is $\psi_{32}(\ins_3)$. Indeed, in $\psi_{32}(\ins_3)$, we have 
\begin{align}
    \frac{1}{2}\mathbb{E}_{i \in H_+}\left[s_i^2 + \frac{1}{d\sigma^2}\sum_{j \in \mathcal{P}}{(c_{ij}^{(2)})^2}\right] =     \frac{1}{2}\mathbb{E}_{i \in H}\left[b_i^2 + \frac{n_{\on{max}}}{n_{\on{min}}}\frac{1}{d\sigma^2}\sum_{j \in S_1 \cup S_{-1}}{(c_{ij}^{(3)})^2}\right] \leq \frac{n_{\on{max}}}{n_{\on{min}}}\frac{P_3}{2},
\end{align}
where we have supersripted the variables $c_{ij}$ in $\psi_{32}(\ins_3)$ by $(2)$, and those in $\ins_3$ by $(3)$. A similar statement holds for the sum over $\mc{N}$, such that 
\begin{align}
    \frac{1}{2}\mathbb{E}_{i \in H_+}\left[s_i^2 + \frac{1}{d\sigma^2}\sum_{j \in \mathcal{P}}{(c_{ij}^{(2)})^2}\right] + \frac{1}{2}\mathbb{E}_{i \in H_-}\left[t_i^2 + \frac{1}{d\sigma^2}\sum_{j \in \mathcal{N}}{(c_{ij}^{(2)})^2}\right] \leq \frac{n_{\on{max}}}{n_{\on{min}}}P_3 = P_2.
\end{align}

It is easy to check also that the objective value of $\psi_{23}(\ins_2)$ is at least that of $\ins_2$, and likewise the objective value of $\psi_{32}(\ins_3)$ is at least that of $\ins_3$.

We can now apply Lemma~\ref{mappings} with Opt A = Opt 2 and Opt B = Opt 3, $q = \frac{h}{2}$, $1 + \delta = \frac{n_{\on{max}}}{n_{\on{min}}}$, and $k = 2$. This yields the result.

\end{proof}

\begin{proof}[Proof of Lemma~\ref{lemma:relax_2_relax_3}]
The proof is similar to the last lemma. It is straightforward to check the conditions of Lemma~\ref{mappings} Opt A = Opt 3, Opt B = Opt 4, the mappings $\psi_{34}$ and $\psi_{43}$, $k = \frac{n}{4}$, $q = \frac{h}{2}$, and $\delta = 0$. The conclusion follows from Lemma~\ref{mappings}.

\end{proof}


\begin{proof}[Proof of Lemma~\ref{lemma:symmetrize}]
We will eventually appeal to Lemma~\ref{mappings}. First observe that the mapping $\psi_{54}$ yields an program in $D_4$ with the exact same objective value and parameter. We will construct an alternative mapping $\psi_{45}': D_4 \rightarrow D_5$ that preserves the parameter and maintains or increases the objective. We use $\psi_{45}$.
Let $P_4$ and $\gamma$ be the parameter and objective of $\ins_4$. First identify the instance $\ins_5^{(i)}$ of $\psi_{45}(\ins_4)$ which achieves the highest ratio between objective value, which we denote $\gamma^{(i)}$, and $P_5^{\frac{2}{h}}$. By the positivity of the of the $\gamma^{(i)}$ and $P_5^{(i)}$ and Jenson's inequality, for at least one instance $i$, we have 
\begin{align}
    \frac{\gamma^{(i)}}{(P_5^{(i)})^{\frac{h}{2}}} \geq \frac{\mathbb{E}[\gamma^{(i)}]}{\mathbb{E}[(P_5^{(i)})^{\frac{h}{2}}]} \geq \frac{\mathbb{E}[\gamma^{(i)}]}{(\mathbb{E}[P_5^{(i)}])^{\frac{h}{2}}} \geq \frac{\gamma}{P_4^{\frac{h}{2}}}.
\end{align}
Then scaling each variable in this instance by a factor of $\sqrt{\frac{P_4}{P_5}}$ to produce an instance feasible instance of Opt 5 with parameter $P_4$ and objective value $\gamma.$

This suffices to apply Lemma~\ref{mappings} with the mappings $\psi_{54}$ and $\psi_{45}'$ and $\delta = 0$. The second conclusion follows.

Now we prove the first conclusion. Let $C_4$ and $C_5$ be such that optimal value of Opt 4 equals $C_4P^{\frac{h}{2}}$ and the optimal value of Opt 5 equals $C_5P^{\frac{h}{2}}$. This holds by the homogeneity of the programs. The argument of Lemma~\ref{mappings} in the paragraph beginning ``We now proceed'', applied using the mappings $\psi_{54}$ and $\psi_{45}'$ shows that $C_4 = C_5$.

Observe that 
\begin{align}
    \mathbb{E}[\gamma^{(i)}] &\leq \mathbb{E}[C_5(P_5^{(i)})^{\frac{h}{2}}]\\
    &\leq \left(\mathbb{E}[C_5(P_5^{(i)})]\right)^{\frac{h}{2}}\\
    &= C_5P_4^{\frac{h}{2}}\\
    &= C_4P_4^{\frac{h}{2}} \leq \frac{\gamma}{1 - \eps}.
\end{align}
Here the first line follows from homogeneity of Opt 5, the second follows from Jenson's since $h < 2$, the third line from observing that the mapping $\psi_{45}$ produces instances with an average parameter equal to $P_4$, and the fourth from the fact that $C_4 = C_5$ and that $\ins_4$ is $(1 - \eps)$-optimal.

This proves the first conclusion of the lemma.
\end{proof}

\subsubsection{Proof of Lemmas analyzing Opt 5}\label{sec:xor_technical_proofs_tri}
We now prove the two lemmas analyze the trivariate program, Opt 5.
\begin{proof}[Proof of Lemma~\ref{tri_opt}]
We consider two classes of feasible solutions. In the first, $S_1$, we impose the constraint that $-b + d > 0$. In the second, $S_2$,  we have $-b + d \leq 0$. 

For solutions in $S_1$, it is easy to check that for any $(b, c, d)$, we can increase the objective value via the solution $(b', c', d')$, where $d' = 0$, and $c' = \sqrt{\frac{1 - b^2}{k}} > c$.

The the optimum in $S_1$ is achieved by setting $d = 0$. It is then easy to check via the KKT conditions of the resulting convex program that the optimum in this set chooses $b$ and $c$ as in the claim.

We now consider the second set, $S_2$. Our goal will be to re-parameterize the objective in terms of a single variable $\alpha := \frac{c - d}{c + d}$, and then analyze the one-dimensional optimization landscape as a function of $\alpha$. 
Recall that $\gamma_{0} = \max_{c, d: kc^2 + kd^2 \leq 1}\left(\phi(b + c) + \phi(-b + d)\right)$. Further define \begin{align}
    \gamma_{bd} := \max_{b, c, d: 0 < b = d, kc^2 + kd^2 \leq 1}\left(\phi(b + c) + \phi(-b + d)\right)
\end{align}
We proceed in a series of claims.

The first claim reduces this 3 variable program to a 2 variable program. 

\begin{claim}
\begin{align}
    &\max_{b, c, d: 0 \leq b \leq d, b^2 + k(c^2 + d^2) \leq 1} \phi(b + c) + \phi(-b + d) \\
    &\qquad \leq \max\left(\gamma_0, \gamma_{bd}, \max_{c, d: k(c - d) \leq d \leq c, k^2(c - d)^2 + k(c^2 + d^2) \leq 1} \phi(k(c - d) + c) + \phi(-k(c - d) + d)\right)\\
\end{align}
\end{claim}
\begin{proof}
This claim reduces to showing that any locally optimal solution in $S_2$ that is not at one of the boundaries $b = 0$ or $b = d$ must satisfy $b = k (c - d)$. We proceed by contradiction. Suppose there was a feasible solution in $S_2$ with $0 < b < d$ which didn't satisfy $b = k (c - d)$. Then we can construct a new solution $b' = b + \Delta$, $c' = c - \Delta$, $d' = d + \Delta$. Then the objective value doesn't change ($\phi(b' + c') + \phi(-b' + d') = \phi(b + c) + \phi(-b + d)$), but for small enough $\Delta$ with the correct sign ($\on{sign}(-b + k(c - d))$) the constraint value decrease, since 
\begin{align}
    (b')^2 + k\left((c')^2 + (d')^2\right) - b^2 - k(c^2 + d^2) = 2\Delta\left(b - k(c - d)\right) + \Theta(\Delta^2) < 0.
\end{align}
Thus is we make this change and then scale up the solution such that the constraint is satisfied with equality, we will have increases the objective. Further, since $b$ is bounded away from $0$ and $d$, for $\Delta$ small enough, we will still have a point in $S_2$. Note, also introduce a constraint that $c \geq d$, since by the convexity of $\phi$, we can switch the values of $c$ and $d$ and increase the objective if $c < d$.
\end{proof}

The next claim reduces the two-variable program to a single variable optimization problem. 

\begin{claim}
\begin{align}
    &\max_{c, d: k(c - d) \leq d \leq c, k^2(c - d)^2 + k(c^2 + d^2) \leq 1} \phi(k(c - d) + c) + \phi(-k(c - d) + d)\\
    &\qquad \leq \max\left(\gamma_0, \gamma_{bd}, \max_{0 \leq \alpha \leq \frac{1}{2k + 1}; f(\alpha) = 0} \left(\frac{1}{\left(k^2 + \frac{k}{2}\right)\alpha^2 + \frac{k}{2}}\right)^{\frac{h}{2}}\left(\phi\left(\left(k + \frac{1}{2}\right)\alpha + \frac{1}{2}\right) + \phi\left(-\left(k + \frac{1}{2}\right)\alpha + \frac{1}{2}\right)\right)\right),
\end{align}
where
\begin{equation}
        f(\alpha) := \left(1 - \alpha\right)\phi'\left(\left(2k + 1\right)\alpha + 1\right) -  \left(1 + \alpha\right)\phi'\left(-\left(2k + 1\right)\alpha + 1\right) = 0.
\end{equation}
\end{claim}
\begin{proof}
First we reparameterize $A = c - d$ and $B = c + d$, such that we can upper bound by the optimum of the following program:
\begin{align}
    \max_{A, B} \qquad & \phi\left(\left(k + \frac{1}{2}\right)A + \frac{1}{2}B\right) + \phi\left(-\left(k + \frac{1}{2}\right)A + \frac{1}{2}B\right)\\
    & s.t. \qquad \left(k^2 + \frac{k}{2}\right)A^2 + \frac{k}{2}B^2 \leq 1 \\
    & \: \qquad 0 \leq A \leq \frac{B}{2k + 1}
\end{align}
Now by the KKT conditions, for any stationary point bounded away from the boundary of $A = 0$ or $A = \frac{B}{2k + 1}$, we must have $\frac{\frac{\partial g}{\partial A}}{\frac{\partial g}{\partial B}} = \frac{\frac{\partial f}{\partial A}}{\frac{\partial f}{\partial B}}$, where $f$ and $g$ represent the objective and the constraint respectively. Thus at these stationary points, we have
\begin{align}
    \frac{(2k^2 + k)A}{kB} = \frac{k + \frac{1}{2}}{\frac{1}{2}}\frac{\phi'\left(\left(k + \frac{1}{2}\right)A + \frac{1}{2}B\right) - \phi'\left(-\left(k + \frac{1}{2}\right)A + \frac{1}{2}B\right)}{\phi'\left(\left(k + \frac{1}{2}\right)A + \frac{1}{2}B\right) + \phi'\left(-\left(k + \frac{1}{2}\right)A + \frac{1}{2}B\right)},
\end{align}
or equivalently, setting $\alpha := \frac{A}{B}$,
\begin{align}
    \frac{1 + \alpha}{1 - \alpha} = \frac{\phi'\left(\left(k + \frac{1}{2}\right)A + \frac{1}{2}B\right)}{\phi'\left(-\left(k + \frac{1}{2}\right)A + \frac{1}{2}B\right)}.
\end{align}
This manipulation and reparameterization in terms of $\alpha$ is useful for analysis because it allows us to leverage the homogeneity of $\phi$ without explicitly computing the KKT solution. Indeed, by homogeneity (and plugging in $\alpha = \frac{A}{B}$), we have at stationary point,
\begin{align}
    \frac{1 + \alpha}{1 - \alpha} = \frac{\phi'\left(\left(2k + 1\right)\alpha + 1\right)}{\phi'\left(-\left(2k + 1\right)\alpha + 1\right)},
\end{align}
or 
\begin{align}
    f(\alpha) := \left(1 - \alpha\right)\phi'\left(\left(2k + 1\right)\alpha + 1\right) -  \left(1 + \alpha\right)\phi'\left(-\left(2k + 1\right)\alpha + 1\right) = 0
\end{align}
Now we check the boundary points. When $A = 0$, this corresponds to the point where $c = d$ and $b = 0$, which yields the objective value $\gamma_0$. When $A = \frac{B}{2k + 1}$, this corresponds to the point when $b = d$, and thus yields the objective value $\gamma_{bd}$.
\end{proof}

In the next claim, we will show the single variable optimization program in terms of $\alpha$ achieves its maximum at the boundaries.
\begin{claim}
\begin{align}
     \max_{0 \leq \alpha \leq \frac{1}{2k + 1}; f(\alpha) = 0} \left(\frac{1}{\left(k^2 + \frac{k}{2}\right)\alpha^2 + \frac{k}{2}}\right)^{\frac{h}{2}}\left(\phi\left(\left(k + \frac{1}{2}\right)\alpha + \frac{1}{2}\right) + \phi\left(-\left(k + \frac{1}{2}\right)\alpha + \frac{1}{2}\right)\right) \leq \max\left(\gamma_0, \gamma_{bd}\right).
\end{align}
\end{claim}
\begin{proof}
First we show that $f(\alpha) = 0$ has at most one strictly positive solution. To show this, since we know $f(0) = 0$, it suffices to check that the second derivative of $f(\alpha)$ is always positive for $0 < \alpha \leq 1$. Indeed, the second derivative evaluates to 
\begin{equation}
     (2 - h)(h-1)t^2\left(\frac{1 + \alpha}{(1 - tx)^{3 - h}} - \frac{1 - \alpha}{(1 + tx)^{3 - h}}\right) + 2(h-1)t\left(\frac{1}{(1 - tx)^{2 - h}} - \frac{1}{(1 + tx)^{2 - h}}\right).
\end{equation}
where $t := \left(2k + 1\right)$. Since $h \in (1, 2)$ and $t \geq 0$, this expression is positive for $\alpha > 0$ since $\frac{1 + \alpha}{(1 - tx)^{3 - h}} > \frac{1 - \alpha}{(1 + tx)^{3 - h}}$ and $\frac{1}{(1 - tx)^{2 - h}} > \frac{1}{(1 + tx)^{2 - h}}$.

Now, we will show the derivative of the objective, which we will call $g(\alpha)$, is positive at the boundary point $\alpha = \frac{1}{2k + 1}$. At this value, the term inside the second $\phi$ evaluates to $0$, and thus so does its derivative. The remaining part of the objective evaluates to 
\begin{align}
\left(\frac{1}{\left(k^2 + \frac{k}{2}\right)\alpha^2 + \frac{k}{2}}\right)^{\frac{h}{2}}\phi\left(\left(k + \frac{1}{2}\right)\alpha + \frac{1}{2}\right) = \frac{1}{k^{\frac{h}{2}}}\left(\left(\left(k + \frac{1}{2}\right)\alpha + \frac{1}{2}\right)^2\left(\left(k + \frac{1}{2}\right)\alpha^2 + \frac{1}{2}\right)^{-1}\right)^{\frac{h}{2}},
\end{align}
so it suffices to check that
\begin{align}
    \left(\left(k + \frac{1}{2}\right)\alpha + \frac{1}{2}\right)^2\left(\left(k + \frac{1}{2}\right)\alpha^2 + \frac{1}{2}\right)^{-1}
\end{align}
is increasing as a function of $\alpha$. Indeed we can take the derivative to confirm this is the case for $\alpha \leq 1$.

Now since $g(\alpha)$ is increasing at the upper boundary $\alpha = \frac{1}{2k + 1}$ and has at most one stationary point between $0$ and the upper boundary, we conclude that this stationary point cannot be a maximum. Thus the maximum must be obtained at the boundary. Again the boundary points correspond to when $c = d$ and $b = 0$, yielding $\gamma_0$ and when $b = d$, yielding $\gamma_{bd}$.
\end{proof}

These three claims have shown that the maximum inside $S_2$ is obtained at one of the boundaries where $b = 0$ or $b = d$. It is easy to check that any solution when $b = d$ is suboptimal if $d > 0$, since we can decrease $d$ and increase $c$ by a small amount which will improve the objective. Now if $\hat{\kappa} > \kgen$, then by definition, $\gamma_0 < \gamma_*$, and thus the optimal solution with $b \geq 0$ is given by choosing $c$ and $d$ as in the lemma.

If $\hat{\kappa} < \kgen$, the greater solution of $\gamma_0$ and $\gamma_*$ is given by $\gamma_0$, and thus we choose $c$ and $d$ to be equal as in the lemma.
\end{proof}

\begin{proof}[Proof of Lemma~\ref{lemma:trivariate_analysis}]
Because of the homogeneity of each constraint, it suffices to prove the result for $B_4 = 1$. Consider some $\epsilon$-optimal solution $(b, c, d)$. Without loss of generality, by the symmetry of the problem and the conclusion, we can assume $b$ is non-negative. 

Observe that for $\epsilon = \epsilon(\hat{\kappa})$ small enough, by the continuity of the objective, any $(1 - \epsilon)$-optimal solution must be arbitrarily close to the solution given in Lemma~\ref{tri_opt}, which we name $(b^*, c^*, d^*)$.

This means that for $\hat{\kappa} > \kgen$, we must have $b$ arbitrarily close to $\sqrt{\frac{\hat{\kappa}}{4 + \hat{\kappa}}}$,  $c$ arbitrarily close to $\sqrt{\frac{16}{\hat{\kappa}(4 + \hat{\kappa})}}$,and $d$ arbitrarily close to $0$. Thus the first conclusion follows by the fact that 
\begin{align}
    \lim_{\epsilon \rightarrow 0} \frac{\phi(b)}{\phi(b + c)} = \frac{\phi(b^*)}{\phi(b^* + c^*)} = \left(\frac{1}{1 + \frac{4}{\hat{\kappa}}}\right)^h.
\end{align}
The second one follows from the fact that in a neighborhood of $b^*$ and $d^*$, both sides are $0$.

If additionally $\hat{\kappa} < \kphen$, then by definition of $\kphen$, we have that $b^* < c^*$. So for small enough $\epsilon$, $b < c$. Thus $\phi(b) = \frac{1}{2^h} \phi(2b) <\frac{1}{2^h}\phi(b + c)$, which yields the third conclusion. Again the fourth conclusion follows from the fact that in a neighborhood of $b^*$ and $d^*$, both sides are $0$.

The last conclusion (if $\hat{\kappa} < \kgen$) follows immediately from the Lemma~\ref{tri_opt}.
\end{proof}







\begin{proof}[Proof of Lemma~\ref{lemma:relax_3_relax_4}]
For each $i \in H$, define an instance of Opt 5 by putting $b = b_i$, $c = c_i$, $d = d_i$, and $P_5 = P_5^{(i)} := b_i^2 + \kappa (c_i^2 + d_i^2)$. Let $\gamma_i$ denote the objective value of this instance of Opt 5, that is, $\phi(b_i + c_i) +  \phi(-b_i + d_i)$.

By the homogeneity of Opt 5, the optimum of Opt 5 equals $C_4P_5^{\frac{h}{2}}$ for some value $C_4$. Further, Claim\ref{lemma:symmetrize} guarantees that the optimum of Opt 4 is at least $\frac{1}{2}C_4(P_4)^{\frac{h}{2}}$, because it is possible to construct a solution to Opt 4 from an optimal solution $(b^*, c^*, d^*)$ of Opt 5 in the following way:
For half the $i \in H$, take $(b_i, c_i, d_i) = (b^*, c^*, d^*)$. For the other half, take $(b_i, c_i, d_i) = (-b^*, d^*, c^*)$. Then the objective value is exactly half of the optimum of the optimum of Opt 5 with $P_5 = P_4$.

For $i \in H$, let $s(i)$ be the optimality of the respective instance of Opt 5, that is, $\frac{\gamma_i}{C_4\left(P_5^{(i)}\right)^{\frac{h}{2}}}.$ If the solution to Opt 4 is $(1 - \epsilon)$-suboptimal, then we have
\begin{align}
    \mathbb{E}_i s(i)C_4\left(P_5^{(i)}\right)^{\frac{h}{2}} &\geq \left(1 - \epsilon\right)C_4(P_4)^{\frac{h}{2}}; \\
    \mathbb{E} P_5^{(i)} \leq P_4.
\end{align}

Plugging the second equation into the first, and applying Jensen's inequality to the concavity of the function $x \rightarrow x^{h/2}$, we obtain
\begin{align}\label{eq:jenson}
    \mathbb{E}_i s(i)\left(P_5^{(i)}\right)^{\frac{h}{2}} &\geq \left(1 - \epsilon\right)(\mathbb{E} P_5^{(i)} )^{\frac{h}{2}} \geq \left(1 - \epsilon\right)\mathbb{E} (P_5^{(i)} )^{\frac{h}{2}}.
\end{align}

Let $\alpha(i) := \left(P_5^{(i)}\right)^{\frac{h}{2}}$. In the remainder of the lemma, we use the $\alpha$ and $s$ to denote random variables over the randomness of $i$, and all expectation are over $i$ uniformly from $H$.

For any $\delta$, we have,
\begin{align}\label{eq:alpha_s}
(1 - \epsilon)\mathbb{E}[\alpha] &\leq \mathbb{E}[\alpha s]\\
    &=\mathbb{E}[\alpha s \mathbbm{1}(s \geq 1 - \delta)] + \mathbb{E}[\alpha s \mathbbm{1}(s < 1 - \delta)]\\
    &\leq \mathbb{E}[\alpha \mathbbm{1}(s \geq 1 - \delta)] + (1 - \delta)\mathbb{E}[\alpha \mathbbm{1}(s < 1 - \delta)]\\
    &= \mathbb{E}[\alpha \mathbbm{1}(s \geq 1 - \delta)] + (1 - \delta)\left(\mathbb{E}[\alpha] - \mathbb{E}[\alpha \mathbbm{1}(s \geq 1 - \delta)]\right)\\
    &= \delta \mathbb{E}[\alpha \mathbbm{1}(s \geq 1 - \delta)] + (1 - \delta)\mathbb{E}[\alpha],
\end{align}
so
\begin{align}\label{eq:a_s_2}
    \mathbb{E}[\alpha \mathbbm{1}(s \geq 1 - \delta)] \geq  \left(1 - \frac{\epsilon}{\delta}\right)\mathbb{E}[\alpha],
\end{align}
and hence,
\begin{align}
    \mathbb{E}[s\alpha \mathbbm{1}(s \geq 1 - \delta)] \geq  \left(1 - \delta\right)\left(1 - \frac{\epsilon}{\delta}\right)\mathbb{E}[\alpha].
\end{align}
Let $\epsilon_{\ref{lemma:trivariate_analysis}}$ and $\eta$ be the constants $\epsilon$ and $\eta$ from Lemma~\ref{lemma:trivariate_analysis}. By Lemma~\ref{lemma:trivariate_analysis}, for any $i$ with $s(i) \geq 1 - \epsilon_{\ref{lemma:trivariate_analysis}}$, we have $\phi(b_i) \geq \eta\phi(b_i + c_i)$. Thus
\begin{align}\label{eq:conclusion}
    \mathbb{E}_{i \in H}[\phi(b_i)] &\geq \mathbb{E}_{i \in H}[\phi(b_i)\mathbbm{1}(s(i) \geq \epsilon_{\ref{lemma:trivariate_analysis}})] \\
    &\geq \eta\mathbb{E}_{i \in H}[\phi(b_i + c_i)\mathbbm{1}(s(i) \geq 1 -  \epsilon_{\ref{lemma:trivariate_analysis}})]\\
    &=  \eta\left(\mathbb{E}_{i \in H}[\phi(b_i + c_i)] - \mathbb{E}_{i \in H}[\phi(b_i + c_i)\mathbbm{1}(s(i) < 1 -  \epsilon_{\ref{lemma:trivariate_analysis}})]\right)\\
    &\geq \eta\left(\mathbb{E}_{i \in H}[\phi(b_i + c_i)] - \mathbb{E}_{i \in H}[\left(\phi(b_i + c_i) + \phi(-b_i + d_i)\right)\mathbbm{1}(s(i) < 1 -  \epsilon_{\ref{lemma:trivariate_analysis}})]\right)\\
    &= \eta\left(\mathbb{E}_{i \in H}[\phi(b_i + c_i)] - \mathbb{E}[\alpha s\mathbbm{1}(s < 1 -  \epsilon_{\ref{lemma:trivariate_analysis}})]\right)\\
    &\geq \eta\left(\mathbb{E}_{i \in H}[\phi(b_i + c_i)] - \mathbb{E}[\alpha \mathbbm{1}(s < 1 -  \epsilon_{\ref{lemma:trivariate_analysis}})]\right)\\
    &\geq \eta\left(\mathbb{E}_{i \in H}[\phi(b_i + c_i)] - \frac{\epsilon}{\epsilon_{\ref{lemma:trivariate_analysis}}}\mathbb{E}[\alpha]\right)
\end{align}

We need one more claim:
\begin{claim}
If the solution to Opt 4 is $(1 - \epsilon)$-optimal, then 
\begin{align}
    \mathbb{E}_{i \in H}[\phi(b_i + c_i)] &\geq \frac{1}{2}\left(1 - \epsilon\right)\mathbb{E}_{i \in H}[\phi(b_i + c_i) + \phi(-b_i + d_i)].\\
    \mathbb{E}_{i \in H}[\phi(-b_i + d_i)] &\geq \frac{1}{2}\left(1 - \epsilon\right)\mathbb{E}_{i \in H}[\phi(b_i + c_i) + \phi(-b_i + d_i)].
\end{align}
\end{claim}
\begin{proof}
Suppose without loss of generality that $ \mathbb{E}_{i \in H}[\phi(b_i + c_i)]  = q\left( \mathbb{E}_{i \in H}[\phi(b_i + c_i) + \phi(-b_i + d_i)]\right)$ for some $q \leq \frac{1}{2}$.

Then the optimum of the program is at most $q $, so we have
\begin{align}
    qC_4(2P_4)^{\frac{h}{2}} \geq q\mathbb{E}_{i \in H}[\phi(b_i + c_i) + \phi(-b_i + d_i)] \geq (1 - \epsilon)\frac{1}{2}C_4(2P_4)^{\frac{h}{2}}.
\end{align}
The conclusion follows.

\end{proof}

Using the claim and Equation~\ref{eq:jenson},
\begin{align}\label{eq:ex_alpha}
   \mathbb{E}_{i \in H}[\phi(b_i + c_i)] \geq  \frac{1}{2}\left(1 - \epsilon\right)\mathbb{E}_{i \in H}[\phi(b_i + c_i) + \phi(-b_i + d_i)] =\frac{1}{2}\left(1 - \epsilon\right) \mathbb{E}[\alpha s] \geq \frac{1}{2}\left(1 - \epsilon\right)^2\mathbb{E}[\alpha].
\end{align}
Thus plugging this into the Equation~\ref{eq:conclusion}, we have
\begin{align}
    \mathbb{E}_{i \in H}[\phi(b_i)] &\geq \eta\left(\mathbb{E}_{i \in H}[\phi(b_i + c_i)] - \frac{\epsilon}{\epsilon_{\ref{lemma:trivariate_analysis}}}\mathbb{E}[\alpha]\right)\\
    & \geq \eta\left(\mathbb{E}_{i \in H}[\phi(b_i + c_i)]\right)\left(1  - \frac{2\epsilon}{\epsilon_{\ref{lemma:trivariate_analysis}}(1 - \epsilon)^2}\right).
\end{align}
The second statement of the lemma can be proved identically, but using the second result of Lemma~\ref{lemma:trivariate_analysis}.

Now we consider the case when additionally we have $\kappa < \kphen$. 

We can bound

\begin{align}
    \mathbb{E}_{i \in H}[\phi(b_i)] &\leq \mathbb{E}_{i \in H}[\phi(b_i)\mathbbm{1}(s(i) \geq 1 - \epsilon_{\ref{lemma:trivariate_analysis}})] + \mathbb{E}_{i \in H}[\phi(b_i)\mathbbm{1}(s(i) < 1 - \epsilon_{\ref{lemma:trivariate_analysis}})]\\
    &\leq \left(\frac{1 + q(\kappa)}{2}\right)\mathbb{E}_{i \in H}[\phi(b_i + c_i)\mathbbm{1}(s(i) \geq 1 - \epsilon_{\ref{lemma:trivariate_analysis}})] + \mathbb{E}_{i \in H}[\alpha(i)\mathbbm{1}(s(i) < 1 - \epsilon_{\ref{lemma:trivariate_analysis}})]\\
    &\leq \left(\frac{1 + q(\kappa)}{2}\right)\mathbb{E}_{i \in H}[\phi(b_i + c_i)] + \frac{\epsilon}{\epsilon_{\ref{lemma:trivariate_analysis}}}\mathbb{E}[\alpha]\\
    &\leq \left(\frac{1 + q(\kappa)/2}{2}\right)\mathbb{E}_{i \in H}[\phi(b_i + c_i)]
\end{align}
for $\epsilon$ a small enough constant. Here in the second inequality we used Lemma~\ref{lemma:trivariate_analysis} and additionally the fact that for any $i$, we have $\phi(b_i) \leq \alpha(i)$ in any feasible solution. In the third inequality, we used Equation~\ref{eq:a_s_2}. In the final inequality, we used Equation~\ref{eq:ex_alpha} and chose $\epsilon$ small enough in terms of $q(\kappa)$ and $\epsilon_{\ref{lemma:trivariate_analysis}}$.

The same argument holds for $d_i$ and $-b_i$.
\end{proof}

\end{appendices}
\end{document}